\documentclass[twoside,english,sort&compress]{elsarticle}
\usepackage[T1]{fontenc}
\usepackage[latin9]{inputenc}
\usepackage{array}
\usepackage{float}
\usepackage{bm}
\usepackage{multirow}
\usepackage{amsthm}
\usepackage{amsmath}
\usepackage{amssymb}
\usepackage{graphicx}
\usepackage{nomencl}
\providecommand{\printnomenclature}{\printglossary}
\providecommand{\makenomenclature}{\makeglossary}
\makenomenclature

\makeatletter

\providecommand{\tabularnewline}{\\}
\floatstyle{ruled}
\newfloat{algorithm}{tbp}{loa}
\providecommand{\algorithmname}{Algorithm}
\floatname{algorithm}{\protect\algorithmname}

\theoremstyle{plain}
\newtheorem{thm}{\protect\theoremname}
\theoremstyle{remark}
\newtheorem{rem}[thm]{\protect\remarkname}
\ifx\proof\undefined
\newenvironment{proof}[1][\protect\proofname]{\par
\normalfont\topsep6\p@\@plus6\p@\relax
\trivlist
\itemindent\parindent
\item[\hskip\labelsep\scshape #1]\ignorespaces
}{%
\endtrivlist\@endpefalse
}
\providecommand{\proofname}{Proof}
\fi
\theoremstyle{plain}
\newtheorem{lem}[thm]{\protect\lemmaname}

\newcommand{\argmax}{\mathop{\mathrm{argmax}}}

\usepackage{amstext}
\newcommand{\MMMC}{M\textsuperscript 3C }

\usepackage{bm}

\usepackage{accents}
\newcommand{\ubar}[1]{\underaccent{\bar}{#1}}

\usepackage{algorithmic}

\renewcommand\nompreamble{For convenience, we summarize some notations used throughout the paper as follows:}

\usepackage{threeparttable}
\usepackage{graphicx}

\@ifundefined{showcaptionsetup}{}{%
 \PassOptionsToPackage{caption=false}{subfig}}
\usepackage{subfig}
\makeatother

\usepackage{babel}
\providecommand{\lemmaname}{Lemma}
\providecommand{\remarkname}{Remark}
\providecommand{\theoremname}{Theorem}

\begin{document}

\begin{frontmatter}{}

\title{Maximum Margin Clustering for State Decomposition of Metastable Systems\tnoteref{t1}}

\tnotetext[t1]{Some preliminary results have been reported in \citep{wu2013iwann}.}

\author{Hao Wu\corref{cor1}}

\ead{hwu@zedat.fu-berlin.de}

\cortext[cor1]{Corresponding author. Tel.: +49 (0)30 838 75774.}

\address{Department of Mathematics \& Computer Science, Free University of
Berlin, Arnimallee 6, 14195 Berlin, Germany}
\begin{abstract}
When studying a metastable dynamical system, a prime concern is how
to decompose the phase space into a set of metastable states. Unfortunately,
the metastable state decomposition based on simulation or experimental
data is still a challenge. The most popular and simplest approach
is geometric clustering which is developed based on the classical
clustering technique. However, the prerequisites of this approach
are: (1) data are obtained from simulations or experiments which are
in global equilibrium and (2) the coordinate system is appropriately
selected. Recently, the kinetic clustering approach based on phase
space discretization and transition probability estimation has drawn
much attention due to its applicability to more general cases, but
the choice of discretization policy is a difficult task. In this paper,
a new decomposition method designated as \emph{maximum margin metastable
clustering} is proposed, which converts the problem of metastable
state decomposition to a semi-supervised learning problem so that
the large margin technique can be utilized to search for the optimal
decomposition without phase space discretization. Moreover, several
simulation examples are given to illustrate the effectiveness of the
proposed method.\end{abstract}

\end{frontmatter}{}

\settowidth{\nomlabelwidth}{$\mathrm{vec}\left(\mathbf G\right)$}
\printnomenclature{}

\nomenclature{$\mathbf G\succeq 0$}{matrix $\mathbf G$ is symmetric and positive-semidefinite}\nomenclature{$\mathbf G\succeq \mathbf H$}{$\mathbf G-\mathbf H\succeq 0$}\nomenclature{$a\le \mathbf G\le b$}{all elements of $\mathbf G$ belongs to $[a,b]$}\nomenclature{$\mathrm{vec}\left(\mathbf G\right)$}{vector $(G_{11},G_{21},\ldots,G_{mn})^\intercal$ which consists of elements of $\mathbf G=[G_{ij}]\in\mathbb{R}^{m\times n}$}\nomenclature{${\left\Vert\mathbf G\right\Vert}$}{Frobenius norm of $\mathbf G$}\nomenclature{$\mathrm{tr}\left(\mathbf G\right)$}{trace of square matrix $\mathbf G$}\nomenclature{$1_\omega$}{indicator function of event $\omega$, taking value $1$ if $\omega$ holds and $0$ otherwise}\nomenclature{$\mathbf 0,\mathbf 1$}{vectors of zeros and ones in appropriate dimensions}\nomenclature{$\mathbf I$}{identity matrix}\nomenclature{$\mathbf G\otimes \mathbf H$}{Kronecker product of $\mathbf G$ and $\mathbf H$}\nomenclature{$\mathrm{atan2}(y,x)$}{the angle $\theta\in(-\pi,\pi]$ which satisfies $\cos\theta=x/\sqrt{(x^2+y^2)}$ and $\sin\theta=y/\sqrt{(x^2+y^2)}$}

\section{Introduction\label{sec:Introduction}}

Metastability is an ubiquitous phenomenon which occurs in many complex
systems in nature, including conformational transitions in macromolecules
\citep{noe2008transition}, autocatalytic chemical reactions \citep{biancalani2012noise}
and climate changes \citep{berglund2002metastability}. Generally
speaking, the metastability of a dynamical system means that the phase
space of the system consists of multiple macrostates called \emph{metastable
states}, in which the system tends to persist for a long time before
transiting rapidly to another metastable state. In this paper, we
focus on the problem of metastable state decomposition, i.e., how
to partition the phase space of a given system into a set of ideal
metastable states so that transitions between different metastable
states are rare events, meanwhile, there is a considerable gap between
the timescales of movements within and between metastable states.
This is an important problem for analysis of metastable systems because:
\begin{enumerate}
\item Metastable state itself is often physically interesting. For instance,
for biomolecules, the metastable states are related to the basins
of the free energy surface, and the metastable state decomposition
is helpful for seeking the native states with lowest energy and kinetic
trap states with locally minimal energy which can represent the large
scale geometric structures of conformations (see e.g., \citep{noe2007hierarchical,schwantes2013improvements}).
\item The dynamical behavior of a metastable system can be approximately
described as a transition network of metastable states on a slow timescale,
and such approximate description is able to provide a simple and global
picture of the system dynamics which preserves the essential dynamical
properties. Fig.~\ref{fig:Illustration-of-metastability} gives an
example of metastable system and the its dynamical approximation.
Note that the metastable state based dynamical approximation is closely
related with the Markov state models (MSMs) \citep{prinz2011markov},
where transitions between discrete states obtained from phase space
discretization are assumed to be Markovian and the corresponding transition
probability matrices can be used to characterize the system dynamics.
A large number of theoretical and experimental studies (see e.g.,
\citep{aldhaheri1989aggregation,chodera2006long,noe2007hierarchical,sarich2010approximation,prinz2011markov})
demonstrated that a small-sized and satisfying MSM can be obtained
by treating each metastable state as a Markovian state if the transition
rates between different metastable states are sufficiently small compared
to lifetimes of metastable states. Although the research focus of
MSMs has shifted from metastability to spectral approximation and
maximizing metastability has been shown to be not the optimal way
to improve the accuracy of MSMs in recent years \citep{chodera2014markov},
the metastable state decomposition is still commonly used to construct
MSMs in applications due to its ease-of-use. Moreover, it is required
for some more advanced analysis and modeling techniques of metastable
systems. For example, it was proved in \citep{sarich2010approximation,prinz2011markov}
that the quality of an MSM can be improved through fine discretization
around boundaries between metastable states, and the hidden Markov
model analysis approach of metastable systems proposed in \citep{noe2013projected}
needs the metastable state decomposition for parameter initialization.
\end{enumerate}
Moreover, the metastable state decomposition is also applicable to
unsupervised classification problems when data are sequentially sampled
and the data categories are rarely changed during sampling (e.g.,
classification of sensor data).

\begin{figure}
\noindent \subfloat[Potential energy landscape]{%
\begin{minipage}[t][0.45\textwidth]{0.5\textwidth}%
\noindent \begin{center}
\includegraphics[width=0.8\textwidth]{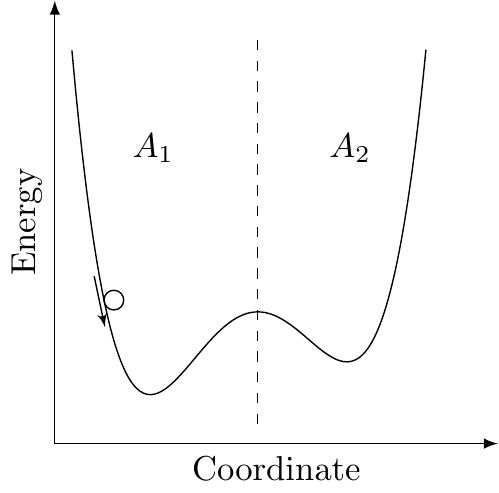}
\par\end{center}%
\end{minipage}

}\subfloat[Transition network of metastable state]{%
\begin{minipage}[t][0.45\textwidth]{0.5\textwidth}%
\noindent \begin{center}
\vspace{0.25\textwidth}
\includegraphics[width=0.8\textwidth]{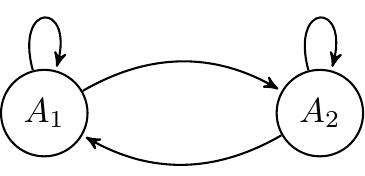}
\par\end{center}%
\end{minipage}

}

\protect\caption{Illustration of dynamical approximation based on metastability decomposition.
The left panel (a) shows the energy landscape of a stochastic diffusion
system, which has two potential wells and can be decomposed into two
metastable states $A_{1}$ and $A_{2}$ according to the energy barrier.
The right panel (b) shows the corresponding transition network obtained
from the metastable decomposition.\label{fig:Illustration-of-metastability}}
\end{figure}

For some low-dimensional systems, metastable states can be found manually
by exploring their energy landscapes (see e.g., \citep{groningen2001essential,swope2004describing,sorin2005exploring,elmer2005foldamerII}).
But for high-dimensional and complex systems, the intuitive partitioning
of the whole phase space is generally infeasible, and the metastable
state decomposition can only be performed through statistical analysis
of the simulation or experimental data.

The simplest approach for data based metastable state decomposition
is geometric clustering, which detects the metastable states through
grouping data points in phase space which are geometrically similar,
and are usually implemented by classical clustering algorithms \citep{becker1997geometric,daura1999folding,chema2003nearest,glattli2004valine,shao2007clustering}
including $k$-means, $k$-medoids, Bayesian clustering and self-organizing
maps. The theory basis of this approach is that metastable states
generally correspond to ``energy basins'' in the phase space and the
local maxima of the density function of the equilibrium state distribution
are then centers of metastable states. The application of the geometric
clustering has two key limitations. First, it requires that the global
equilibrium is reached in simulations or experiments so that the empirical
distribution of the data set is consistent with the equilibrium distribution.
Second, it is sensitive to the selection of the coordinates or projected
coordinates of the phase space because a ``bad'' coordinate system
may destroy the geometric structure of the energy landscape. (Note
that for analysis of experimental data, the coordinate system is determined
by the observation model of experiments and cannot be arbitrarily
selected.) An alternative to the geometrical cluster approach is the
density based clustering approach \citep{yao2009topological,keller2010comparing},
which tries to measure the data point density directly and cut out
regions of high data point density from data sets. It is, however,
rather susceptible to statistical uncertainty and noise.

A more general approach is kinetic clustering, which involves two
steps: (1) discretize the phase space into small bins, and estimate
the transition probabilities between bins from simulation or experimental
trajectories, (2) lump the discrete bins to metastable states through
minimizing the transition probabilities between different metastable
states. In contrast to the geometric clustering approach, this approach
implements clustering based on ``kinetic similarity'' rather than
geometric similarity, and can effectively utilize the information
of the system dynamics contained in data to improve the accuracy of
the decomposition. Furthermore, since the kinetic clustering method
is implemented based on transition probability distributions, it only
requires that simulations or experiments are in local equilibrium
instead of global equilibrium \citep{noe2009constructing}. From the
viewpoint of Markov chain theory, the kinetic clustering for metastable
state decomposition is in fact a Markov chain compression problem
\citep{aldhaheri1989aggregation}.

The most popular representatives of this approach are Perron cluster
cluster analysis (PCCA) and PCCA+ \citep{deuflhard2000identification,deuflhard2005robust,noe2007hierarchical},
which were developed based on the principle of spectral clustering
by using transition probabilities to define similarities between bins.
PCCA(+) methods have been widely used in applications due to their
computational efficiency and the acceptable quality, but they are
only applicable to time reversible systems. In \citep{mehrmann2008svd},
a singular value decomposition based lumping algorithm for nonreversible
systems was proposed. Furthermore, Jain and Stock \citep{jain2012identifying}
presented a ``most probable path algorithm'' for bin lumping in order
to avoid the computation of large matrix decompositions, and a Bayesian
lumping method was proposed in \citep{bowman2012improved} which considers
the statistical uncertainty in transition probabilities. The main
difficulty of kinetic clustering comes from the choice of bins, and
the boundaries of metastable states are unable to be accurately captured
with a poor choice of bins. Generally speaking, the discrete bins
used in kinetic clustering are still given by some geometric clustering
algorithm where the system dynamics is not considered, and one can
only improve the decomposition accuracy by adding more bins. But a
large number of bins may cause overfitting problem in the transition
probability estimation. Some recent publications \citep{kellogg2012evaluation,mcgibbon2014statistical}
investigated the choice of bin number from the perspective of model
comparison, but how to adjust the shape of bins to meet the requirement
of metastable state decomposition is still an open problem. In \citep{chodera2007automatic},
an adaptive decomposition method was proposed, which iteratively implements
the space discretization and lumping to modify boundaries of metastable
states. However, the convergence of this method cannot be guaranteed.

The objective of this paper is to propose a new clustering method
for meta\-stable state decomposition based on the large margin principle.
Recently, large margin techniques have received increasing attention
in the machine learning community \citep{crammer2001algorithmic,xu2007convex}.
For a given supervised or unsupervised classification problem, large
margin techniques can improve the robustness and generalization capability
of the classifier through maximizing the margin between training data
and classification boundaries. In this paper, we propose an optimization
model called \emph{maximum margin metastable clustering} for metastable
state decomposition by combining the large margin criterion and metastability
criterion, which can effectively utilize both the geometric and the
dynamical information contained in data, and develop a two-stage algorithm
for searching the optimal decomposition by combining global search
and local search techniques. In comparison with the previous decomposition
approaches, this decomposition method directly constructs the boundaries
of metastable states in the continuous phase space without any pre-discretization,
and can achieve a reliable and robust decomposition based on the the
large margin criterion even for small-scale data sets.

\section{Preliminaries\label{sec:Preliminaries}}

In this section, we review briefly some of the relevant background
on large margin learning. Given a set of training data $\{\mathbf{x}_{n}\}_{n=1}^{N}$
and their class labels $\{y_{n}\}_{n=1}^{N}$, where each $\mathbf{x}_{n}$
is a data point sampled from a domain $\mathcal{X}$ and $y_{n}\in\{1,\ldots,\kappa\}$,
the support vector machine (SVM) finds a hyperplane classifier defined
by the decision rule
\begin{equation}
y=\argmax_{1\le k\le\kappa}\left(\mathbf{w}_{k}^{\intercal}\bm{\phi}\left(\mathbf{x}\right)+b_{k}\right)\label{eq:decision-rule}
\end{equation}
which can map the training data to correct labels and achieve strong
generalization performance through maximizing the classification margin
between training data and decision boundaries. Here $\bm{\phi}$ is
a mapping function from $\mathcal{X}$ to a feature space $\mathbb{R}^{d}$,
and is generally induced by a Mercer kernel function $\mathrm{Ker}\left(\cdot,\cdot\right)$
with 
\begin{equation}
\mathrm{Ker}\left(\mathbf{x},\mathbf{x}'\right)=\bm{\phi}\left(\mathbf{x}\right)^{\intercal}\bm{\phi}\left(\mathbf{x}'\right)
\end{equation}
The optimal decision rule of SVM can be obtained by solving the following
optimization problem \citep{crammer2001algorithmic}:
\begin{equation}
\begin{array}{cl}
\min\limits _{\mathbf{W},\mathbf{b}} & \frac{1}{2}\left\Vert \mathbf{W}\right\Vert ^{2}\\
\mathrm{s.t.} & \forall n=1,\ldots,N,\quad k=1,\ldots,\kappa,\\
 & \left(\mathbf{w}_{y_{n}}^{\intercal}-\mathbf{w}_{k}^{\intercal}\right)\bm{\phi}\left(\mathbf{x}_{n}\right)+\left(b_{y_{n}}-b_{k}\right)+1_{y_{n}=k}\ge1.
\end{array}\label{eq:svm-no-slack}
\end{equation}
where $\mathbf{W}=(\mathbf{w}_{1}^{\intercal},\ldots,\mathbf{w}_{\kappa}^{\intercal})^{\intercal}$
and $\mathbf{b}=(b_{1},\ldots,b_{\kappa})^{\intercal}$. The constraint
of \eqref{eq:svm-no-slack}, called \emph{large margin constraint},
states that for each training data $\mathbf{x}_{n}$, the decision
function value of $\mathbf{x}_{n}$ for the true class $y_{n}$ exceeds
the decision function value for any other competing class $k\neq y_{n}$
by at least one, and $\left\Vert \mathbf{W}\right\Vert $ can be shown
to be inversely proportional to the classification margin in the feature
space under the large margin constraint. Note that \eqref{eq:svm-no-slack}
is a quadratic programming and the global optimum can be efficiently
found by convex optimization methods. However, in many practical cases,
there is no feasible solution which satisfies the large margin constraint
exactly due to the linear non-separability of training data, and we
can only seek a large margin classifier which separates the data ``as
much as possible''. In order to achieve this, we can introduce a slack
variable $\xi_{n}$ for each $\mathbf{x}_{n}$ and modify the optimization
problem \eqref{eq:svm-no-slack} as 
\begin{equation}
\begin{array}{cl}
\min\limits _{\mathbf{W},\mathbf{b},\bm{\xi}} & \frac{1}{2}\beta\left\Vert \mathbf{W}\right\Vert ^{2}+\frac{1}{N}\mathbf{1}^{\intercal}\bm{\xi}\\
\mathrm{s.t.} & \forall n=1,\ldots,N,\quad k=1,\ldots,\kappa,\\
 & \left(\mathbf{w}_{y_{n}}^{\intercal}-\mathbf{w}_{k}^{\intercal}\right)\bm{\phi}\left(\mathbf{x}_{n}\right)+\left(b_{y_{n}}-b_{k}\right)+1_{y_{n}=k}\ge1-\xi_{n}.
\end{array}\label{eq:svm}
\end{equation}
where $\mathbf{\mathbf{\bm{\xi}}}=(\xi_{1},\ldots,\xi_{N})^{\intercal}$
and $\beta>0$ is the regularization parameter. It can be observed
from constraints of \eqref{eq:svm} that $\xi_{n}\ge0$ for all $n$,
and $\xi_{n}>0$ if and only if $\left(\mathbf{w}_{y_{n}}^{\intercal}\bm{\phi}\left(\mathbf{x}_{n}\right)+b_{y_{n}}\right)-\left(\mathbf{w}_{k}^{\intercal}\bm{\phi}\left(\mathbf{x}_{n}\right)+b_{k}\right)\ge1$
does not hold for some $k\neq n$. Therefore we can conclude that
$\xi_{n}$ represents the misclassification loss of $\mathbf{x}_{n}$,
and the objective function of \eqref{eq:svm} balances the empirical
risk on the training data versus the classification margin.
\begin{rem}
It is worth pointing out that SVM was originally designed for two-class
classification \citep{vapnik1998statistical}. A variety of strategies
have been proposed to extend the SVM for multi-class problems, such
as one-against-all \citep{bottou1994comparison}, one-against-one
\citep{friedman1996another} and error-correcting output coding \citep{allwein2001reducing}.
Here we select the multi-class SVM proposed in \citep{crammer2001algorithmic}
which minimizes the the multi-class margin loss directly, because
it is more ``logical'' than the other strategies from the perspective
of large margin learning and can be easily applied to unsupervised
learning problems (see below).
\end{rem}
Maximum margin clustering (MMC) \citep{xu2004maximum,xu2005unsupervised}
extends the maximum margin principle to unsupervised learning, i.e.,
data clustering. For a set of unlabeled data $\{\mathbf{x}_{n}\}$,
MMC targets to construct a maximum margin decision rule by optimizing
\eqref{eq:svm} with both $(\mathbf{W},\mathbf{b})$ and data labels
$\{y_{n}\}$ being decision variables. But in this unsupervised case,
the optimization problem \eqref{eq:svm} has a trivially optimal solution
with $y_{n}\equiv1$ and $\left\Vert \mathbf{W}\right\Vert =0$, which
implies that all data are assigned to the same class and an infinite
classification margin is obtained. To prevent such physically meaningless
solutions with empty classes, the following class balance constraint
is required for MMC:
\begin{equation}
\varrho_{l}N\le\sum_{m=1}^{N}1_{y_{m}=k}\le\varrho_{u}N,\quad\forall k=1,\ldots,\kappa\label{eq:class-balance-constraint}
\end{equation}
where $\varrho_{l},\varrho_{u}$ denote the lower and upper bounds
of the proportion of each class size, and satisfy $0<\varrho_{l}<1/\kappa<\varrho_{u}<1$.
Then, the complete MMC problem (with slack variables) can be formulated
as

\begin{equation}
\begin{array}{cl}
\min\limits _{\mathbf{y},\mathbf{W},\mathbf{b},\bm{\xi}} & \frac{1}{2}\beta\left\Vert \mathbf{W}\right\Vert ^{2}+\frac{1}{N}\mathbf{1}^{\mathrm{T}}\mathbf{\bm{\xi}}\\
\mathrm{s.t.} & \forall n=1,\ldots,N,\quad k=1,\ldots,\kappa,\\
 & \left(\mathbf{w}_{y_{n}}^{\intercal}-\mathbf{w}_{k}^{\intercal}\right)\bm{\phi}\left(\mathbf{x}_{n}\right)+\left(b_{y_{n}}-b_{k}\right)+1_{y_{n}=k}\ge1-\xi_{n},\\
 & \varrho_{l}N\le\sum_{m=1}^{N}1_{y_{m}=k}\le\varrho_{u}N.
\end{array}\label{eq:mmc}
\end{equation}
with $\mathbf{y}=(y_{1},\ldots,y_{N})\in\left\{ 1,\ldots,\kappa\right\} ^{N}$.
In contrast with the SVM problem \eqref{eq:svm}, \eqref{eq:mmc}
is a nonconvex mixed-integer problem and much more difficult to solve.
The existing optimization methods for the MMC problem can be roughly
categorized into two types. The first type of method relaxes the MMC
problem into a semidefinite programming (SDP) problem that can be
globally optimized by standard SDP solvers \citep{xu2005unsupervised,valizadegan2006generalized}.
However, this type of method is very computationally expensive and
can only be applied to small data sets. The other type of method utilizes
some local search algorithms, such as concave-convex procedure \citep{zhao2008efficient}
and alternating optimization \citep{zhang2009maximum}, to solve the
MMC problem directly, which is more efficient than the relaxation
method but usually suffers from undesired local minima.
\begin{rem}
The balance constraint \eqref{eq:class-balance-constraint} is also
useful to avoid the problem of separating a single outlier (or very
small group of outliers) from the rest of the data, and allows one
to incorporate the prior knowledge on the balance of data distribution.
The theoretical analysis and empirical research on the balance constraint
is still very limited and the parameters $\varrho_{l},\varrho_{u}$
are usually selected by trial and error. According to our numerical
experience, the parameters can be simply set as $[\varrho_{l},\varrho_{u}]=[\epsilon,1-\epsilon]$
with $\epsilon=10^{-3}\sim10^{-1}$, and the value of $\epsilon$
influences only a little on clustering results. Moreover, it is worth
pointing out that the balance constraints $[\varrho_{l},\varrho_{u}]=[0,1]$
and $[\varrho_{l},\varrho_{u}]=[\epsilon,1-\epsilon]$ with $\epsilon$
being positive and sufficiently small are completely different in
essence for MMC. The former one is always an ``inactive'' constraint
and the optimal decision boundary is infinitely far away from all
the data, and the latter one enforces that the optimal decision boundary
of MMC must pass through the data set.
\end{rem}

\begin{rem}
For some commonly used Mercer kernel functions (e.g., Gaussian kernels),
the dimensions of the corresponding feature spaces are extremely high
or infinite, and the direct feature mappings are impossible. In such
a case, the kernel trick can be used to solve the SVM or MMC problem
by performing the feature mapping implicitly based on the kernel matrix
$\mathbf{K}=\left[K_{ij}\right]=\left[\mathrm{Ker}\left(\mathbf{x}_{i},\mathbf{x}_{j}\right)\right]\in\mathbb{R}^{N\times N}$
(see \citep{crammer2001algorithmic,xu2005unsupervised} for details).
While the kernel trick has been widely and successfully applied in
large margin learning, the calculation of kernel matrices is a bottleneck
of the kernel trick for large-scale data sets. In recent years, a
lot of alternatives to the kernel trick have been proposed to reduce
the computational and storage costs (see, e.g., \citep{Rahimi2007Random,zhang2009maximum,Pham2013fast,Kwak2013nonlinear}),
which can approximate the induced feature mapping $\bm{\phi}$ by
a low dimensional function $\hat{\bm{\phi}}\left(\mathbf{x}\right)$
such that
\begin{equation}
\bm{\phi}\left(\mathbf{x}\right)^{\intercal}\bm{\phi}\left(\mathbf{x}'\right)=\mathrm{Ker}\left(\mathbf{x},\mathbf{x}'\right)\approx\hat{\bm{\phi}}\left(\mathbf{x}\right)^{\intercal}\hat{\bm{\phi}}\left(\mathbf{x}'\right),\quad\forall\mathbf{x},\mathbf{x}'\in\mathcal{X}
\end{equation}
Therefore, in this paper, we only consider the case that the feature
mapping $\bm{\phi}$ can be explicitly defined and computed.
\end{rem}

\section{Maximum margin metastable clustering}

In this section, we apply the framework of large margin learning to
metastability analysis. Suppose that we have $L$ trajectories $\{\mathbf{x}_{1}^{1},\ldots,\mathbf{x}_{M_{1}}^{1}\},\allowbreak\{\mathbf{x}_{1}^{2},\ldots,\mathbf{x}_{M_{2}}^{2}\},\allowbreak\ldots,\allowbreak\{\mathbf{x}_{1}^{L},\ldots,\mathbf{x}_{M_{L}}^{L}\}$
in a phase space $\mathcal{X}$ which are generated by simulations
or experiments of a dynamical system, where the point $\mathbf{x}_{t}^{l}$
represents the system state at time $t$ in the $l$-th run, and $M_{l}$
denotes the time length of the $l$-th run. Furthermore, for convenience
of notation, we define
\begin{equation}
\mathcal{P}=\{(\bar{\mathbf{x}}_{n},\ubar{\mathbf{x}}_{n})\}_{n=1}^{N}=\{(\mathbf{x}_{t}^{l},\mathbf{x}_{t+1}^{l})|1\le l\le L,1\le t<M_{l}\}
\end{equation}
as the set of all transition pairs that appear in trajectories. The
purpose of metastable state decomposition is to partition $\mathcal{X}$
into $\kappa$ subspaces (\emph{metastable states}), so that it is
a rare event -- not only for the given $L$ trajectories but also
for new trajectories generated by the same system -- that $\mathbf{x}_{t}$
and $\mathbf{x}_{t+1}$ belong to different metastable states.
\begin{rem}
An important issue for metastable state decomposition is determining
the value of $\kappa$. But this issue is beyond the scope of this
paper and we simply assume that $\kappa$ is given. In practical applications,
$\kappa$ can be obtained by analyzing of life times of metastable
states (see \citep{noe2007hierarchical} for details).
\end{rem}
Analogous to learning in SVM and MMC, we can construct a large margin
decision rule for metastable state decomposition based on the following
criteria:
\begin{enumerate}
\item \textbf{\textit{Metastability criterion.}} For most of state transition
pairs $(\bar{\mathbf{x}}_{n},\ubar{\mathbf{x}}_{n})\in\mathcal{P}$,
$\bar{\mathbf{x}}_{n}$ and $\ubar{\mathbf{x}}_{n}$ should be classified
to the same metastable state, which implies that the boundaries between
metastable states are rarely crossed in runs of the system.
\item \textbf{\textit{Large margin criterion.}} The metastable state boundaries
should be placed as far away from the trajectory data as possible
in order to improve the generalization performance of the decomposition
result for unknown trajectories.
\end{enumerate}
The two criteria leads to a maximum margin metastable clustering (M\textsuperscript{3}C)
method that minimizes $\left\Vert \mathbf{W}\right\Vert $ under the
following \emph{large margin metastable constraint}:
\begin{eqnarray}
\left(\mathbf{w}_{y_{n}}^{\intercal}-\mathbf{w}_{k}^{\intercal}\right)\bm{\phi}\left(\bar{\mathbf{x}}_{n}\right)+\left(b_{y_{n}}-b_{k}\right)+1_{y_{n}=k} & \ge & 1\nonumber \\
\left(\mathbf{w}_{y_{n}}^{\intercal}-\mathbf{w}_{k}^{\intercal}\right)\bm{\phi}\left(\ubar{\mathbf{x}}_{n}\right)+\left(b_{y_{n}}-b_{k}\right)+1_{y_{n}=k} & \ge & 1,\quad\forall n,k\label{eq:metastable-largin-margin-constraint}
\end{eqnarray}
where $y_{n}\in\left\{ 1,\ldots,\kappa\right\} $ denotes the metastable
state label of $\bar{\mathbf{x}}_{n}$ and $\ubar{\mathbf{x}}_{n}$.
It can be seen that M\textsuperscript{3}C is in fact a semi-supervised
learning method which can be interpreted as follows:
\begin{enumerate}
\item The constraint enforces $\bar{\mathbf{x}}_{n}$ and $\ubar{\mathbf{x}}_{n}$
being assigned to the same metastable state for each state transition
pair $(\bar{\mathbf{x}}_{n},\ubar{\mathbf{x}}_{n})$.
\item For any $\mathbf{x}\in\bigcup_{n=1}^{N}\{\bar{\mathbf{x}}_{n},\ubar{\mathbf{x}}_{n}\}_{n=1}^{N}$,
if it is assigned to the metastable state $y$, the the distances
between $\mathbf{x}$ and the boundaries of $y$ are
\begin{equation}
\frac{\left(\mathbf{w}_{y}^{\intercal}-\mathbf{w}_{k}^{\intercal}\right)\bm{\phi}\left(\mathbf{x}\right)+\left(b_{y}-b_{k}\right)}{\left\Vert \mathbf{w}_{y}-\mathbf{w}_{k}\right\Vert }\ge\frac{1}{\left\Vert \mathbf{w}_{y}-\mathbf{w}_{k}\right\Vert },\quad\text{for }k\neq y
\end{equation}
Then the margin between data and decision boundaries can be increased
by minimizing the objective function $\left\Vert \mathbf{W}\right\Vert $
of M\textsuperscript{3}C.
\end{enumerate}
Moreover, it is interesting to note that the M\textsuperscript{3}C
method can be expressed as a specific MMC method in the space of transition
pairs $(\bar{\mathbf{x}}_{n},\ubar{\mathbf{x}}_{n})$, which seeks
a decision rule
\begin{equation}
\left(\bar{y},\ubar y\right)=\argmax_{\bar{k},\ubar k}\mathbf{w}_{\bar{k}\ubar k}^{\intercal}\bm{\phi}\left(\bar{\mathbf{x}},\ubar{\mathbf{x}}\right)+b_{\bar{k}\ubar k}\label{eq:decision-rule-pair}
\end{equation}
under a large margin constraint
\begin{equation}
\left(\mathbf{w}_{y_{n}y_{n}}^{\intercal}-\mathbf{w}_{\bar{k}\ubar k}^{\intercal}\right)\bm{\phi}\left(\bar{\mathbf{x}}_{n},\ubar{\mathbf{x}}_{n}\right)+\left(b_{y_{n}y_{n}}-b_{\bar{k}\ubar k}\right)+1_{y_{n}=\bar{k}=\ubar k}\ge1\label{eq:large-margin-constraint-pair}
\end{equation}
where $\mathbf{w}_{\bar{k}\ubar k}=(\mathbf{w}_{\bar{k}}^{\intercal},w_{\ubar k}^{\intercal})^{\intercal}$,
$b_{\bar{k}\ubar k}=b_{\bar{k}}+b_{\ubar k}$ and $\bm{\phi}(\bar{\mathbf{x}}_{n},\ubar{\mathbf{x}}_{n})=(\bm{\phi}(\bar{\mathbf{x}}_{n})^{\intercal},\bm{\phi}(\ubar{\mathbf{x}}_{n})^{\intercal})^{\intercal}$
denotes the extended feature function. (The proof of the equivalence
between \eqref{eq:metastable-largin-margin-constraint} and \eqref{eq:large-margin-constraint-pair}
is given in Appendix \pageref{sec:Proof-of-equivalence-margin-constraint}.)
The decision rule \eqref{eq:decision-rule-pair} can map a transition
pair in the phase space to a pair of metastable state labels, and
the constraint \eqref{eq:large-margin-constraint-pair} restricts
that all transition pairs in $\mathcal{P}$ are not only far away
from the decision boundaries in the product feature space $\bm{\phi}\left(\mathcal{X}\right)\times\bm{\phi}\left(\mathcal{X}\right)$
but also enclosed in the area $\cup_{k=1}^{\kappa}\{(\bar{\mathbf{x}},\ubar{\mathbf{x}})|(k,k)=\argmax_{\bar{k},\ubar k}w_{\bar{k}\ubar k}^{\intercal}\bm{\phi}(\bar{\mathbf{x}},\ubar{\mathbf{x}})+b_{\bar{k}\ubar k}\}$.

Based on the above discussion, as well as by introducing slack variables
and the class balance constraint as in MMC, we can formulate the M\textsuperscript{3}C
problem as
\begin{equation}
\begin{array}{cl}
\min\limits _{\mathbf{y},\mathbf{W},\mathbf{b},\bm{\xi}} & \frac{1}{2}\beta\left\Vert \mathbf{W}\right\Vert ^{2}+\frac{1}{N}\mathbf{1}^{\intercal}\bm{\xi}\\
\mathrm{s.t.} & \forall n=1,\ldots,N,\quad\forall\bar{k},\ubar k=1,\ldots,\kappa,\\
 & \left(\mathbf{w}_{y_{n}y_{n}}^{\intercal}-\mathbf{w}_{\bar{k}\ubar k}^{\intercal}\right)\bm{\phi}\left(\bar{\mathbf{x}}_{n},\ubar{\mathbf{x}}_{n}\right)\\
 & +\left(b_{y_{n}y_{n}}-b_{\bar{k}\ubar k}\right)+1_{y_{n}=\bar{k}=\ubar k}\ge1-\xi_{n},\\
 & \varrho_{l}N\le\sum_{m=1}^{N}1_{y_{m}=\bar{k}}\le\varrho_{u}N.
\end{array}\label{eq:m3c}
\end{equation}

\section{Comparison of maximum margin metastable clustering with relative
methods}

Fig.~\ref{fig:Illustration-of-decision-boundary} shows a two dimensional
and two-class example to illustrate the difference between the ideas
of MMC and M\textsuperscript{3}C (or the other geometric clustering
methods). In this example, all the trajectory data are distributed
in areas $A_{1}$, $A_{2}$ and $A_{3}$, and $A_{1}$ is far away
from $A_{2}$ and $A_{3}$. As shown in Fig.~\ref{fig:MMCIllustration},
the optimal decision boundary founded by MMC is a horizontal line
for which can achieve the maximum classification margin, and then
MMC decomposes the three areas as $\{A_{1}\}\cup\{A_{2},A_{3}\}$.
Fig.~\ref{fig:M3CIllustration} displays both data points and trajectories.
As can be seen, the MMC decomposition severely violates the large
margin metastable constraint because of the existence of frequent
switches between $A_{1}$ and $A_{2}$ in trajectories, and the optimal
decomposition provided by M\textsuperscript{3}C (with an appropriate
class balance constraint) is then $\{A_{1},A_{2}\}\cup\{A_{3}\}$.
Obviously, the decomposition result of M\textsuperscript{3}C is more
reasonable from the view of metastability analysis. From this example
we can observe that the M\textsuperscript{3}C method is able to exploit
the information on both phase space distribution and metastable dynamics
contained in the trajectory data by using the large margin metastable
constraint.

\begin{figure}
\subfloat[MMC\label{fig:MMCIllustration}]{\includegraphics[width=0.5\textwidth]{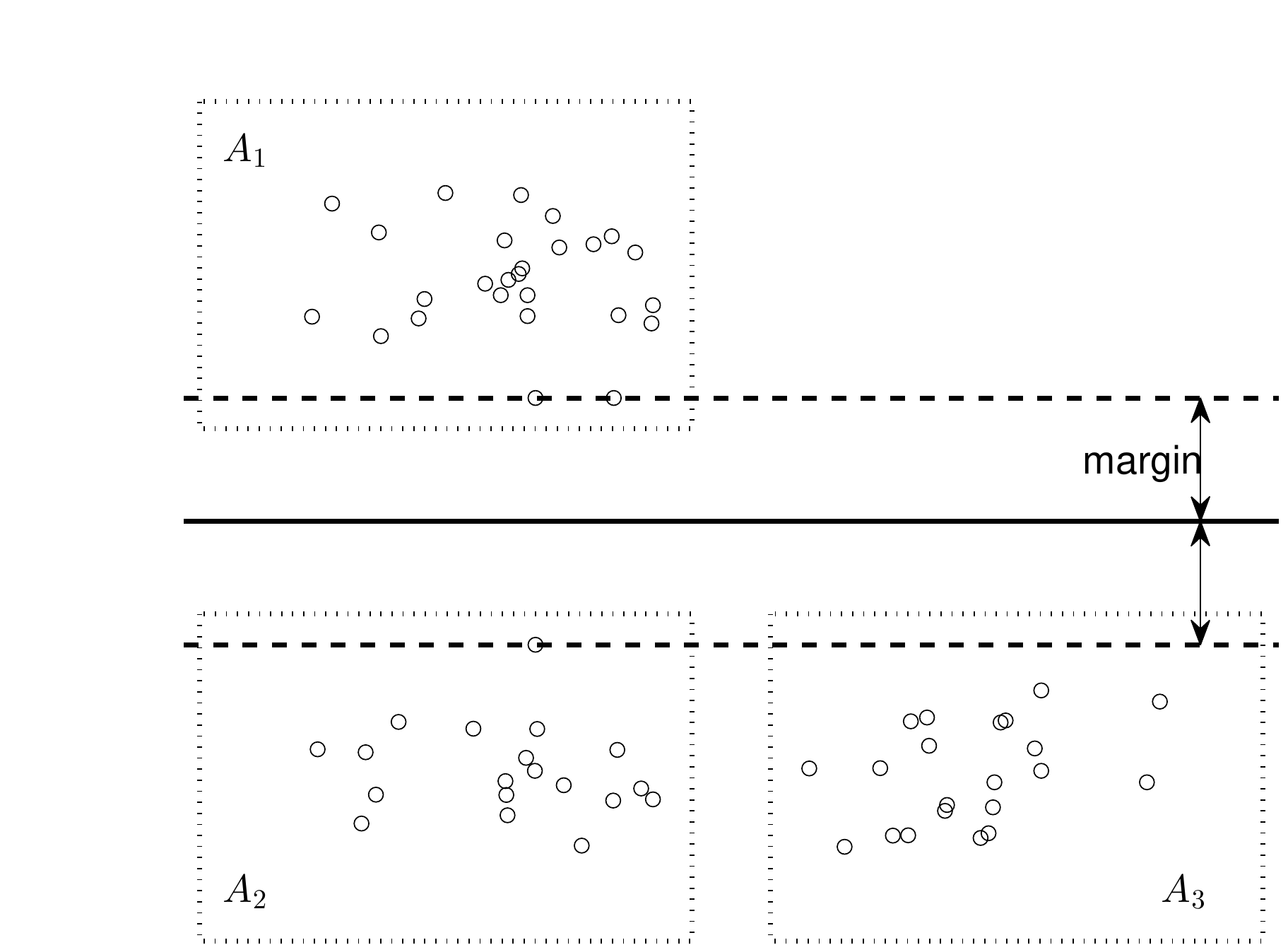}

}\hfill{}\subfloat[M\textsuperscript{3}C\label{fig:M3CIllustration}]{\includegraphics[width=0.5\textwidth]{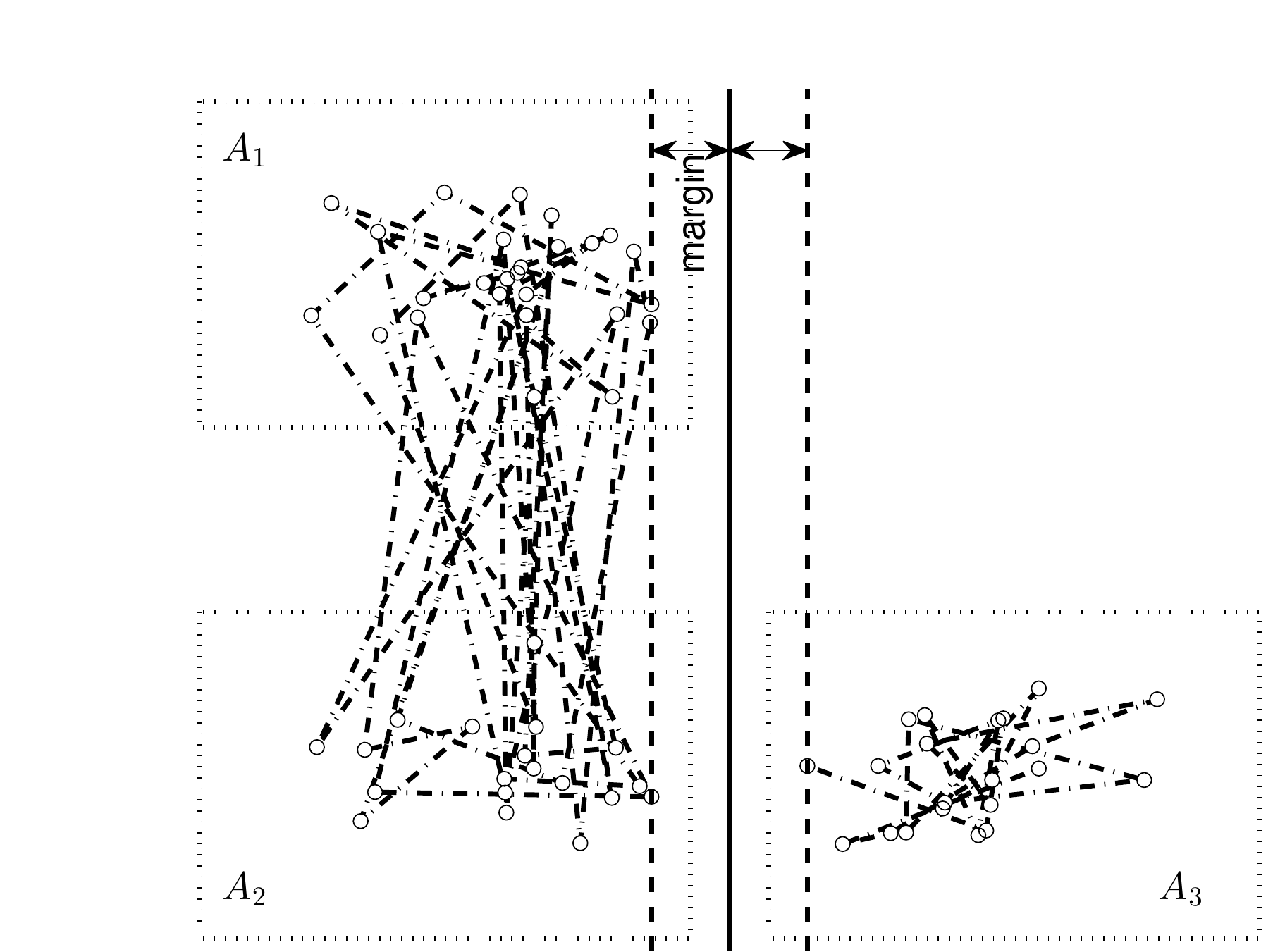}

}

\protect\caption{Illustration of decision boundaries obtained by MMC and M\protect\textsuperscript{3}C,
where solid lines are decision boundaries, dashed lines show the margins,
dash-dot lines represent phase space trajectories, and circles denote
data points sampled from the trajectories.\label{fig:Illustration-of-decision-boundary}}
\end{figure}

Compared with the kinetic clustering methods of metastable state decomposition
(e.g., PCCA+), the M\textsuperscript{3}C method directly minimizes
the ``loss'' caused by boundary crossings between metastable states
without an explicit model of the system dynamics. Therefore, plenty
of statistical problems which are essential and difficult for kinetic
clustering methods, including the choice of space discretization and
the estimation of transition probabilities, are not present in M\textsuperscript{3}C.
Furthermore, the large margin criterion used in M\textsuperscript{3}C
provides a convenient and powerful way to reduce the influence of
statistical noise. Although some kinetic clustering methods can also
handle the statistical noise in a Bayesian manner (see e.g., \citep{bowman2012improved}),
M\textsuperscript{3}C offers the advantage of naturally achieving
good performance of generalization and robustness without assuming
any prior distribution.

The major challenge of M\textsuperscript{3}C arises from \eqref{eq:m3c},
which is a nonconvex mixed-integer optimization problem as the MMC
problem \eqref{eq:mmc}. Due to the similarity between the MMC and
M\textsuperscript{3}C problems, solving of the latter encounters
similar difficulties: the SDP relaxation based global search requires
a high computational cost and the local search easily gets stuck in
poor local optima.

\section{Optimization method for metastable maximum margin clustering}

In this section, we will propose a coarse graining based strategy
to solve the M\textsuperscript{3}C problem \eqref{eq:m3c} that takes
advantages of both global search and local search methods, and is
applicable to large-scale data sets. The proposed optimization procedure
can be sketched in the following steps:
\begin{enumerate}
\item \textbf{\textit{Coarse graining.}} Discretize the space $\mathcal{X}\times\mathcal{X}$
of transition pairs into $N^{c}$ ($\kappa<N^{c}\ll N$) bins and
approximate $\mathcal{P}$ by a coarse-grained data set $\mathcal{P}^{c}=\{(\bar{\mathbf{x}}_{n}^{c},\ubar{\mathbf{x}}_{n}^{c})\}_{n=1}^{N^{c}}$
with normalized weights $\{c_{n}\}$, where $(\bar{\mathbf{x}}_{n}^{c},\ubar{\mathbf{x}}_{n}^{c})$
denotes the center of the $n$-th bin and $c_{n}$ is proportional
to the number of transition pairs contained in the $n$-th bin. This
step can be done by $k$-means, $k$-medoids or any other traditional
clustering algorithm, and the value of $N^{c}$ is chosen based on
the limitation of computational resources.
\item \textbf{\textit{Global search.}} Perform M\textsuperscript{3}C on
the coarse-grained data set $\mathcal{P}^{c}$ by using the SDP relaxation
method. Note that $\mathcal{P}^{c}$ only has a small number of distinct
elements, then a satisfactory solution can be achieved in this step
without too much computing burden.
\item \textbf{\textit{Local search.}} Refine the solution obtained in the
global search step by applying some local search algorithm to the
M\textsuperscript{3}C problem on $\mathcal{P}$. Here we select the
alternating optimization presented in \citep{zhang2009maximum} to
perform the local search because it is computationally efficient and
can handle the class balance constraint easily.
\end{enumerate}
Below we will discuss the last two steps in more detail.

\subsection{Global search}

For convenience of analysis and computation, we confine that all decision
hyperplanes pass through the origin in the feature space, i.e., $b_{1}=\ldots=b_{\kappa}=0$.
(This is a mild restriction especially for high dimensional feature
spaces, since it reduces the number of freedom degrees of decision
hyperplanes only by one.) Then the M\textsuperscript{3}C problem
on $\mathcal{P}^{c}$ can be expressed as
\begin{equation}
\begin{array}{cl}
\min\limits _{\mathbf{y}^{c},\mathbf{W},\bm{\xi}^{c}} & \frac{1}{2}\beta\left\Vert \mathbf{W}\right\Vert ^{2}+\mathbf{c}^{\intercal}\bm{\xi}^{c}\\
\mathrm{s.t.} & \forall n=1,\ldots,N^{c},\quad\forall\bar{k},\ubar k=1,\ldots,\kappa,\\
 & \left(\mathbf{w}_{y_{n}^{c}y_{n}^{c}}^{\intercal}-\mathbf{w}_{\bar{k}\ubar k}^{\intercal}\right)\bm{\phi}\left(\bar{\mathbf{x}}_{n}^{c},\ubar{\mathbf{x}}_{n}^{c}\right)+1_{y_{n}^{c}=\bar{k}=\ubar k}\ge1-\xi_{n}^{c},\\
 & \varrho_{l}\le\sum_{m=1}^{N^{c}}1_{y_{m}^{c}=\bar{k}}\cdot c_{m}\le\varrho_{u}.
\end{array}\label{eq:m3c-coarse-grained}
\end{equation}
where $y_{n}^{c}\in\{1,\ldots,\kappa\}$ and $\xi_{n}^{c}$ denotes
the label and slack variable of the $n$-th coarse-grained transition
pair $(\bar{\mathbf{x}}_{n}^{c},\ubar{\mathbf{x}}_{n}^{c})$, $\mathbf{y}^{c}=(y_{1}^{c},\ldots,y_{N^{c}}^{c})$,
$\bm{\xi}^{c}=(\xi_{1}^{c},\ldots,\xi_{N^{c}}^{c})^{\intercal}$,
and $\mathbf{c}=(c_{1},\ldots,c_{N^{c}})^{\intercal}$ denotes the
vector of weights of coarse-grained transition pairs. It is clear
that the assignment $\mathbf{y}^{c}$ in \eqref{eq:m3c-coarse-grained}
can be represented by the equivalence relation matrices $\mathbf{D}\in\mathbb{R}^{N^{c}\times\kappa}$
and $\mathbf{M}\in\mathbb{R}^{N^{c}\times N^{c}}$ which are defined
by
\begin{equation}
\mathbf{D}=[D_{ij}]=[1_{y_{i}^{c}=j}],\quad\mathbf{M}=[M_{ij}]=[1_{y_{i}^{c}=y_{j}^{c}}]\label{eq:MD-definition}
\end{equation}
i.e., the $(i,j)$-th element of $\mathbf{D}$ indicates if the label
of $(\bar{\mathbf{x}}_{i}^{c},\ubar{\mathbf{x}}_{i}^{c})$ is $j$,
and the $(i,j)$-th element of $\mathbf{M}$ indicates if $(\bar{\mathbf{x}}_{i}^{c},\ubar{\mathbf{x}}_{i}^{c})$
and $(\bar{\mathbf{x}}_{j}^{c},\ubar{\mathbf{x}}_{j}^{c})$ are assigned
to the same label.

Replacing $\mathbf{y}^{c}$ by $(\mathbf{M},\mathbf{D})$ as assignment
variables in \eqref{eq:m3c-coarse-grained} and using the strong duality
theorem, we can get an equivalent form of the coarse-grained M\textsuperscript{3}C
problem.
\begin{thm}
\label{thm:cg-m3c-equivalence}Solving \eqref{eq:m3c-coarse-grained}
is equivalent to solving the optimization problem
\begin{equation}
\begin{array}{cl}
\min\limits _{\mathbf{M},\mathbf{D},\bm{\alpha},\bm{\theta}} & \frac{1}{2}\bm{\theta}^{\intercal}\bm{\theta}-\mathbf{c}^{\intercal}\bm{\alpha}-\frac{1}{2\beta}\mathrm{tr}\left(\mathbf{M}\mathbf{C}\mathbf{K}^{s}\mathbf{C}\right)+\mathbf{1}^{\intercal}\mathbf{c}\\
\mathrm{s.t.} & \mathbf{q}\left(\mathbf{D}\right)+\left(\mathbf{1}\otimes\mathbf{I}\right)\bm{\alpha}+\mathbf{R}\bm{\theta}\le0,\\
 & \varrho_{l}\le\mathbf{M}\mathbf{c}\le\varrho_{u},\\
 & \mathrm{diag}\left(\mathbf{M}\right)=\mathbf{1},\\
 & \mathbf{M}=\mathbf{D}\mathbf{D}^{\intercal},\\
 & \mathbf{M}\in\left\{ 0,1\right\} ^{N^{c}\times N^{c}},\mathbf{D}\in\left\{ 0,1\right\} ^{N^{c}\times\kappa}.
\end{array}\label{eq:m3c-coarse-grained-equivalent}
\end{equation}
where $\mathbf{C}$ is a diagonal matrix with the elements of $\mathbf{c}$
on the diagonal, $\mathbf{q}\left(\mathbf{D}\right)$ is a linear
function of $\mathbf{D}$ defined in \eqref{eq:q}, and definitions
of $\mathbf{K}^{s}$ and $\mathbf{R}$ are given by \eqref{eq:Ks}
and \eqref{eq:R}.\end{thm}
\begin{proof}
See \ref{sec:Proof-of-Theorem-cg-m3c-equivalence}.\end{proof}
\begin{rem}
\label{rem:M-D}The last four constraints on $\mathbf{M}$ and $\mathbf{D}$
of \eqref{eq:m3c-coarse-grained-equivalent} come from the fact that
\emph{there is a $\mathbf{y}^{c}$ such that \eqref{eq:MD-definition}
holds if and only if $\mathbf{M}$ and $\mathbf{D}$ are both binary
matrices and satisfy $\mathrm{diag}\left(\mathbf{M}\right)=\mathbf{1}$
and $\mathbf{M}=\mathbf{D}\mathbf{D}^{\intercal}$} \citep{xu2007convex}.
\end{rem}
It is easy to see that \eqref{eq:m3c-coarse-grained-equivalent} is
a convex optimization problem if we drop the constraints that $\mathbf{M}$
and $\mathbf{D}$ are binary matrices and the nonlinear constraint
$\mathbf{M}=\mathbf{D}\mathbf{D}^{\intercal}$. This motivates an
approximation method for solving the coarse-grained M\textsuperscript{3}C
problem based on two relaxations as in \citep{xu2007convex}. The
first relaxation allows elements of $\mathbf{M}$ and $\mathbf{D}$
to take values in $[0,1]$ instead of $\{0,1\}$ and the second relaxation
is to replace $\mathbf{M}=\mathbf{D}\mathbf{D}^{\intercal}$ with
a convex inequality $\mathbf{M}\succeq\mathbf{D}\mathbf{D}^{\intercal}$.
By using these two relaxations, \eqref{eq:m3c-coarse-grained-equivalent}
can be relaxed to a convex optimization problem:

\begin{equation}
\begin{array}{cl}
\min\limits _{\mathbf{M},\mathbf{D},\bm{\alpha},\bm{\theta}} & \frac{1}{2}\bm{\theta}^{\intercal}\bm{\theta}-\mathbf{c}^{\intercal}\bm{\alpha}-\frac{1}{2\beta}\mathrm{tr}\left(\mathbf{M}\mathbf{C}\mathbf{K}^{s}\mathbf{C}\right)+\mathbf{1}^{\intercal}\mathbf{c}\\
\mathrm{s.t.} & \mathbf{q}\left(\mathbf{D}\right)+\left(\mathbf{1}\otimes\mathbf{I}\right)\bm{\alpha}+\mathbf{R}\bm{\theta}\le0,\\
 & \varrho_{l}\le\mathbf{M}\mathbf{c}\le\varrho_{u},\\
 & \mathrm{diag}\left(\mathbf{M}\right)=\mathbf{1},\\
 & \mathbf{M}\succeq\mathbf{D}\mathbf{D}^{\intercal},\\
 & 0\le\mathbf{M}\le1,0\le\mathbf{D}\le1.
\end{array}\label{eq:m3c-coarse-grained-equivalent-convex}
\end{equation}
According to the Schur complement lemma \citep{horn1988matrix}, we
can further reformulate \eqref{eq:m3c-coarse-grained-equivalent-convex}
in a more convenient form:

\begin{equation}
\begin{array}{cl}
\min\limits _{\mathbf{M},\mathbf{D},\bm{\alpha},\bm{\theta},\zeta} & \zeta\\
\mathrm{s.t.} & \left[\begin{array}{cc}
\mathbf{I} & \bm{\theta}\\
\bm{\theta}^{\intercal} & 2\left(\zeta+\mathbf{c}^{\intercal}\bm{\alpha}+\frac{1}{2\beta}\mathrm{tr}\left(\mathbf{M}\mathbf{C}\mathbf{K}^{s}\mathbf{C}\right)-\mathbf{1}^{\intercal}\mathbf{c}\right)
\end{array}\right]\succeq0,\\
 & \mathbf{q}\left(\mathbf{D}\right)+\left(\mathbf{1}\otimes\mathbf{I}\right)\bm{\alpha}+\mathbf{R}\bm{\theta}\le0,\\
 & \varrho_{l}\le\mathbf{M}\mathbf{c}\le\varrho_{u},\\
 & \mathrm{diag}\left(\mathbf{M}\right)=\mathbf{1},\\
 & \left[\begin{array}{cc}
\mathbf{I} & \mathbf{D}^{\intercal}\\
\mathbf{D} & \mathbf{M}
\end{array}\right]\succeq0,\\
 & 0\le\mathbf{M}\le1,0\le\mathbf{D}\le1.
\end{array}\label{eq:m3c-sdp}
\end{equation}
It is an SDP problem and can be solved in polynomial time.

Since the optimal $\mathbf{M}$ obtained from \eqref{eq:m3c-sdp}
is a symmetric matrix satisfying $\mathrm{diag}\left(\mathbf{M}\right)=\mathbf{1}$
and $0\le\mathbf{M}\le1$, it can be interpreted as a similarity matrix
of $\mathcal{P}^{c}$ with each element $M_{ij}$ representing some
measure of the similarity between $(\bar{\mathbf{x}}_{i}^{c},\ubar{\mathbf{x}}_{i}^{c})$
and $(\bar{\mathbf{x}}_{j}^{c},\ubar{\mathbf{x}}_{j}^{c})$. Hence,
we can utilize the spectral clustering method to recover $\mathbf{y}^{c}$
from the optimal $\mathbf{M}$ such that $M_{ij}=1_{y_{i}^{c}=y_{j}^{c}}$
approximately holds for all $i,j$.
\begin{rem}
Obviously, the relaxation method proposed in this section can be directly
applied to the original M\textsuperscript{3}C problem \eqref{eq:m3c}
without any coarse graining, and the corresponding relaxed problem
can be solved even in the case that we do not know the feature mapping
implicitly but only know the kernel function $\mathrm{Ker}\left(\cdot,\cdot\right)$.
However, this scheme is generally computationally infeasible because
the it involves an SDP problem with $O(N^{2})$ parameters.
\end{rem}

\subsection{Local search\label{sub:Local-search}}

We now investigate how to refine the clustering result obtained in
the global search by local search. A natural way to do this is to
alternatively minimize \eqref{eq:m3c} with respect to $\left(\mathbf{W},\mathbf{b}\right)$
keeping $\mathbf{y}$ fixed and vice versa until convergence, where
the initial value of $\mathbf{y}$ is given by the coarse-grained
label assignment $\mathbf{y}^{c}$ with
\begin{equation}
y_{n}=y_{i}^{c},\text{ if }(\bar{\mathbf{x}}_{n},\ubar{\mathbf{x}}_{n})\text{ is in the bin centered at }(\bar{\mathbf{x}}_{i}^{c},\ubar{\mathbf{x}}_{i}^{c})\label{eq:y-init}
\end{equation}

Note that the M\textsuperscript{3}C problem \eqref{eq:m3c} with
a fixed $\mathbf{y}$ is just a standard quadratic programming problem.
Here we only discuss the optimization problem with respect to $\mathbf{y}$.
For fixed $\mathbf{W}$ and $\mathbf{b}$, \eqref{eq:m3c} is reduced
to
\begin{equation}
\begin{array}{cl}
\min\limits _{\mathbf{y}} & \mathbf{1}^{\intercal}\bm{\xi}\\
\mathrm{s.t.} & \forall n=1,\ldots,N,\quad\forall\bar{k},\ubar k=1,\ldots,\kappa,\\
 & \left(\mathbf{w}_{y_{n}y_{n}}^{\intercal}-\mathbf{w}_{\bar{k}\ubar k}^{\intercal}\right)\bm{\phi}\left(\bar{\mathbf{x}}_{n},\ubar{\mathbf{x}}_{n}\right)\\
 & +\left(b_{y_{n}y_{n}}-b_{\bar{k}\ubar k}\right)+1_{y_{n}=\bar{k}=\ubar k}\ge1-\xi_{n},\\
 & \varrho_{l}N\le\sum_{m=1}^{N}1_{y_{m}=\bar{k}}\le\varrho_{u}N.
\end{array}\label{eq:m3c-y}
\end{equation}
It is simple to verify that \eqref{eq:m3c-y} can be transformed into
a binary linear programming problem

\begin{equation}
\begin{array}{cl}
\min\limits _{\mathbf{D}^{f}} & \mathrm{tr}\left(\mathbf{H}^{\intercal}\mathbf{D}^{f}\right)\\
\mathrm{s.t.} & \varrho_{l}N\le\mathbf{1^{\intercal}}\mathbf{D}^{f}\le\varrho_{u}N,\\
 & \mathbf{D}^{f}\in\left\{ 0,1\right\} ^{N\times\kappa}.
\end{array}\label{eq:m3c-D}
\end{equation}
where $\mathbf{D}^{f}=[D_{ij}^{f}]$ is a relation matrix with $D_{ij}^{f}=1_{y_{i}=j}$,
and $\mathbf{H}=[H_{ij}]$ is defined by
\begin{equation}
H_{ij}=\max_{\bar{k},\ubar k}1-1_{j=\bar{k}=\ubar k}-\left(\mathbf{w}_{jj}^{\intercal}-\mathbf{w}_{\bar{k}\ubar k}^{\intercal}\right)\bm{\phi}\left(\bar{\mathbf{x}}_{i},\ubar{\mathbf{x}}_{i}\right)-\left(b_{jj}-b_{\bar{k}\ubar k}\right)
\end{equation}
Although \eqref{eq:m3c-D} belongs to the class of NP-hard problems,
it can be efficiently tackled by enumeration and cutting-plane techniques
\citep{genova2011linear} in practice, and there exist a lot of software
packages for solving large-scale integer programming problems like
\eqref{eq:m3c-D}, including MOSEK \citep{mosek}, Gurobi \citep{gurobi}
and GLPK \citep{pryor2011faster}.
\begin{rem}
It was reported in \citep{zhang2009maximum} that the alternating
optimization method for MMC often suffers from premature convergence
and can only change a small proportion of data labels even with a
poor initialization, but our numerical experiments show that this
problem is not serious for the local search procedure of M\textsuperscript{3}C.
The analysis of this phenomenon requires further investigation and
we only give a rough explanation here. Generally speaking, the switching
between metastable states can be observed many times in trajectories.
Therefore, during the local search of M\textsuperscript{3}C, there
are always a number of transition pairs that are close to the decision
boundaries and violate the large margin metastable constraint with
not-so-small slack values even if $\beta\ll1$. These misclassified
transition pairs are helpful to ``push'' the decision boundaries away
from the initial positions and their labels can be easily changed
especially in early iterations of the local search.
\end{rem}

\subsection{Full description of the optimization procedure\label{sub:Full-description-of-algorithm}}

Based on the above analysis, the complete optimization method developed
for M\textsuperscript{3}C can be summarized in Algorithm \ref{alg:m3c}.

\begin{algorithm}
\begin{algorithmic}[1]

\STATE generate a coarse-grained transition pair set $\mathcal{P}^{c}$
with cardinality $N^{c}$ and normalized weights $\mathbf{c}=(c_{1},\ldots,c_{N^{c}})^{\intercal}$
from $\mathcal{P}$ by the $k$-means or $k$-medoids algorithm

\STATE find $(\mathbf{M}^{*},\mathbf{D}^{*})$ as the solution to
the SDP problem \eqref{eq:m3c-sdp}

\STATE perform the spectral clustering algorithm with similarity
matrix $\mathbf{M}^{*}$ to get class labels $\mathbf{y}^{c}=(y_{1}^{c},\ldots,y_{N^{c}}^{c})$
of coarse-grained transition pairs in $\mathcal{P}^{c}$

\STATE calculate class labels $\mathbf{y}=(y_{1},\ldots,y_{N})$
of transition pairs in $\mathcal{P}$ from $\mathbf{y}^{c}$ by \eqref{eq:y-init}

\STATE initialize with $\mathbf{y}^{(0)}=\mathbf{y}$ and $r=1$

\REPEAT

\STATE find $(\mathbf{W}^{*},\mathbf{b}^{*})$ as the solution to
the quadratic programming problem \eqref{eq:m3c} with $\mathbf{y}$
fixed to be $\mathbf{y}^{(r-1)}$ and set $(\mathbf{W}^{(r)},\mathbf{b}^{(r)})=(\mathbf{W}^{*},\mathbf{b}^{*})$

\STATE find $\mathbf{D}^{f*}=[D_{ij}^{f*}]$ as the solution to the
binary linear programming problem \eqref{eq:m3c-D} with $(\mathbf{W},\mathbf{b})$
set to be $(\mathbf{W}^{(r)},\mathbf{b}^{(r)})$.

\STATE calculate $\mathbf{y}^{*}=(y_{1}^{*},\ldots,y_{N}^{*})$ by
$y_{n}^{*}=\argmax_{j}D_{nj}^{f*}$ and set $\mathbf{y}^{(r)}=\mathbf{y}^{*}$

\STATE set $r:=r+1$

\UNTIL{the Hamming distance between $\mathbf{y}^{(r)}$ and $\mathbf{y}^{(r-1)}$
is less than or equal to a given constant $\alpha_{H}$ or $r$ is
larger than a pre-defined threshold $r_{\max}$}

\RETURN $(\mathbf{W}^{(r)},\mathbf{b}^{(r)},\mathbf{y}^{(r)})$

\end{algorithmic}

\protect\caption{Optimization procedure for M\protect\textsuperscript{3}C\label{alg:m3c}}
\end{algorithm}

\section{Experiments}

In this section, we demonstrate the performance of the proposed decomposition
method on some synthetic metastable systems and molecular dynamics
simulations.

The detailed settings of Algorithm \ref{alg:m3c} for M\textsuperscript{3}C
are as follows:
\begin{itemize}
\item In the coarse graining step, $\mathcal{P}^{c}$ is provided by the
standard $k$-medoids algorithm \citep{trevor2001elements} with $N^{c}=30$.
(Here we use $k$-medoids rather than $k$-means because $k$-medoids
is more robust to outliers and can avoid the appearance of the coarse-grained
transitions which make no physical sense.)
\item All the involved SDP problems and the binary linear programming problem
\eqref{eq:m3c-D} are solved by using the Mosek solver \citep{mosek}
through the CVX interface in Matlab \citep{cvx}.
\item The spectral clustering algorithm proposed in \citep{weber2003improved}
is utilized to extract $\mathbf{y}^{c}$ from the relation matrix
$\mathbf{M}$.
\item The feature mapping $\bm{\phi}\left(\cdot\right)$ is induced from
the Gaussian kernel
\begin{equation}
\mathrm{Ker}\left(\mathbf{x},\mathbf{x}'\right)=\exp\left(-\frac{\left\Vert \mathbf{x}-\mathbf{x}'\right\Vert ^{2}}{2\sigma^{2}}\right)
\end{equation}
and explicitly computed by the random Fourier method \citep{Rahimi2007Random}
with $d=50$. The width parameter $\sigma$ is determined by a grid
search over $\{2^{-4},2^{-3},\ldots,2^{4}\}$ and we select the value
of $\sigma$ which leads to the minimal value of the objective function
of the M\textsuperscript{3}C problem \eqref{eq:m3c}.
\item The value of regularization parameter $\beta$ does not have a significant
effect for our experiments, so it is simply fixed to $0.01$.
\item The parameters of termination condition are set to be $\alpha_{H}=0$
and $r_{\max}=100$, and the parameters of class balance constraint
are $\left(\varrho_{l},\varrho_{u}\right)=\left(0.01,0.99\right)$.
\end{itemize}
For comparison purposes, the following three decomposition methods
are also considered in our experiments:
\begin{enumerate}
\item $k$-medoids clustering, which directly decomposes the phase space
into $\kappa$ macrostates based on the set $\mathcal{S}=\{\mathbf{x}_{t}^{l}|1\le l\le L,1\le t\le M_{l}\}$.
\item MMC based on $\mathcal{S}$, where the MMC problem is solved by a
mixed algorithm similar to Algorithm \ref{alg:m3c} (see \ref{sec:Optimization-procedure-for-mmc}).
\item PCCA+ \citep{deuflhard2005robust}, where the implementation details
are given in Appendix \ref{sec:Implementation-procedure-of-PCCA}.\end{enumerate}
\begin{rem}
$k$-medoids clustering and MMC can be viewed as geometric clustering
methods in metastability analysis.
\end{rem}

\begin{rem}
Considering the inherent randomness of the $k$-medoids algorithm,
here we perform the $k$-medoids clustering (or $k$-medoids clustering
in coarse graining steps of MMC, PCCA+ and \MMMC) in the following
way: Repeat the $k$-medoids algorithm $100$ times independently
with random initialization and pick the solution with the minimal
``within-cluster point scatter''.
\end{rem}

\subsection{Sequential unlabeled data}

Here we apply M\textsuperscript{3}C and other metastable state decomposition
methods to sequential data which are generated by using data sets
\textbf{letter}, \textbf{satellite}, \textbf{spambase}, \textbf{waveform}
and \textbf{segment} from the UCI machine learning repository \citep{Asuncion+Newman:2007}.
All data sets consist of multiple classes of instances and the pattern
of each instance is represented by a multidimensional vector of real-
or integer-valued features. A summary of all the data sets is in Table
\ref{tab:datasets}.

For each data set, we construct a sequence of unlabeled patterns as
follows:
\begin{enumerate}
\item Generate a reversible Markov chain $\{\bar{y}_{n}\}_{n=1}^{N}$ in
$\{1,\ldots,\kappa\}$, where $\kappa$ denotes the class number of
the data set, and the transition matrix $\mathbf{P}=[P_{ij}]=[\Pr(\bar{y}_{n+1}=j|\bar{y}_{n}=i)]\in\mathbb{R}^{\kappa\times\kappa}$
of $\{\bar{y}_{n}\}$ is given by
\begin{eqnarray*}
\mathbf{P} & = & \left[\begin{array}{cc}
0.97 & 0.03\\
0.03 & 0.97
\end{array}\right]\\
\mathbf{P} & = & \left[\begin{array}{ccc}
0.97 & 0.015 & 0.015\\
0.025 & 0.95 & 0.025\\
0.02 & 0.02 & 0.96
\end{array}\right]\\
\mathbf{P} & = & \left[\begin{array}{cccc}
0.9517 & 0.0198 & 0.0138 & 0.0147\\
0.0198 & 0.9509 & 0.0134 & 0.0159\\
0.0138 & 0.0134 & 0.9535 & 0.0193\\
0.0147 & 0.0159 & 0.0193 & 0.9501
\end{array}\right]
\end{eqnarray*}
for $\kappa=2,3,4$ repsectively.
\item For every $n=1,\ldots,N$, randomly select a pattern $\mathbf{x}_{n}$
with class label $\bar{y}_{n}$ from the data set without repetition.
(This step cannot be implemented if the element number of $\{n|\bar{y}_{n}=i\}$
is bigger than the data size of the $i$-th class in the data set
for some $i$. For such a case, we will repeat generating $\{\bar{y}_{n}\}$
until this step is feasible.)
\end{enumerate}
Since the self-transition probabilities in transition matrices $\mathbf{P}$
(i.e., the diagonal elements of $\mathbf{P}$) are all close to $1$,
a pattern sequence $\{\mathbf{x}_{n}\}$ generated as above can be
viewed as a metastable processes with each class being a metastable
state. Then we perform clustering of $\{\mathbf{x}_{n}\}$ by metastable
state decomposition after removing the class labels of $\{\mathbf{x}_{n}\}$.
Finally, the clustering accuracies of different methods are evaluated
by calculating classification errors on the training data set $\{\mathbf{x}_{n}\}$
and the testing data set consisting of all instances not in $\{\mathbf{x}_{n}\}$:
\begin{eqnarray}
\mathrm{err}_{\mathrm{train}} & = & \frac{1}{N}\sum_{n=1}^{N}1_{\bar{y}_{n}\neq y\left(\mathbf{x}_{n}\right)}\nonumber \\
\mathrm{err}_{\mathrm{test}} & = & \frac{1}{|\{\text{testing data}\}|}\sum_{\mathbf{x}\in\{\text{testing data}\}}1_{\bar{y}\left(\mathbf{x}\right)\neq y\left(\mathbf{x}\right)}
\end{eqnarray}
where $\bar{y}\left(\mathbf{x}\right)$ denotes the true class label
of $\mathbf{x}$, and $y\left(\mathbf{x}\right)$ denotes the predicted
label given by metastable state decomposition results.

Table \ref{tab:clustering-errors} summarizes clustering errors on
the various data sets with $N=1000$. It can be observed from the
table that both PCCA+ and \MMMC outperform the geometric clustering
methods, $k$-medoids and MMC, by utilizing the metastable structure
in the sequential, and \MMMC achieves the best clustering performance
among all the four methods. Moreover, considering that the clustering
given by PCCA+ depend on the space discretization results (see Section
\ref{sec:Introduction}), we report the clustering errors of PCCA+
with different numbers of discrete bins in this table. As can be seen,
for most data sets, except \textbf{segment}, either the finest discretization
(with $400$ bins) or the coarsest discretization (with $50$ bins)
cannot lead to the best PCCA+ clustering results, which demonstrates
the sensitivity of PCCA+ to the choice of space discretization.

\begin{table}
\centering\begin{threeparttable}\protect\caption{Summary of the data sets\label{tab:datasets}}

\begin{tabular}{|c|c|c|c|}
\hline 
 & size & pattern dimension & number of classes\tnote{a)} ~($\kappa$)\tabularnewline
\hline 
\hline 
\textbf{letter} & $1555$ & $17$ & $2$\tabularnewline
\hline 
\textbf{satellite} & $2236$ & $36$ & $2$\tabularnewline
\hline 
\textbf{spambase} & $4601$ & $57$ & $2$\tabularnewline
\hline 
\textbf{waveform} & $5000$ & $21$ & $3$\tabularnewline
\hline 
\textbf{segment} & $1320$ & $19$ & $4$\tabularnewline
\hline 
\end{tabular}

\begin{tablenotes}
\item [a)] For data sets which contain more than $\kappa$ classes, we only use their subsets consising of the first $\kappa$ classes.
\end{tablenotes}

\end{threeparttable}
\end{table}

\begin{table}
\protect\caption{Means and standard deviations of clustering errors (in percent) calculated
over $20$ independent experiments of various data sets\label{tab:clustering-errors}}

\resizebox{\textwidth}{!}{

\begin{tabular}{|c|c|c|c|c|c|c|c|c|}
\hline 
\multicolumn{2}{|c|}{} & $k$-medoids & MMC & $\begin{array}{c}
\text{PCCA+}\\
\text{(50 bins)}
\end{array}$ & $\begin{array}{c}
\text{PCCA+}\\
\text{(200 bins)}
\end{array}$ & $\begin{array}{c}
\text{PCCA+}\\
\text{(300 bins)}
\end{array}$ & $\begin{array}{c}
\text{PCCA+}\\
\text{(400 bins)}
\end{array}$ & \multicolumn{1}{c|}{\MMMC}\tabularnewline
\hline 
\hline 
\multirow{2}{*}{\textbf{letter}} & $\mathrm{err}_{\mathrm{train}}$ & $8.6306\pm3.0850$ & $7.0090\pm3.2752$ & $2.2450\pm1.0128$ & $0.5450\pm0.1849$ & $0.4450\pm0.2585$ & $0.5950\pm0.2837$ & $\mathbf{0.1100\pm0.0852}$\tabularnewline
\cline{2-9} 
 & $\mathrm{err}_{\mathrm{test}}$ & $8.6400\pm3.0025$ & $7.3600\pm3.6059$ & $2.7207\pm1.4567$ & $1.1171\pm0.5746$ & $0.7658\pm0.5871$ & $1.0270\pm0.6299$ & $\mathbf{0.2342\pm0.1963}$\tabularnewline
\hline 
\multirow{2}{*}{\textbf{satellite}} & $\mathrm{err}_{\mathrm{train}}$ & $2.3503\pm0.8560$ & $2.3058\pm0.8616$ & $0.7727\pm0.2905$ & $0.6600\pm0.2741$ & $0.7600\pm0.2604$ & $0.8600\pm0.2501$ & $\mathbf{0.0900\pm0.0912}$\tabularnewline
\cline{2-9} 
 & $\mathrm{err}_{\mathrm{test}}$ & $6.7950\pm1.3069$ & $6.6450\pm1.4365$ & $1.1950\pm0.3348$ & $0.7727\pm0.3457$ & $0.9102\pm0.3803$ & $1.0720\pm0.5436$ & \multicolumn{1}{c|}{$\mathbf{0.3196\pm0.1325}$}\tabularnewline
\hline 
\multirow{2}{*}{\textbf{spambase}} & $\mathrm{err}_{\mathrm{train}}$ & $36.8773\pm1.8360$ & $22.5923\pm9.7192$ & $23.6000\pm9.2388$ & $25.7450\pm13.4768$ & $22.7950\pm12.4952$ & $28.3671\pm8.2107$ & $\mathbf{8.5296\pm1.2141}$\tabularnewline
\cline{2-9} 
 & $\mathrm{err}_{\mathrm{test}}$ & $45.0150\pm4.2245$ & $23.7650\pm7.4026$ & $24.7362\pm7.8620$ & $27.1355\pm9.7611$ & $23.9572\pm8.3935$ & $30.2900\pm15.3507$ & $\mathbf{12.2350\pm1.0338}$\tabularnewline
\hline 
\multirow{2}{*}{\textbf{waveform}} & $\mathrm{err}_{\mathrm{train}}$ & $46.6212\pm2.8506$ & $34.1600\pm16.3716$ & $23.5425\pm2.6599$ & $21.5525\pm2.9225$ & $26.1075\pm8.1388$ & $30.3725\pm13.9271$ & $\mathbf{15.0550\pm7.1295}$\tabularnewline
\cline{2-9} 
 & $\mathrm{err}_{\mathrm{test}}$ & $47.4550\pm2.0623$ & $37.4900\pm16.0950$ & $27.2483\pm3.7676$ & $27.6700\pm4.5708$ & $33.6117\pm9.91157$ & $36.5450\pm13.2638$ & $\mathbf{22.2675\pm10.4104}$\tabularnewline
\hline 
\multirow{2}{*}{\textbf{segment}} & $\mathrm{err}_{\mathrm{train}}$ & $21.2417\pm11.3240$ & $8.1050\pm3.9986$ & $5.4500\pm1.0273$ & $2.6050\pm0.4729$ & $2.4850\pm0.8845$ & $2.3400\pm0.6652$ & $\mathbf{1.3750\pm0.3059}$\tabularnewline
\cline{2-9} 
 & $\mathrm{err}_{\mathrm{test}}$ & $31.9750\pm7.9013$ & $18.3234\pm18.1635$ & $14.9662\pm12.2071$ & $10.2918\pm12.3009$ & $10.5916\pm12.2503$ & $9.1230\pm9.8576$ & $\mathbf{2.5392\pm1.6581}$\tabularnewline
\hline 
\end{tabular}

}
\end{table}

\subsection{Diffusion models}

We now consider two examples of time-reversible diffusion processes,
denoted by Model I and Model II, which can be described by two dimensional
Fokker-Planck equations (see \ref{sec:Description-of-ModelI-II} for
details).

For metastable state decomposition of Model I, we generate $10$ trajectories
by $10$ independent simulations with time length $80$, sample interval
$\Delta t=0.2$ and initial states $\mathbf{x}_{0}$ distributed according
to an uniform distribution on $\left[-1.5,1.5\right]^{2}$. Fig.~\ref{fig:Illustration-of-Model-I}
shows the potential energy and simulation trajectories, where $\mathbf{x}_{t}=(x_{t}^{(1)},x_{t}^{(2)})$
denotes the system state at time $t$, the potential function $V\left(\mathbf{x}\right)$
is defined by $\pi\left(\mathbf{x}\right)\propto\exp\left(-V\left(\mathbf{x}\right)\right)$
and $\pi\left(\mathbf{x}\right)$ is the equilibrium distribution
$\lim_{t\to\infty}p\left(\mathbf{x}_{t}=\mathbf{x}\right)$. We can
observe that Model I has $6$ potential wells, and the energy barrier
between the ``upper'' wells (with $x^{(2)}>0$) and the ``lower''
ones (with $x^{(2)}<0$) can be easily crossed. Therefore, the phase
space of Model I can be decomposed into $3$ metastable states with
each one containing an upper and a lower potential well. Fig.~\ref{fig:Decomposition-results-of-Model-I}
displays the decomposition results of all four methods with $\kappa=3$
by using the trajectory data shown in Fig.~\ref{fig:Simulation-data-modelI},
where the bin number of PCCA+ is set to be $10$. It is obvious that
\MMMC and PCCA+ accurately identify the three metastable states,
while the macrostates given by $k$-medoids and MMC do not exhibit
strong metastability although the decomposition results of the latter
two methods are ``reasonable'' in the sense of geometric clustering.
This shows the limitation of geometric clustering methods in the case
that the spatial structure of metastable states does not only depend
on the shape of equilibrium state distribution function.

Note that the decomposition result displayed in Fig.~\ref{fig:MMC-Model-I}
might not be the global optima of the MMC problem. It is natural to
ask if MMC can correctly find the metastable states by choosing a
better initial solution for the local search procedure. In order to
answer this question, we solve the MMC problem by the local search
algorithm in \citep{zhang2009maximum} starting from the decomposition
provided by \MMMC. Fig.~\ref{fig:Decomposition-results-mmc-m3cinit}
shows the corresponding decomposition result, which is very similar
to that obtained by \MMMC and PCCA+. However, the objective function
value of this decomposition in the MMC problem is $0.0652$, which
is larger than the objective function value $0.0496$ of the decomposition
shown in Fig.~\ref{fig:MMC-Model-I}. This indicates that MMC prefer
the decomposition in Fig.~\ref{fig:MMC-Model-I} to that in Fig.~\ref{fig:Decomposition-results-mmc-m3cinit},
although the latter one is better for metastability analysis.

Moreover, we also perform the local search algorithm to solve the
\MMMC problem with random initialization, and Fig.~\ref{fig:Decomposition-results-m3c-randominit-Model-I}
plots the optimization result. As observed from the figure, the randomly
generated initial solution leads to the algorithm getting stuck in
a local optimum. (The optimal objective function values of the \MMMC
problem obtained by the local search with random initialization and
the mixed-algorithm proposed in Section \ref{sub:Full-description-of-algorithm}
are $0.1553$ and $0.0967$ separately.) It shows that the local search
algorithm proposed in Section \ref{sub:Local-search} is sensitive
to initial conditions, and some heuristic method (e.g., the global
search algorithm presented in this paper) is needed to provide a satisfactory
initial solution.

\begin{figure}
\subfloat[Potential function]{\begin{centering}
\includegraphics[width=0.45\textwidth]{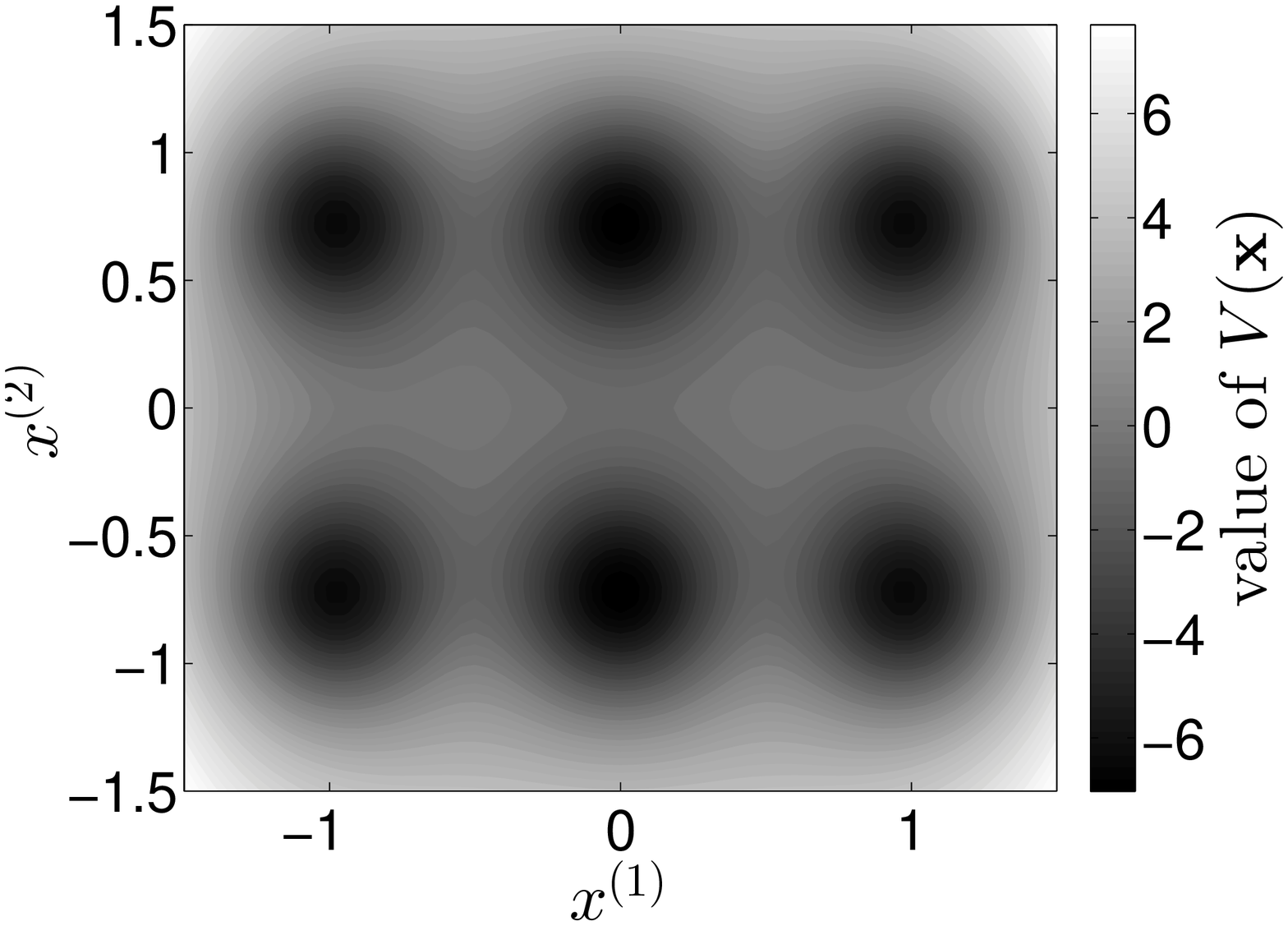}
\par\end{centering}

}\hfill{}\subfloat[Trajectory data, where dotted lines represent trajectories of $\mathbf{x}_{t}$
and circles denote data points sampled from trajectories.\label{fig:Simulation-data-modelI}]{\begin{centering}
\includegraphics[width=0.45\textwidth]{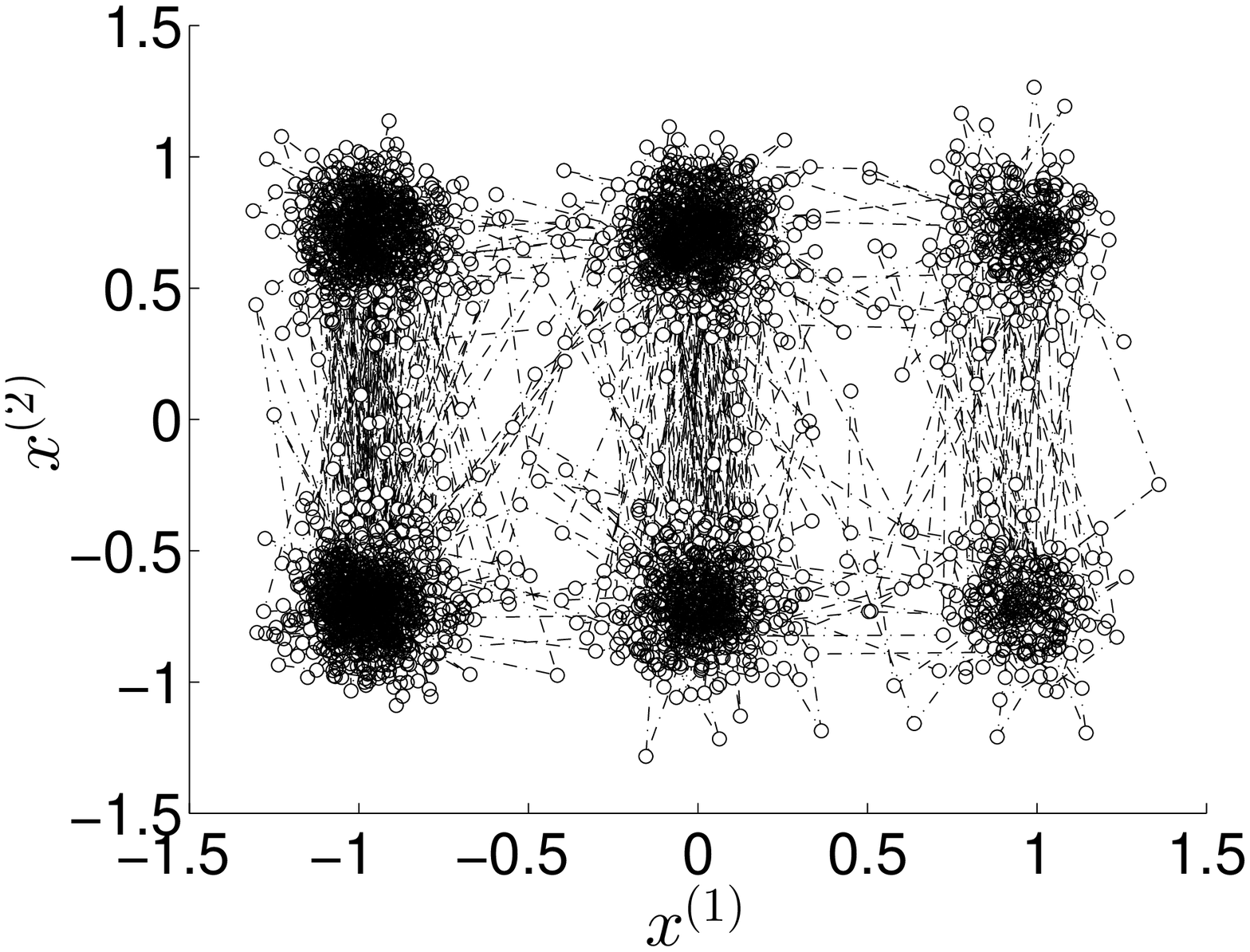}
\par\end{centering}

}

\protect\caption{Illustration of Model I.\label{fig:Illustration-of-Model-I}}
\end{figure}

\begin{figure}
\subfloat[\MMMC]{\begin{centering}
\includegraphics[width=0.45\textwidth]{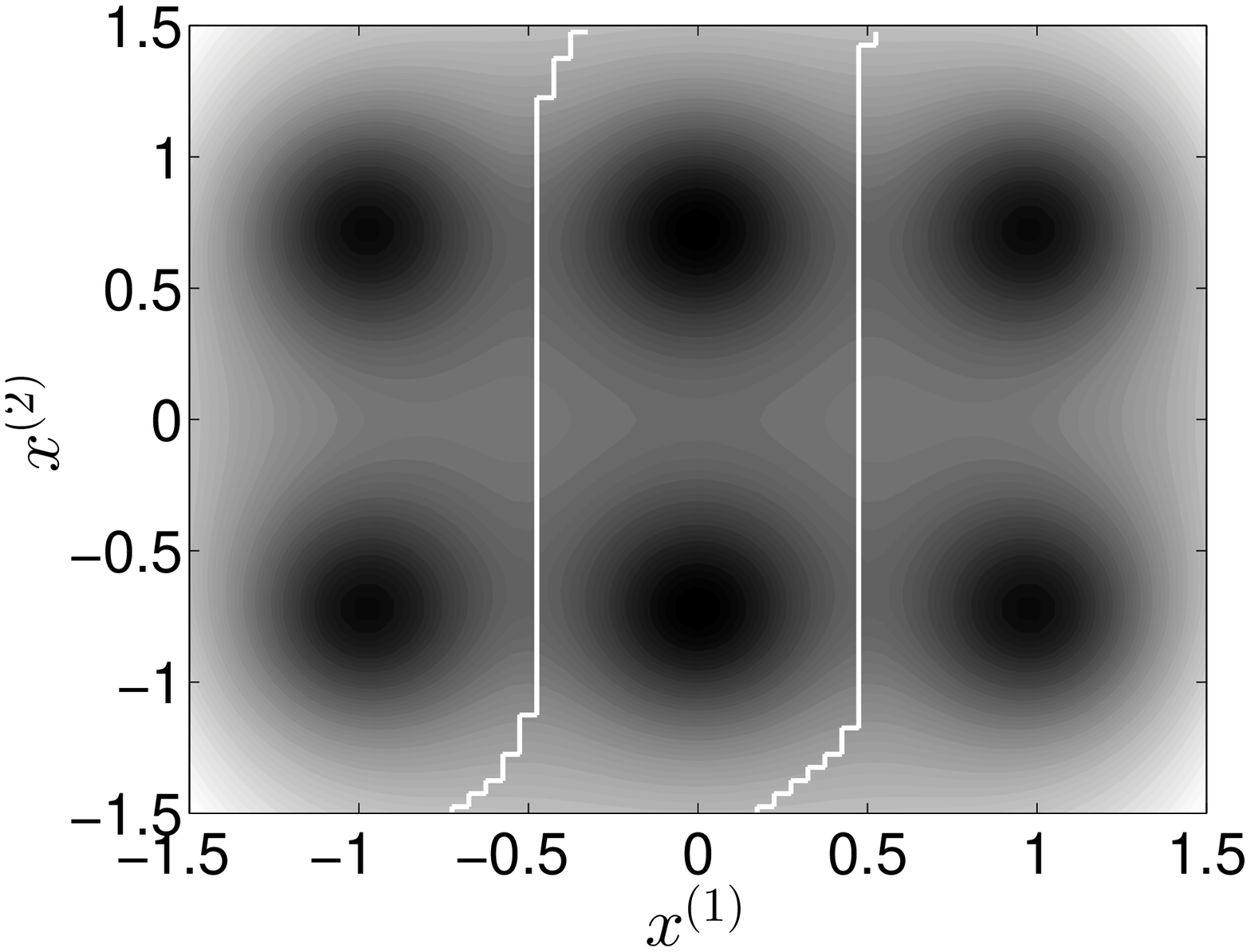}
\par\end{centering}

}\hfill{}\subfloat[$k$-medoids]{\begin{centering}
\includegraphics[width=0.45\textwidth]{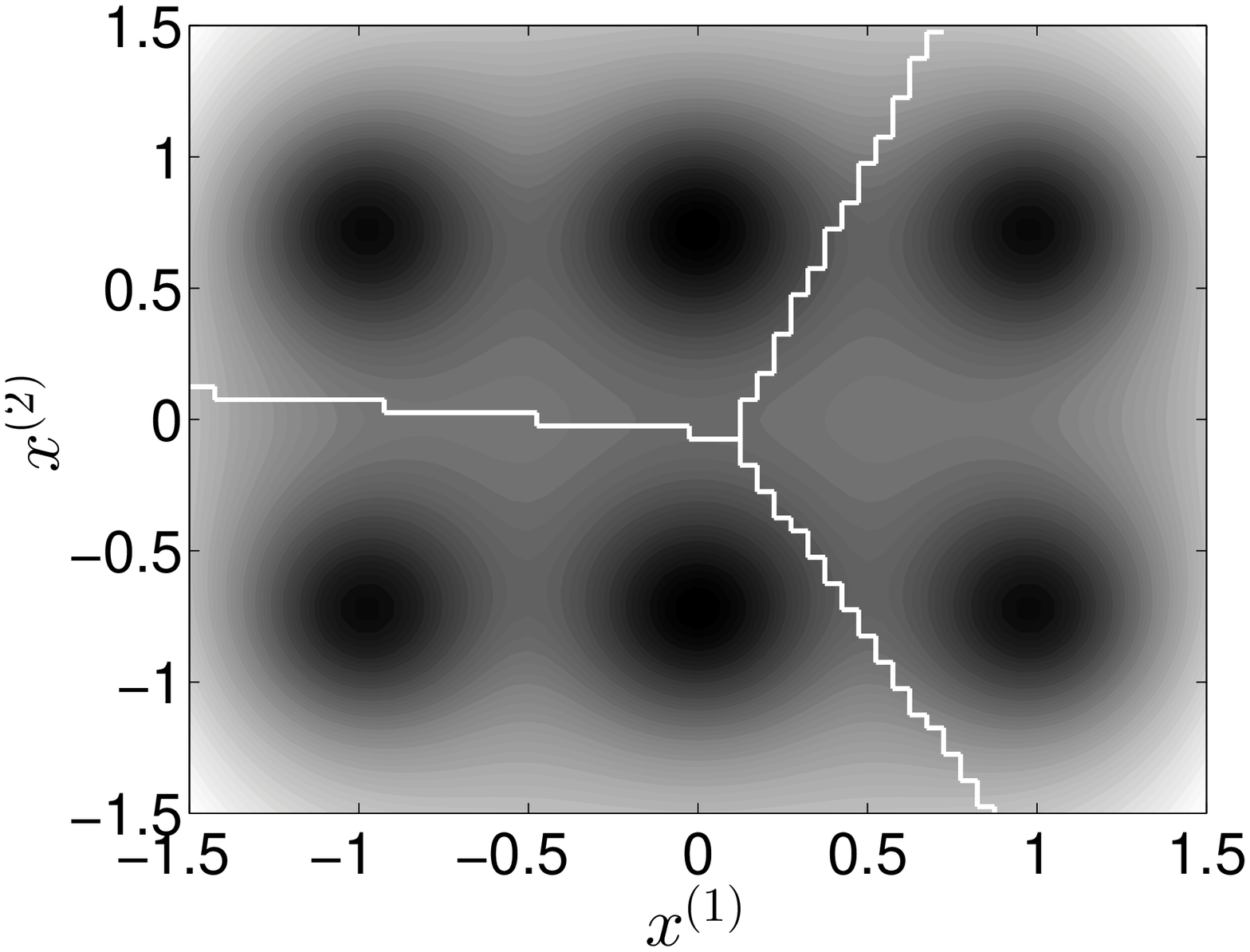}
\par\end{centering}

}

\subfloat[MMC\label{fig:MMC-Model-I}]{\begin{centering}
\includegraphics[width=0.45\textwidth]{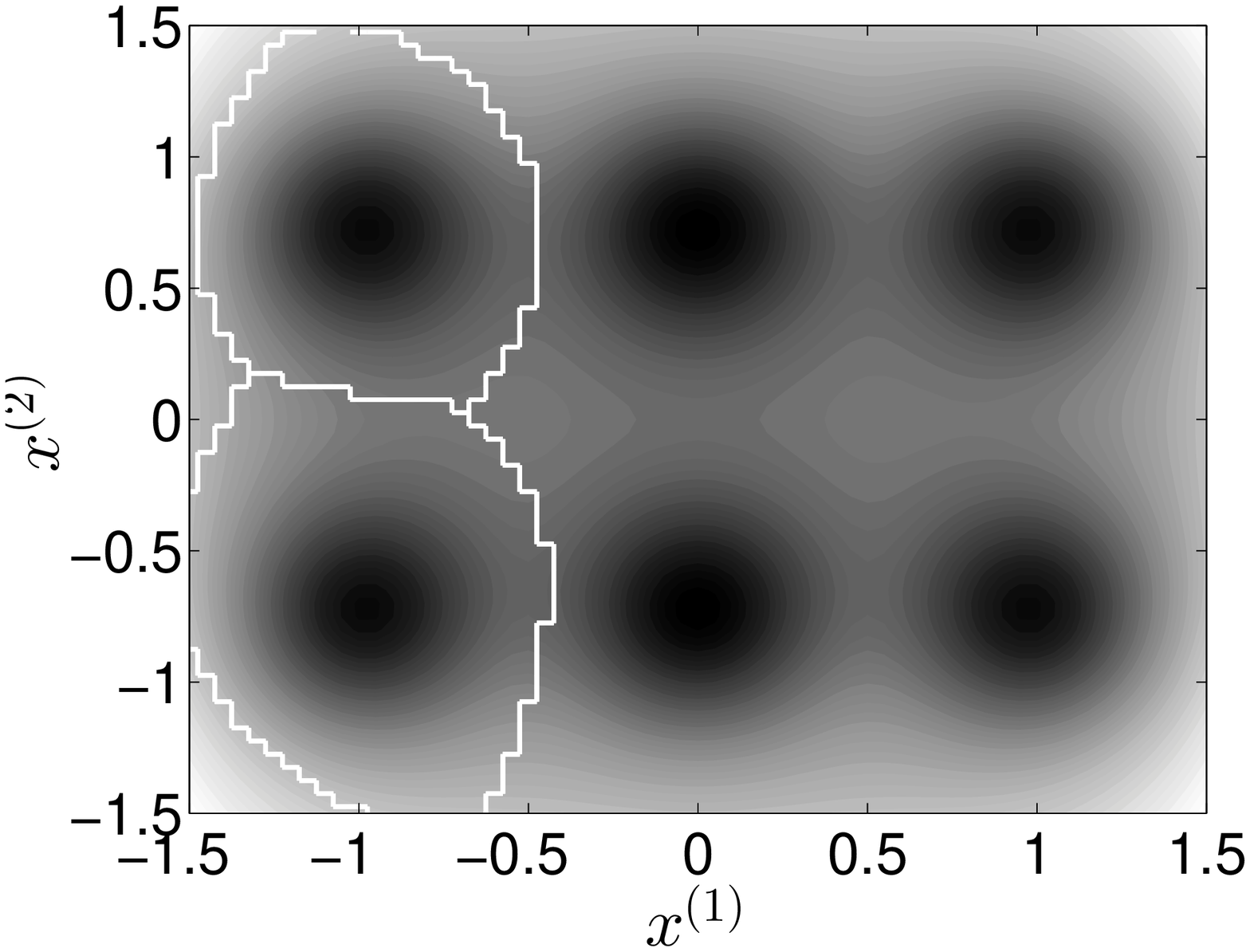}
\par\end{centering}

}\hfill{}\subfloat[PCCA+ (with 10 bins)]{\begin{centering}
\includegraphics[width=0.45\textwidth]{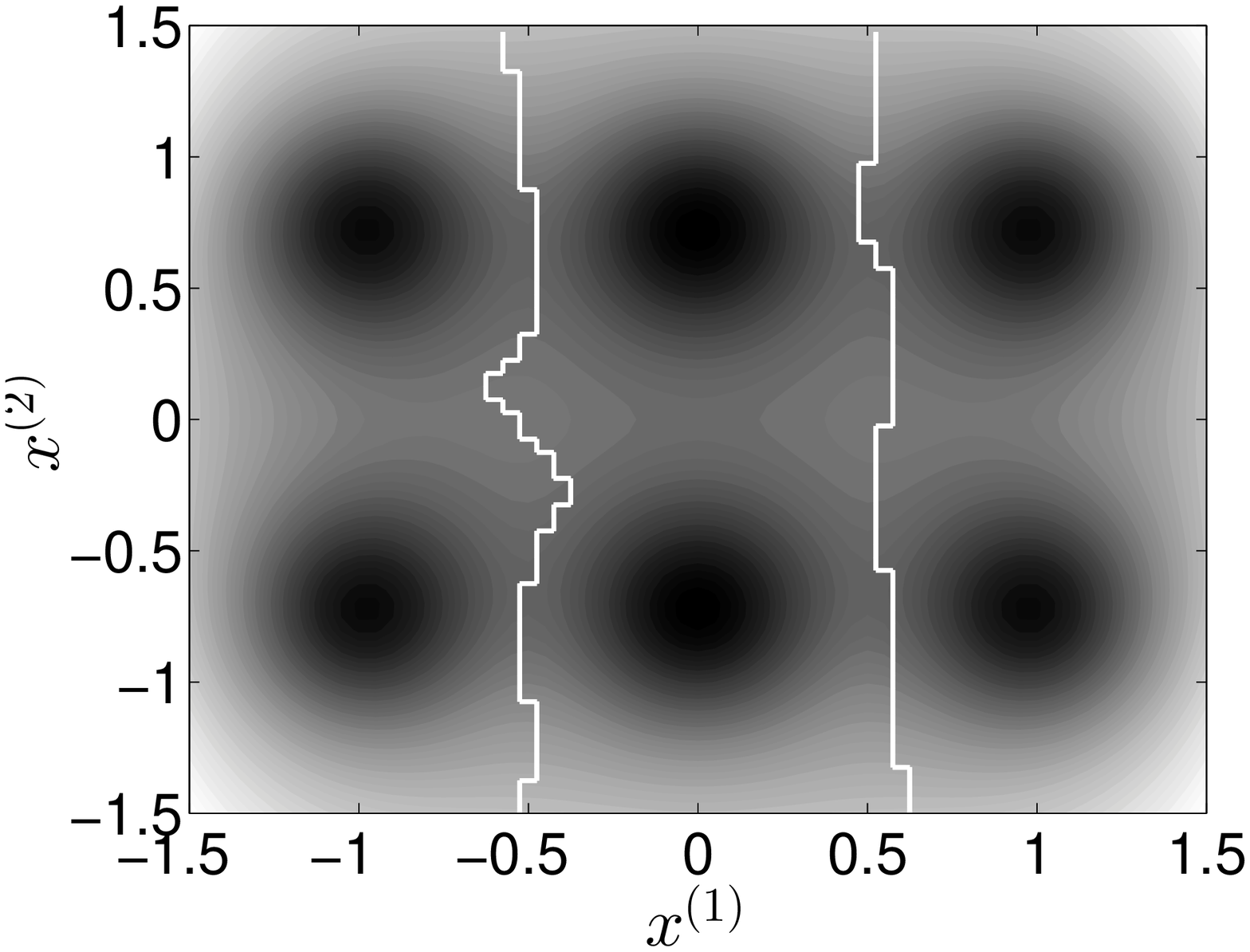}
\par\end{centering}

}

\protect\caption{Decomposition results of Model I, where white lines represent boundaries
between macrostates. The boundaries are computed by the finite element
method with mesh size $0.05\times0.05$.\label{fig:Decomposition-results-of-Model-I}}
\end{figure}

\begin{figure}
\begin{centering}
\includegraphics[width=0.45\textwidth]{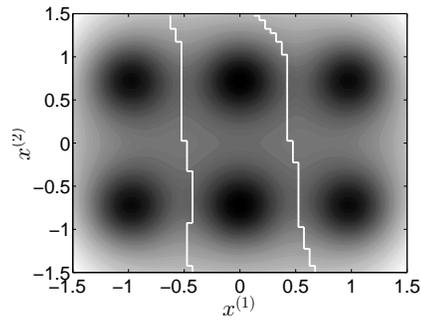}
\par\end{centering}

\protect\caption{Decomposition result of Model I obtained by the local search based
MMC with initial solution given by \MMMC.\label{fig:Decomposition-results-mmc-m3cinit}}

\end{figure}

\begin{figure}
\begin{centering}
\includegraphics[width=0.45\textwidth]{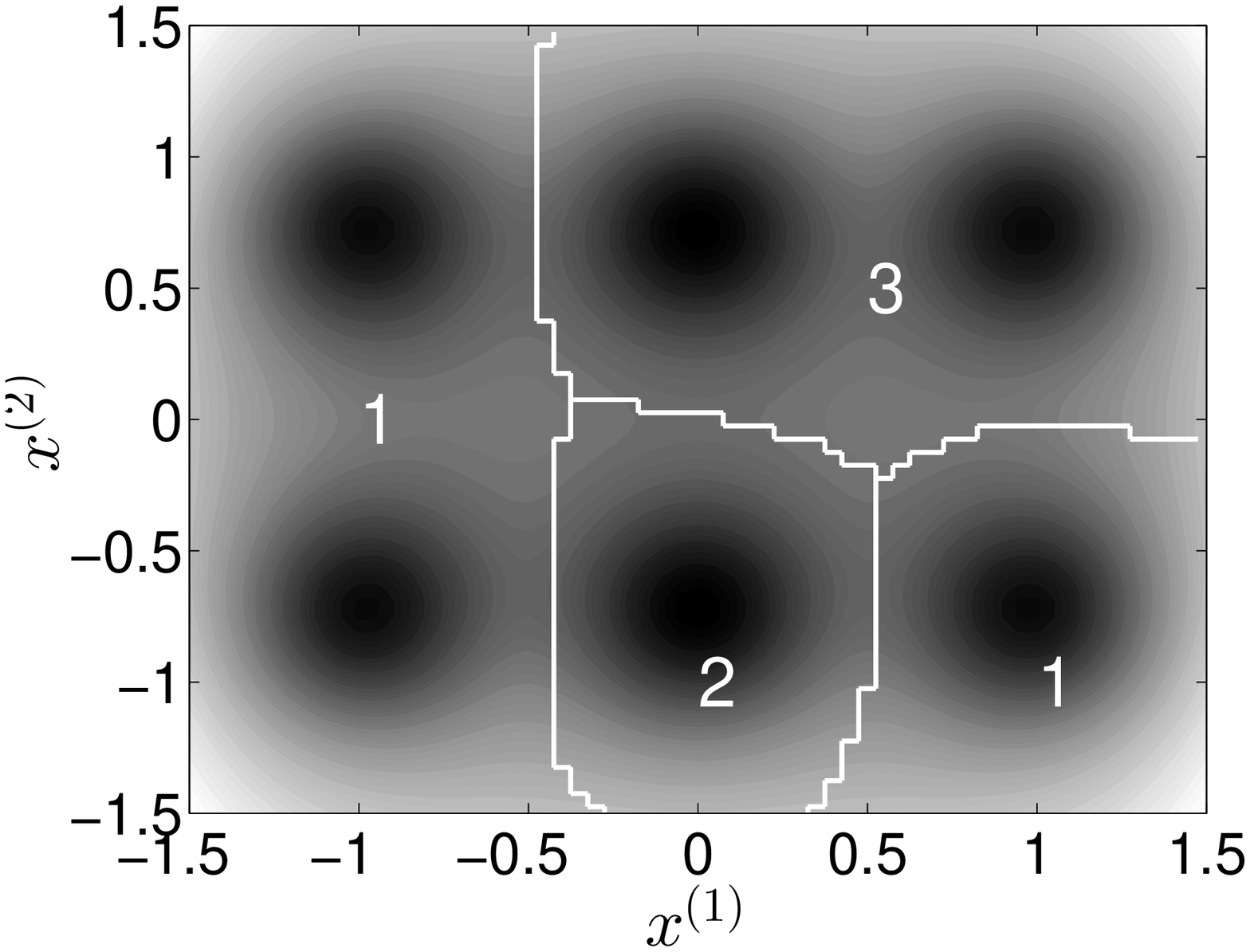}
\par\end{centering}

\protect\caption{Decomposition result of Model I obtained by the local search based
\MMMC with random initialization.\label{fig:Decomposition-results-m3c-randominit-Model-I}}
\end{figure}

The potential energy of Model II is shown in Fig.~\ref{fig:Model-II-potential}.
We generate $50$ trajectories in the state space of Model II in order
to test the metastable state decomposition methods (see Fig~\ref{fig:Model-II-trajectory}),
where the time length of each trajectory is $1$, the sample interval
is $\Delta t=0.02$ and the initial states are uniformly sampled from
$\left[-2,2\right]^{2}$. It is easy to see that Model II has three
metastable states formed by the three potential wells.

Fig.~\ref{fig:Decomposition-results-of-Model-II} summarizes the
decomposition results with $\kappa=3$. Figs.~\ref{fig:decomposition-kmedoids-modelII}
and \ref{fig:decomposition-mmc-modelII} show that $k$-medoids and
MMC fail to identify the three metastable states in Model II due to
the large difference between the empirical data distribution and the
equilibrium distribution. It is interesting to see from Fig.~\ref{fig:decomposition-pcca-modelII}
that PCCA+ (with $20$ discrete bins) also gives an undesired decomposition.
We now analyze why PCCA+ fails in this example. Note that there is
a low energy barrier at the center of the right potential well (see
the rectangular region with the dashed line in Fig.~\ref{fig:Illustration-of-Model-II}
and its magnified picture shown in Fig.~\ref{fig:pcca-detail-results-Model-II}),
which can be easily crossed for equilibrium simulations. But the trajectory
data used in the above is very far away from the global equilibrium
due to short simulation lengths. So only a small number of transitions
between discrete bins of PCCA+ are observed around this region, and
the PCCA+ assigns these bins to different metastable states. In all
the four methods, \MMMC is the only method that get correct metastable
state decomposition by utilizing both the geometric and the dynamical
information, and it avoids splitting the right potential well into
two metastable states as PCCA+ because such a decomposition leads
to a small margin between metastable boundaries.

\begin{figure}
\subfloat[Potential function\label{fig:Model-II-potential}]{\begin{centering}
\includegraphics[width=0.45\textwidth]{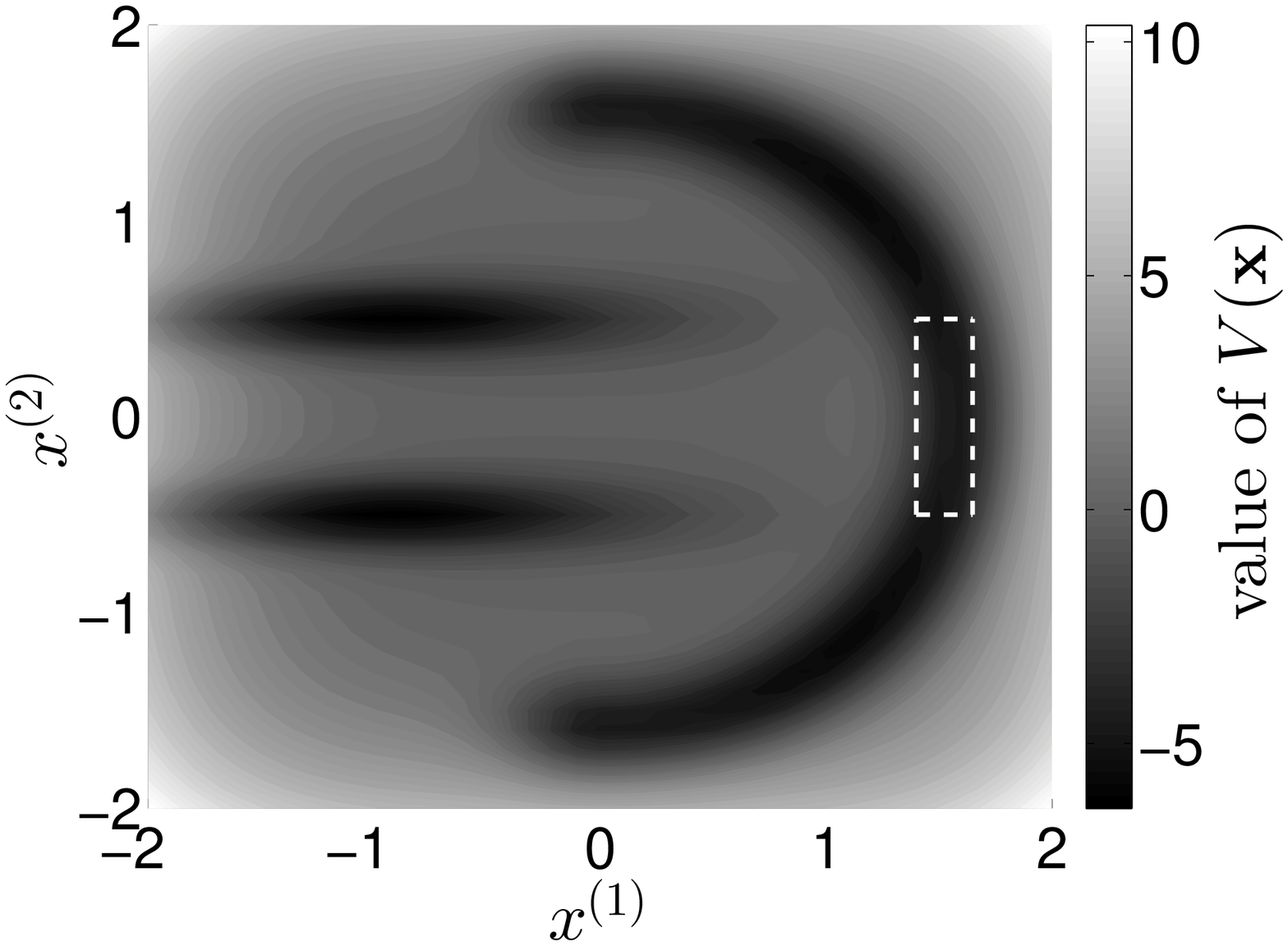}
\par\end{centering}

}\hfill{}\subfloat[Trajectory data, where dotted lines represent trajectories of $\mathbf{x}_{t}$
and circles denote data points sampled from trajectories.\label{fig:Model-II-trajectory}]{\begin{centering}
\includegraphics[width=0.45\textwidth]{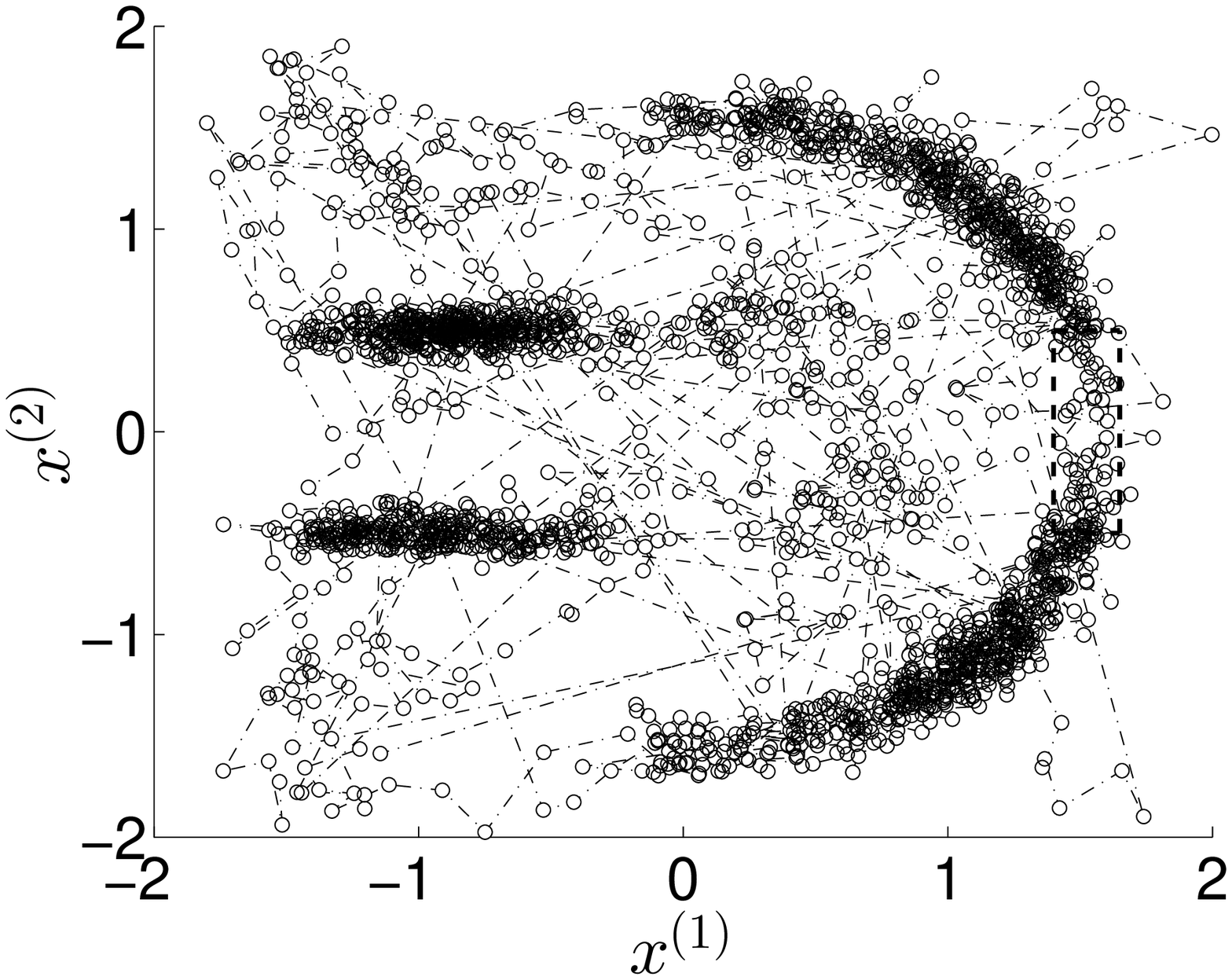}
\par\end{centering}

}

\protect\caption{Illustration of Model II, where the rectangular with dashed lines
shows the region $[1.4,1.65]\times[-0.5,0.5]$ of a low energy barrier
within the right potential well.\label{fig:Illustration-of-Model-II}}
\end{figure}

\begin{figure}
\subfloat[\MMMC\label{fig:decomposition-m3c-modelII}]{\begin{centering}
\includegraphics[width=0.45\textwidth]{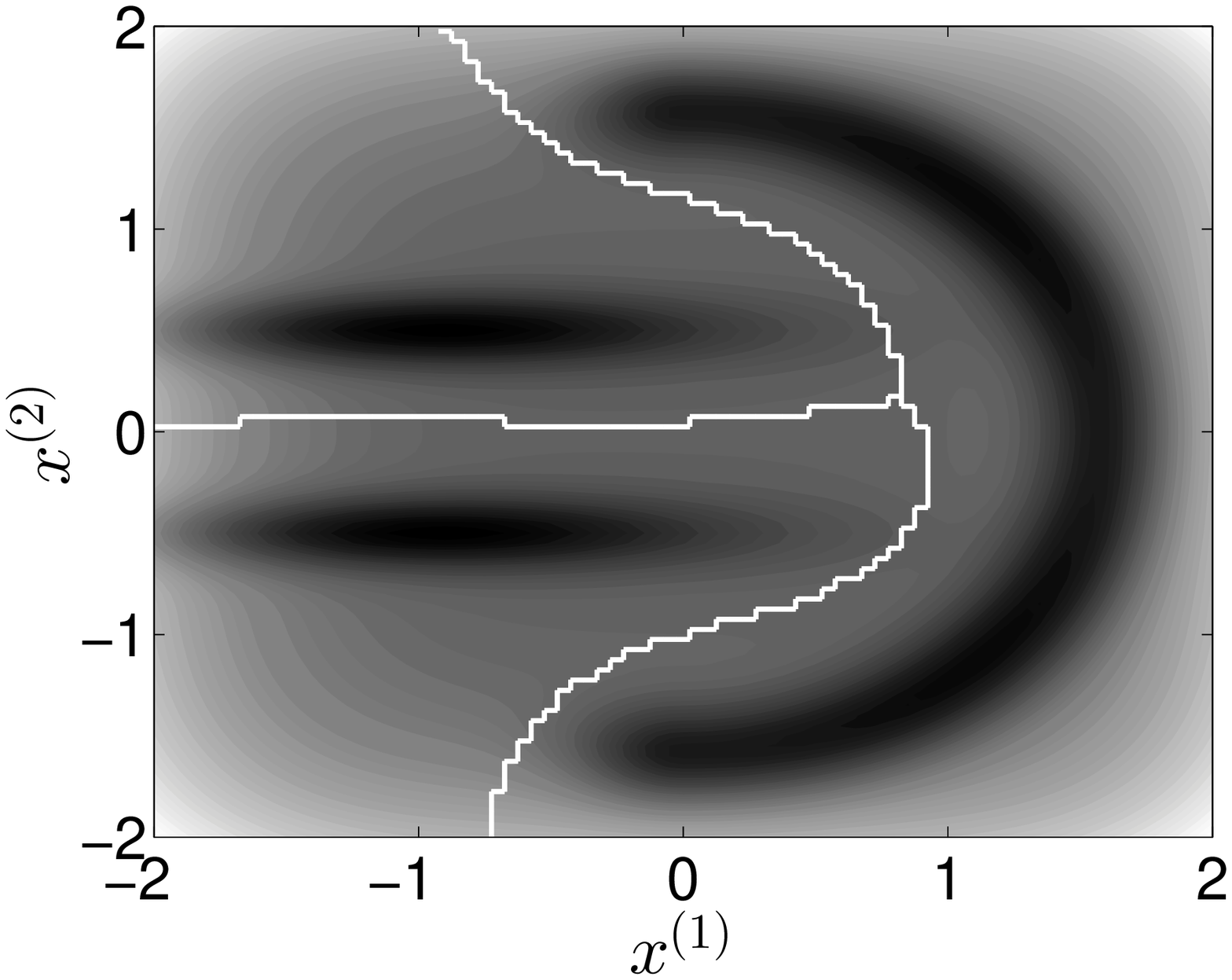}
\par\end{centering}

}\hfill{}\subfloat[$k$-medoids\label{fig:decomposition-kmedoids-modelII}]{\begin{centering}
\includegraphics[width=0.45\textwidth]{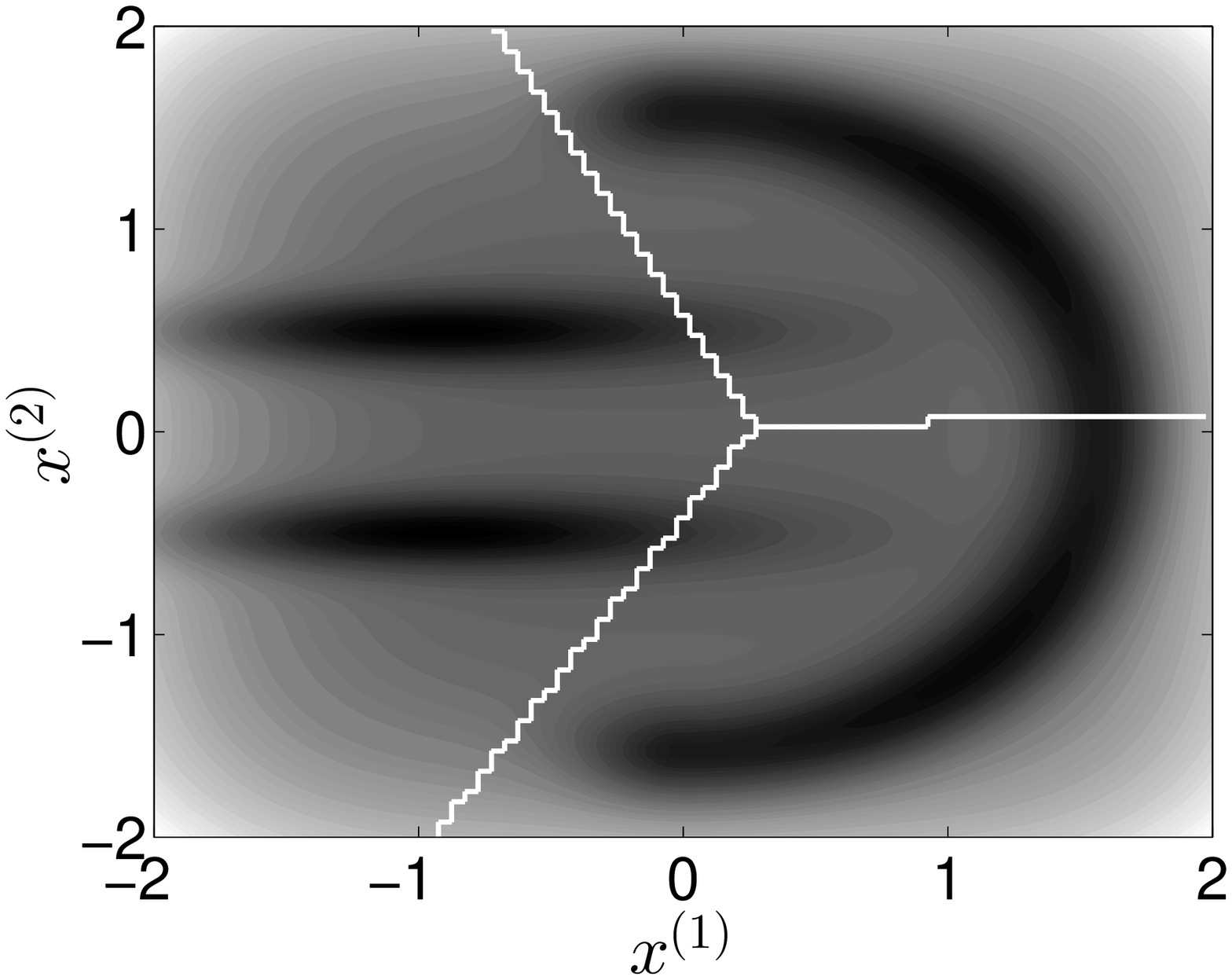}
\par\end{centering}

}

\subfloat[MMC\label{fig:decomposition-mmc-modelII}]{\begin{centering}
\includegraphics[width=0.45\textwidth]{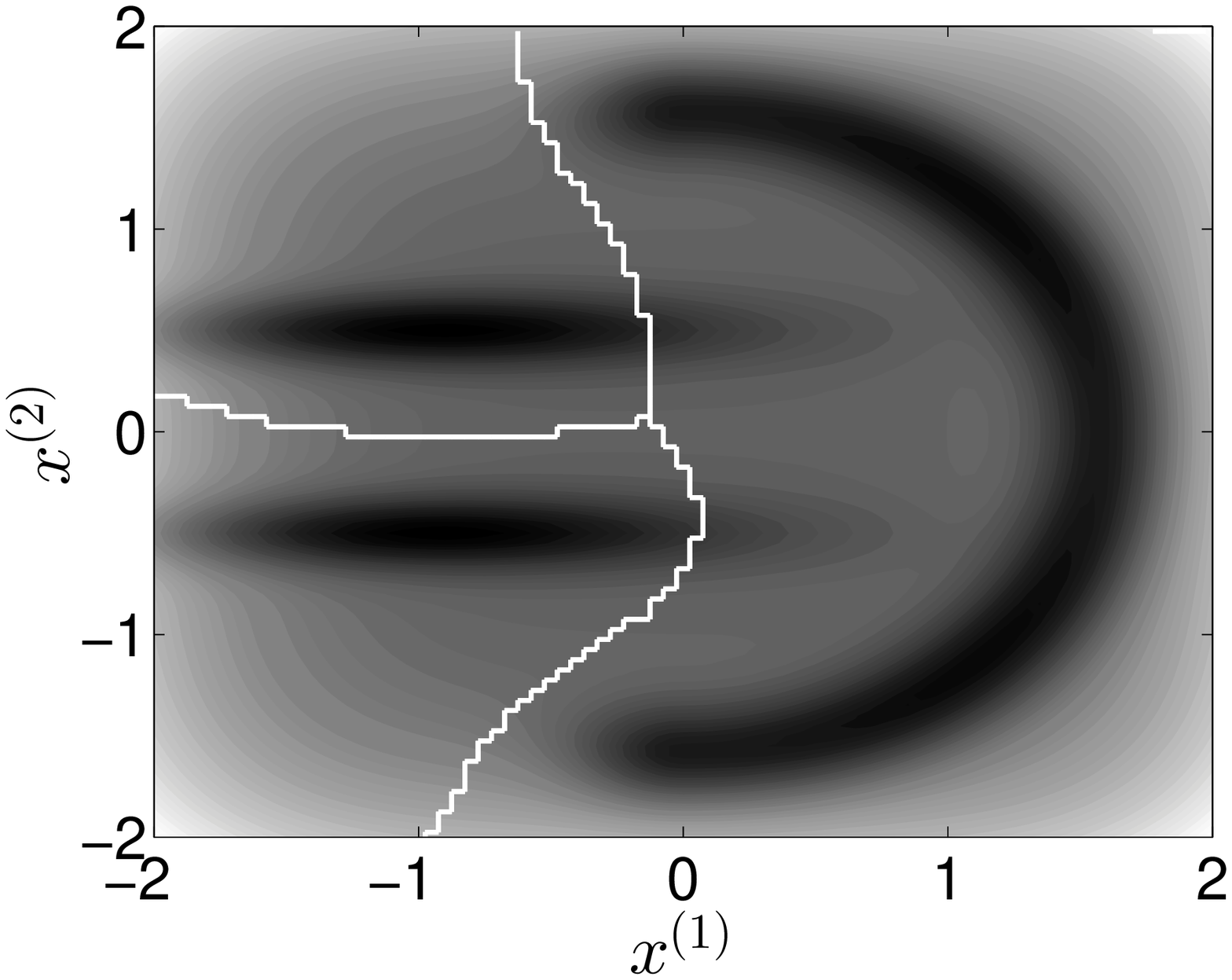}
\par\end{centering}

}\hfill{}\subfloat[PCCA+ (with 20 bins)\label{fig:decomposition-pcca-modelII}]{\begin{centering}
\includegraphics[width=0.45\textwidth]{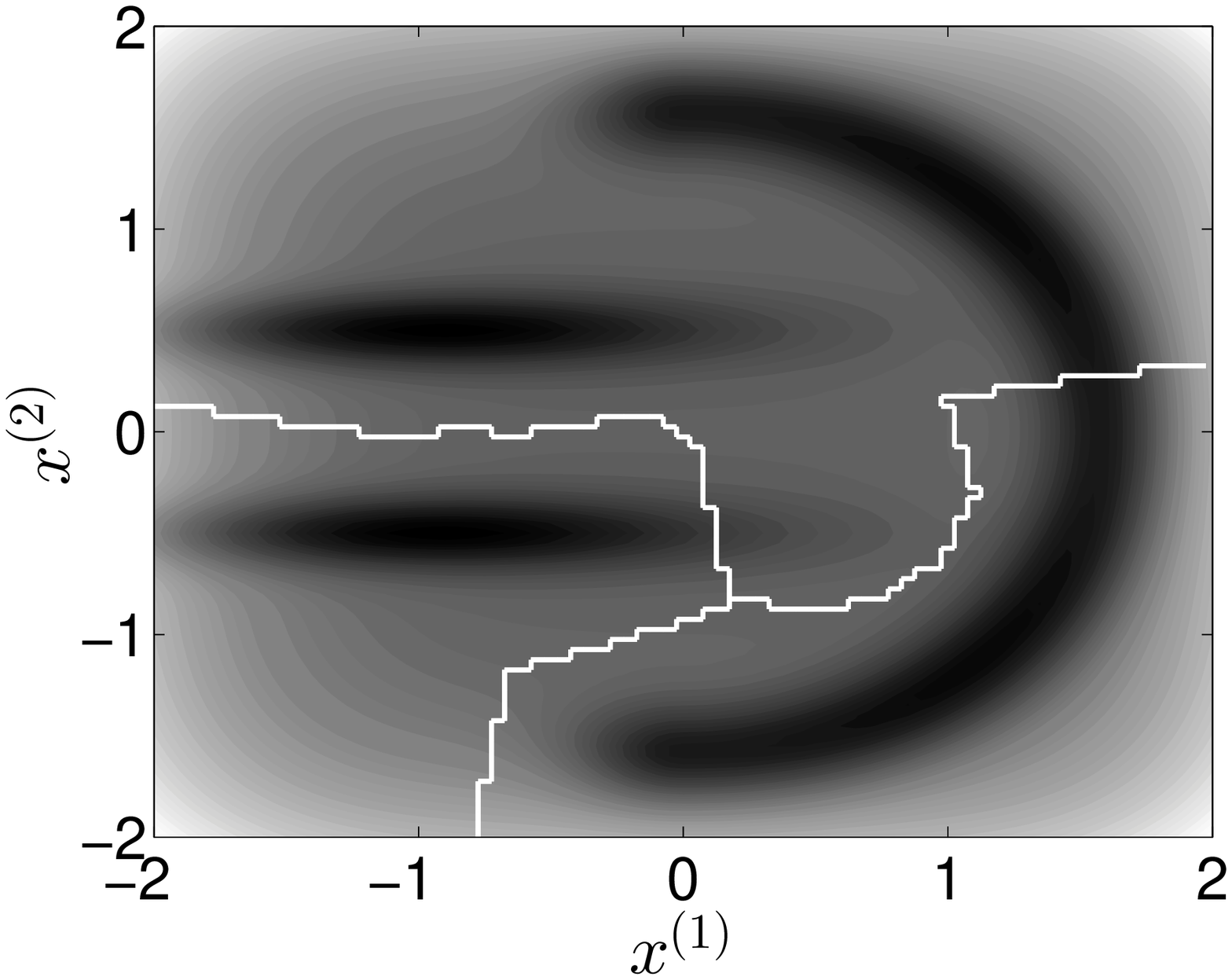}
\par\end{centering}

}

\protect\caption{Decomposition results of Model II, where white lines represent boundaries
between macrostates. The boundaries are computed by the finite element
method with mesh size $0.05\times0.05$.\label{fig:Decomposition-results-of-Model-II}}
\end{figure}

\begin{figure}
\begin{centering}
\includegraphics[width=0.45\textwidth]{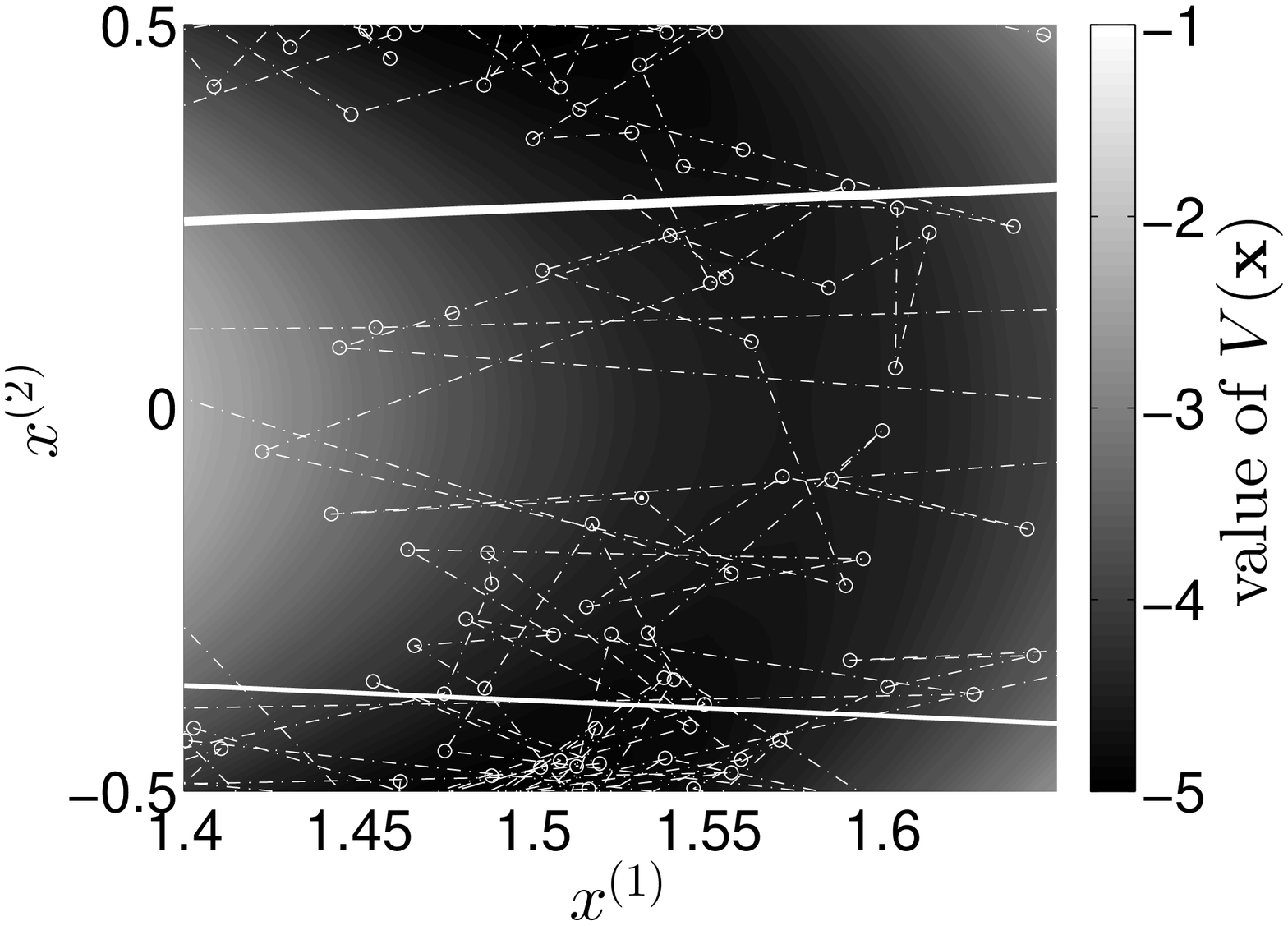}
\par\end{centering}

\protect\caption{PCCA+ decomposition result of Model II in the area $[1.4,1.65]\times[-0.5,0.5]$,
where where dotted lines represent trajectories, circles denote data
points sampled from trajectories, solid lines show boundaries of discrete
bins obtained by space discretization, and the upper boundary is chosen
to be the metastable state boundary by the PCCA+ algorithm.\label{fig:pcca-detail-results-Model-II}}
\end{figure}

Finally, we repeat the above numerical experiments of Model I and
Model II for $20$ times and utilize the following quantities to evaluate
and compare the performance of different decomposition methods quantitatively:
\begin{enumerate}
\item $Q=\sum_{k=1}^{\kappa}P_{kk}\left(\Delta t\right)$, where $\Delta t$
denotes the sample interval and
\begin{eqnarray}
P_{ij}\left(\tau\right) & = & \lim_{t\to\infty}\Pr(\mathbf{x}_{t+\tau}\in\text{metastable state }j|\mathbf{x}_{t}\in\nonumber \\
 &  & \text{metastable state }i)
\end{eqnarray}
denotes the transition probability between metastable states with
lagtime $\tau$. It is clear that $Q$ can measure the metastability
of a system and a decomposition with strongly metastablity will result
in a large $Q$ close to $\kappa$ \citep{chodera2007automatic}.
\item Implied timescale $\mathrm{ITS}_{i}\left(\tau\right)=-\tau/\ln\lambda_{i}\left(\tau\right)$
with $i>1$, where $\lambda_{i}\left(\tau\right)$ denotes the $i$-th
largest eigenvalue of transition probability matrix $\mathbf{P}\left(\tau\right)=[P_{ij}\left(\tau\right)]$.
(The first implied timescale $\mathrm{ITS}_{i}\left(\tau\right)\equiv\infty$)
It can be proved that the value of $\mathrm{ITS}_{i}\left(\tau\right)$
is a constant independent of $\tau$ and equal to the dominant relaxation
timescales of the original system if the transitions between metastable
states are exactly Markovian \citep{elmer2005foldamerII,noe2013projected}.
Thus, we can check if a Markov chain on metastable states can accurately
approximate the system dynamics through comparing implied timescales
with different $\tau$.\end{enumerate}
\begin{rem}
We run a long simulation with time length $10^{4}$ to estimate values
of $P_{ij}\left(\tau\right)$ for each model. Furthermore, for convenience
of comparison, we use the finely discretized Markov state models with
$50$ states to estimate the true relaxation timescales of Model I
and Model II. The detailed estimation algorithms are given in \citep{noe2007hierarchical}.
\end{rem}
Table \ref{tab:Q-diffusion-model} and Figs.~\ref{fig:its-Model-I}
and \ref{fig:its-Model-II} summarize the values of $Q$ and implied
timescales given by different decomposition methods, including the
PCCA+ method with different bin numbers. The table and figures also
demonstrate the superior performance of M\textsuperscript{3}C. It
is worth pointing out that in $9$ out of $20$ experiments of Model
II, PCCA+ wrongly decomposes the right potential well into two metastable
states (with any bin number), whereas M\textsuperscript{3}C gives
the ``ideal'' decomposition in all the $20$ experiments. Moreover,
as can be seen from the figures, the implied timescales obtained from
M\textsuperscript{3}C converge fast and are very close to the relaxation
timescales estimated by finely discretized Markov state models, which
implies that the essential dynamical properties of Model I and Model
II can be accurately captured by $3$-state Markov models using the
metastable states identified by M\textsuperscript{3}C.

\begin{table}
\protect\caption{Means and standard deviations of $Q$ values calculated over $20$
independent experiments of Models I and II\label{tab:Q-diffusion-model}}

\resizebox{\textwidth}{!}{

\begin{tabular}{|c|c|c|c|c|c|c|c|c|}
\hline 
\multicolumn{1}{|c|}{} & $k$-medoids & MMC & $\begin{array}{c}
\text{PCCA+}\\
\text{(6 bins)}
\end{array}$ & $\begin{array}{c}
\text{PCCA+}\\
\text{(10 bins)}
\end{array}$ & $\begin{array}{c}
\text{PCCA+}\\
\text{(20 bins)}
\end{array}$ & $\begin{array}{c}
\text{PCCA+}\\
\text{(30 bins)}
\end{array}$ & $\begin{array}{c}
\text{PCCA+}\\
\text{(40 bins)}
\end{array}$ & \multicolumn{1}{c|}{\MMMC}\tabularnewline
\hline 
\hline 
Model I & $2.8138\pm0.0317$ & $2.8142\pm0.0244$ & $2.9603\pm0.0170$ & $2.9637\pm0.0011$ & $2.9628\pm0.0017$ & $2.9619\pm0.0027$ & $2.9613\pm0.0021$ & $\mathbf{2.9641\pm0.0005}$\tabularnewline
\hline 
Model II & $2.9579\pm0.0031$ & $2.9832\pm0.0116$ & $2.9681\pm0.0211$ & $2.9666\pm0.0161$ & $2.9676\pm0.0173$ & $2.9655\pm0.0187$ & $2.9600\pm0.0285$ & $\mathbf{2.9882\pm0.0024}$\tabularnewline
\hline 
\end{tabular}

}
\end{table}

\begin{figure}
\subfloat[]{\begin{centering}
\includegraphics[width=0.45\textwidth]{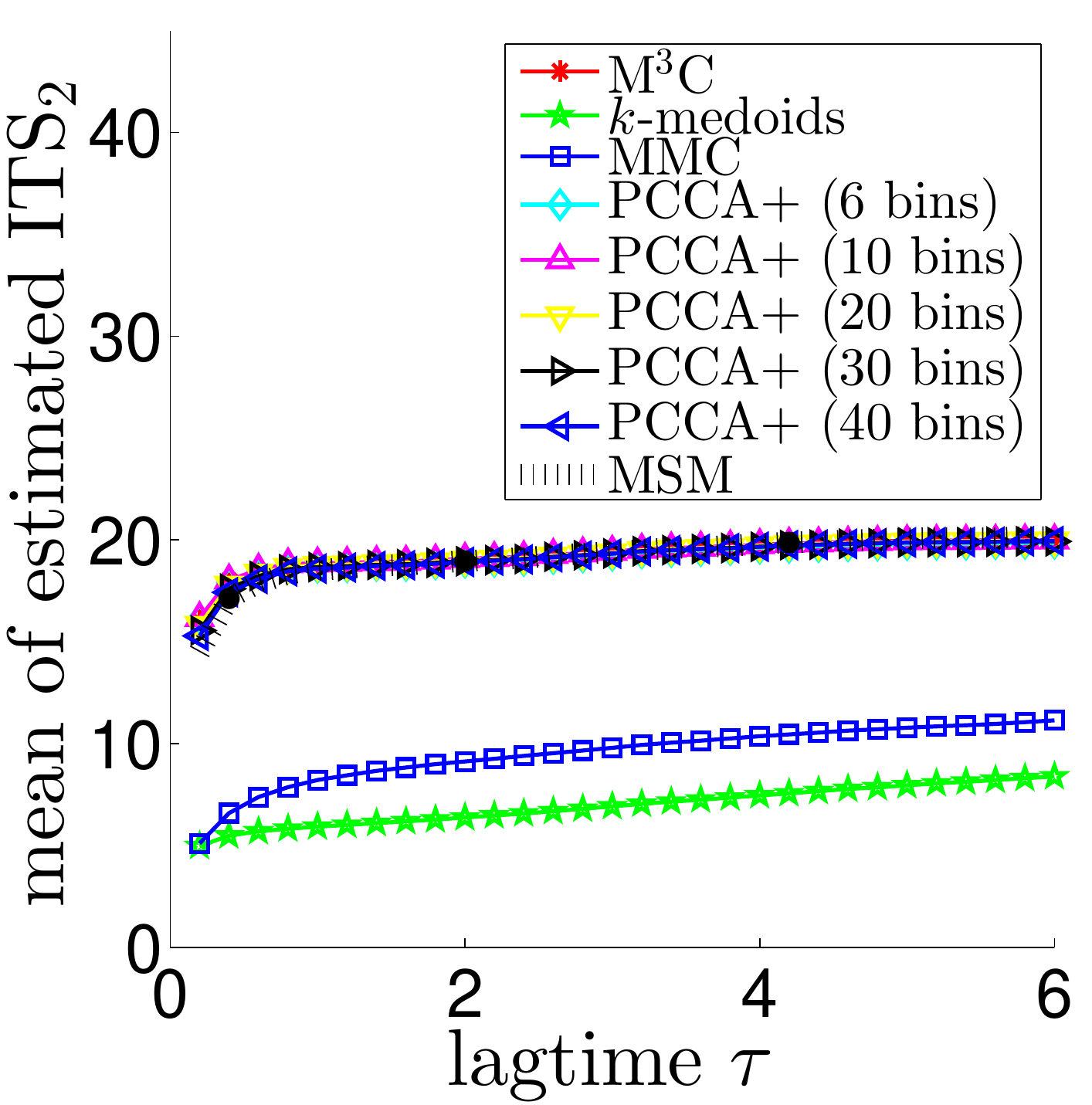}
\par\end{centering}

}\hfill{}\subfloat[]{\begin{centering}
\includegraphics[width=0.45\textwidth]{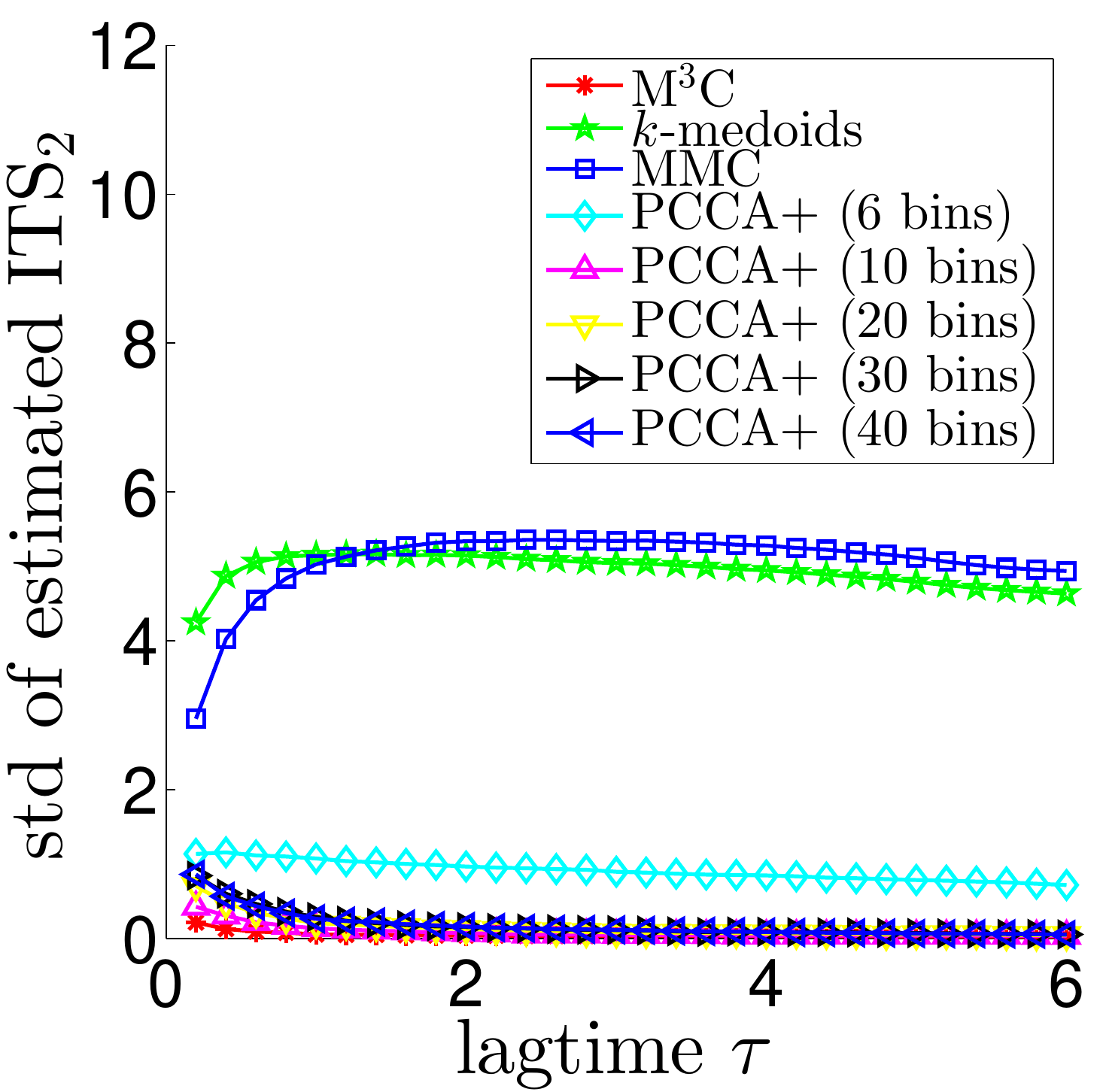}
\par\end{centering}

}

\subfloat[]{\begin{centering}
\includegraphics[width=0.45\textwidth]{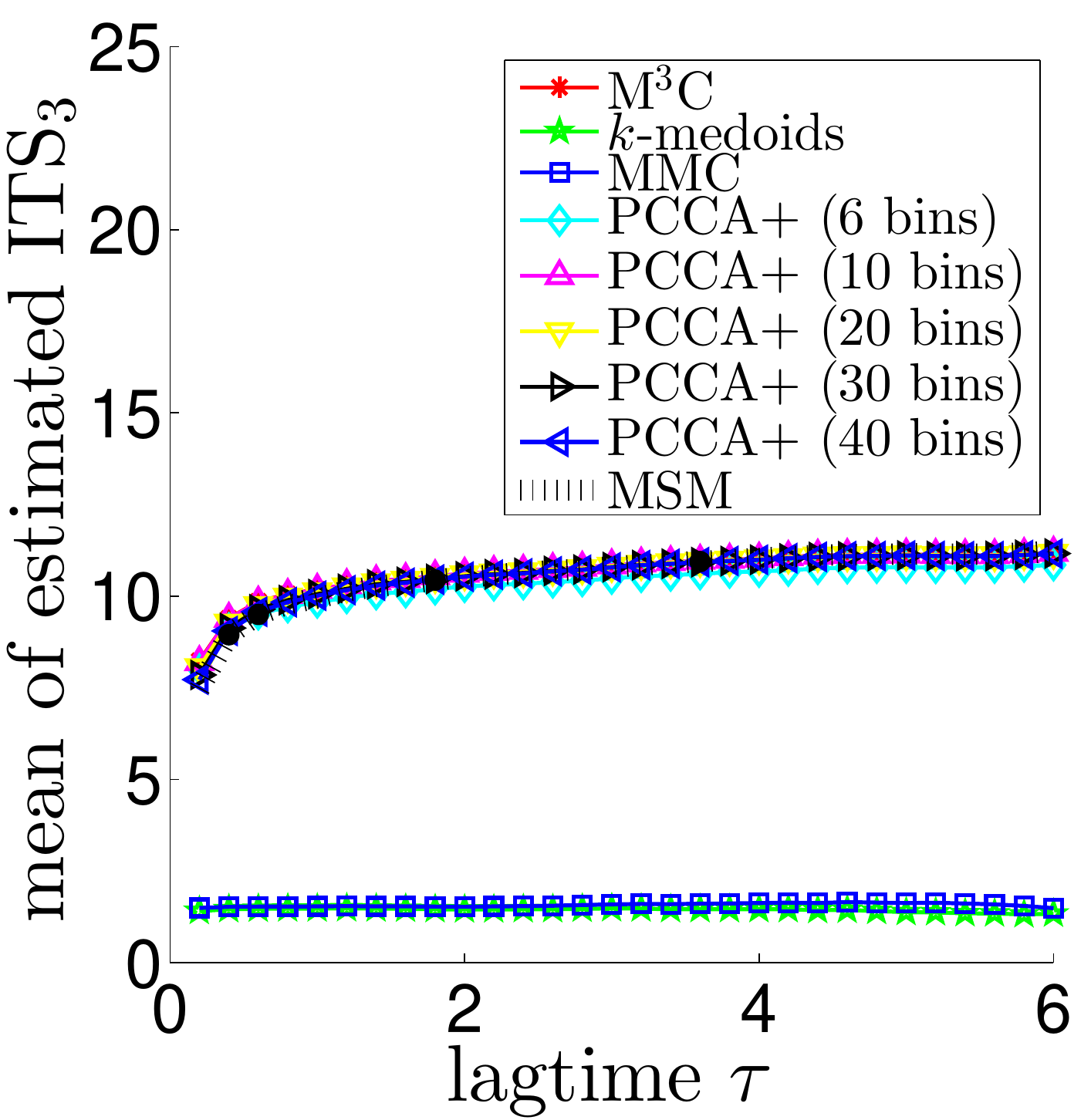}
\par\end{centering}

}\hfill{}\subfloat[]{\begin{centering}
\includegraphics[width=0.45\textwidth]{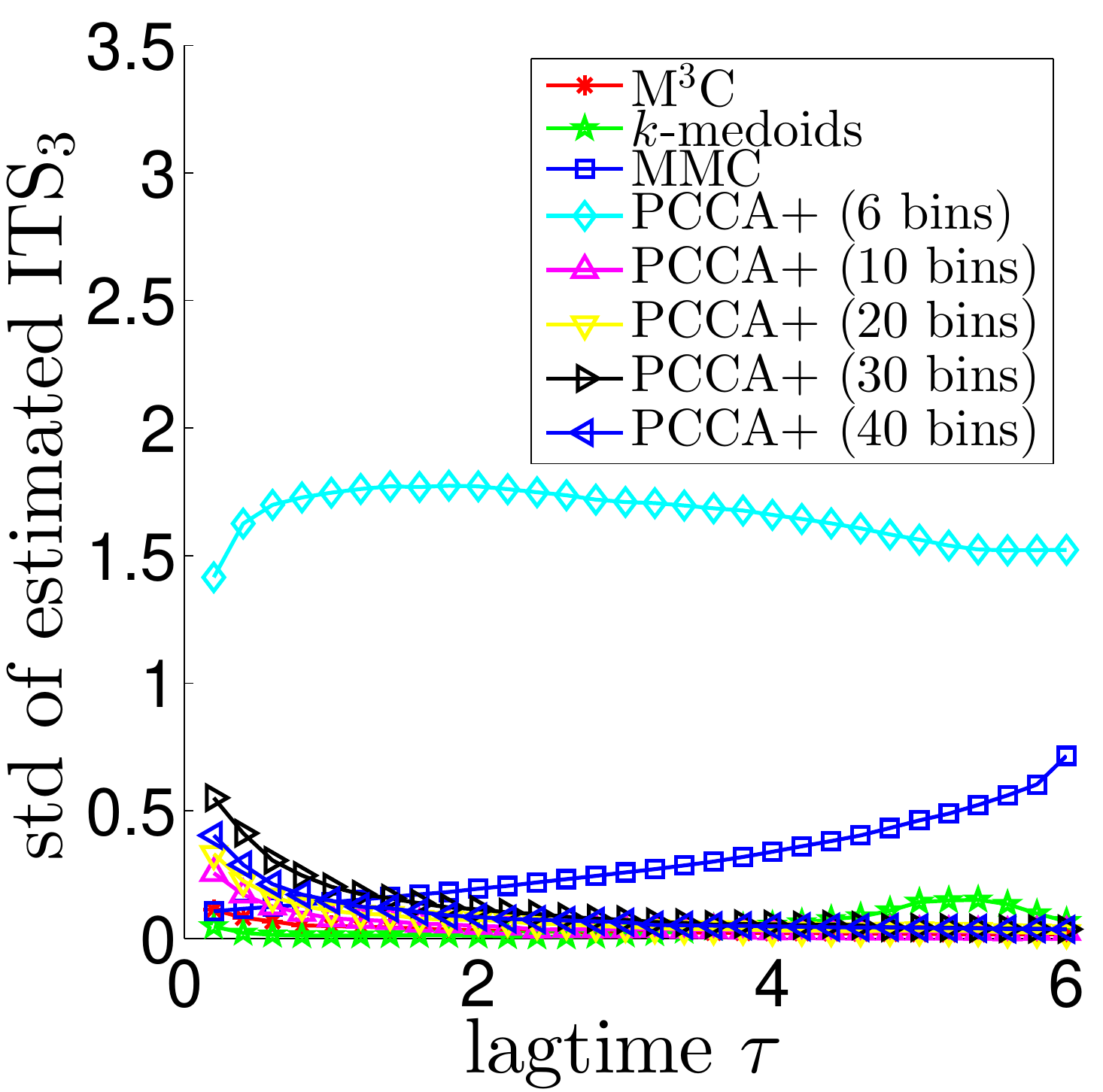}
\par\end{centering}

}

\protect\caption{Means and standard deviations of estimated implied timescales of Model
I obtained by different decomposition methods, where dotted lines
indicate estimates of the second and third relaxation timescales of
Model I computed by $50$-state Markov state models.\label{fig:its-Model-I}}
\end{figure}

\begin{figure}
\subfloat[]{\begin{centering}
\includegraphics[width=0.45\textwidth]{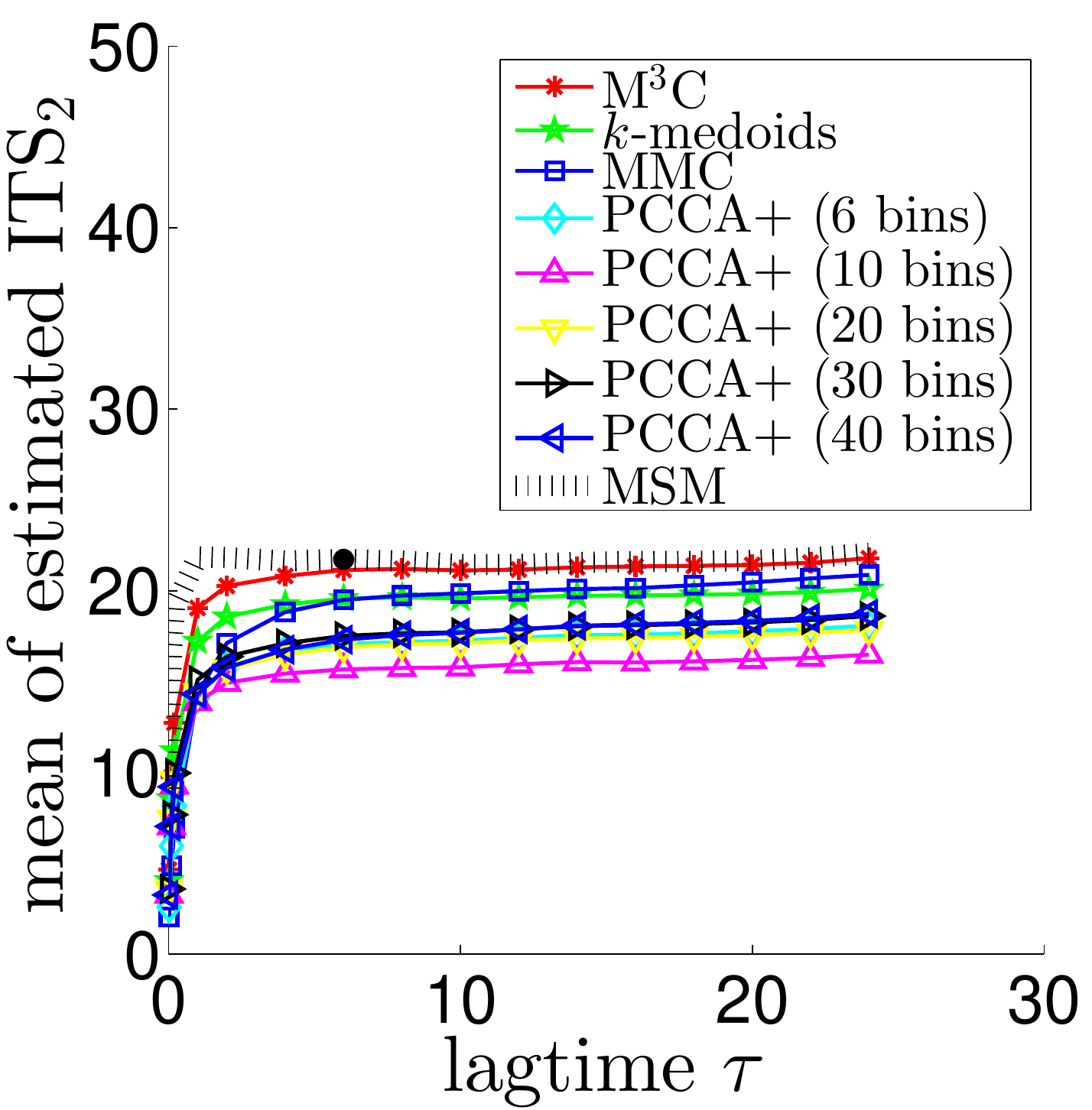}
\par\end{centering}

}\hfill{}\subfloat[]{\begin{centering}
\includegraphics[width=0.45\textwidth]{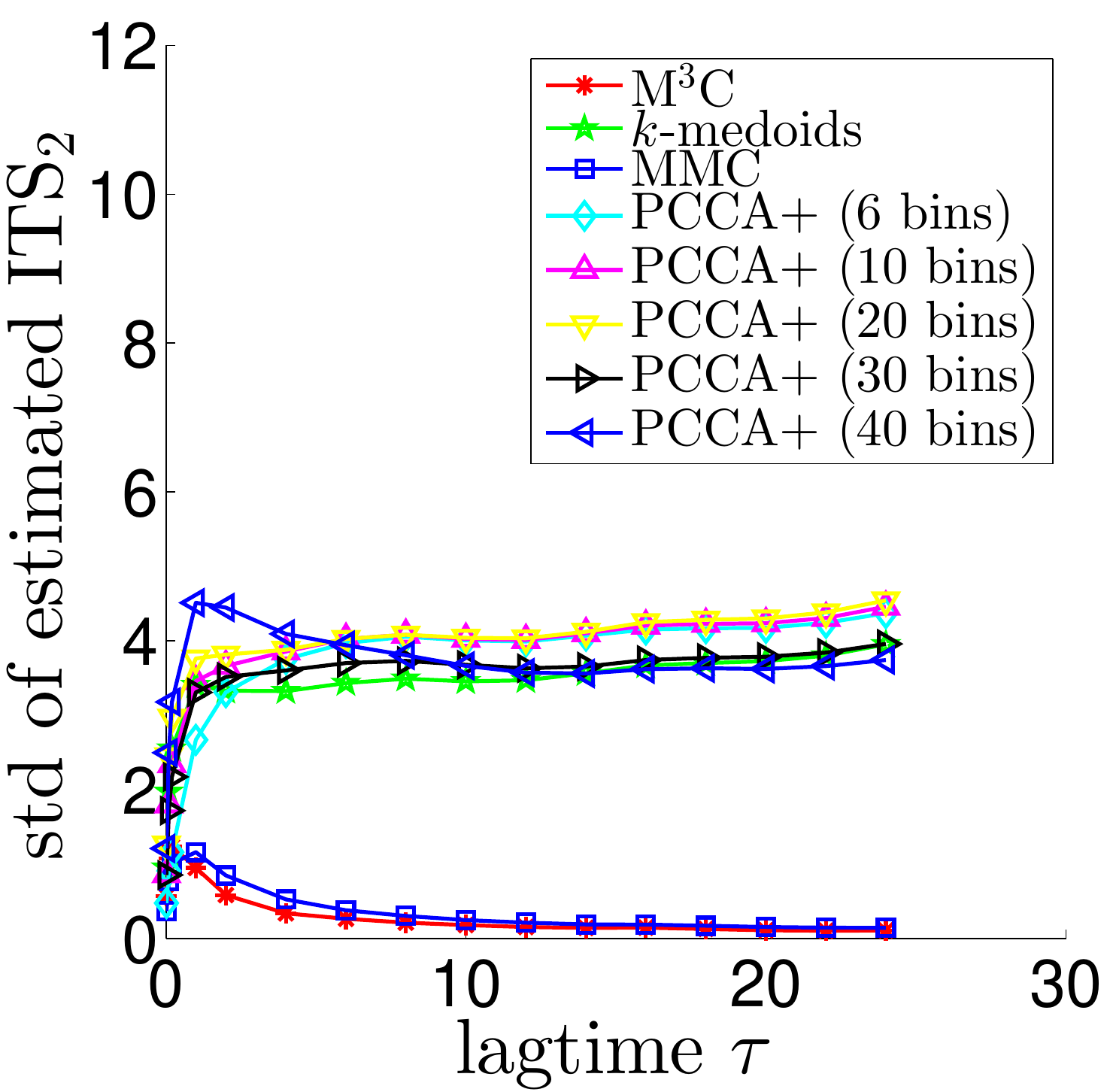}
\par\end{centering}

}

\subfloat[]{\begin{centering}
\includegraphics[width=0.45\textwidth]{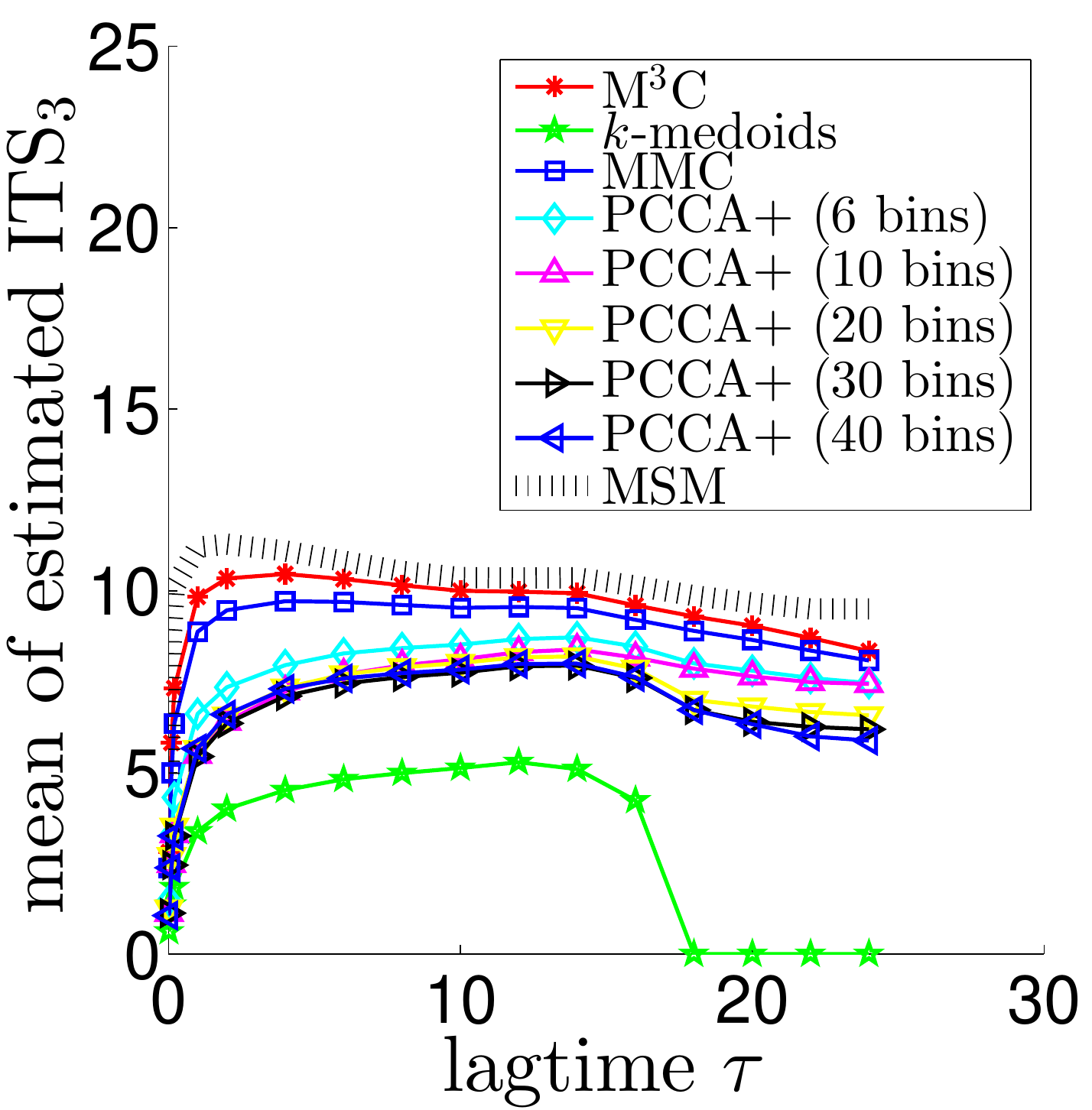}
\par\end{centering}

}\hfill{}\subfloat[]{\begin{centering}
\includegraphics[width=0.45\textwidth]{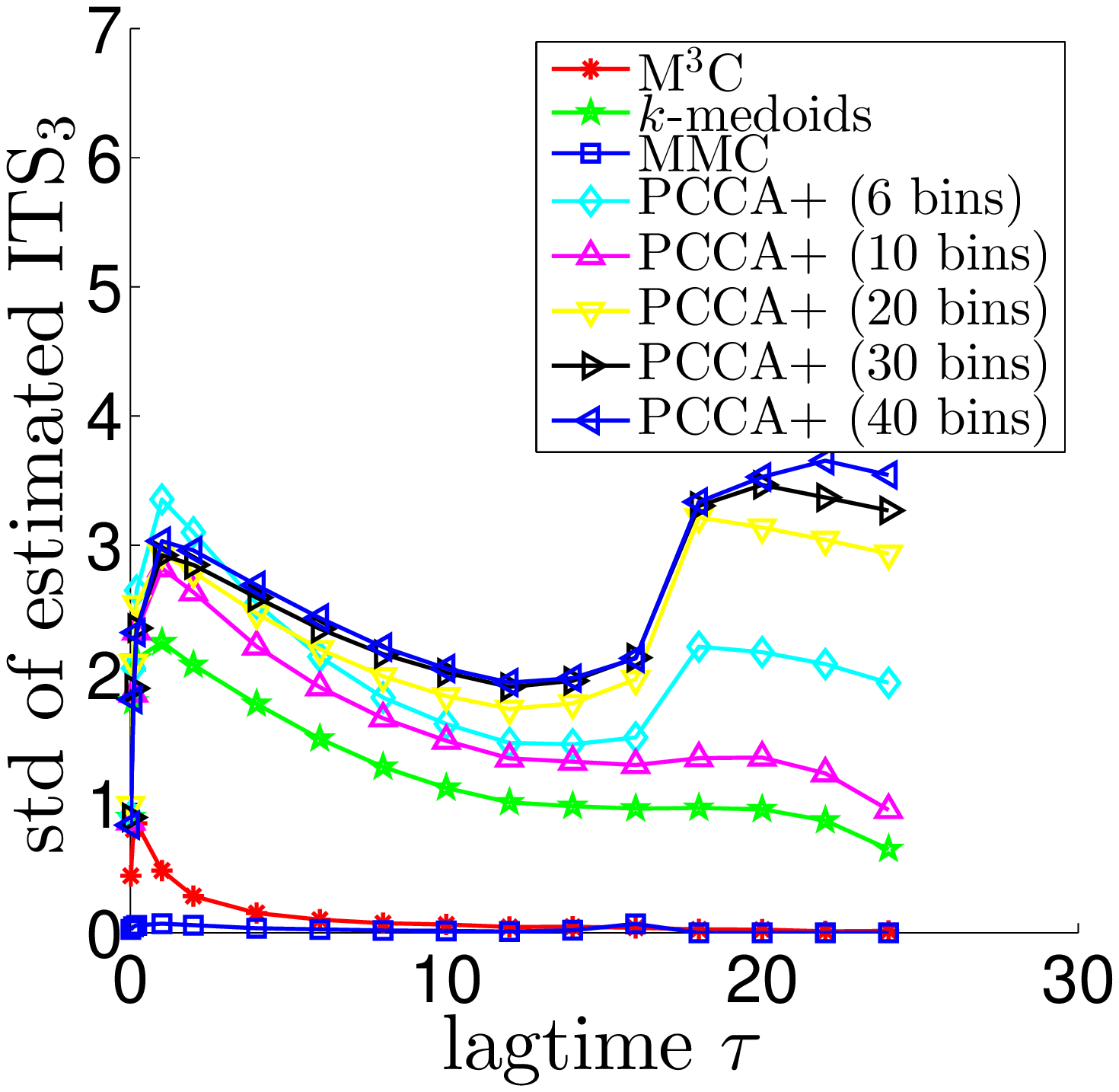}
\par\end{centering}

}

\protect\caption{Means and standard deviations of estimated implied timescales of Model
II obtained by different decomposition methods, where dotted lines
indicate estimates of the second and third relaxation timescales of
Model II computed by $50$-state Markov state models. (The implied
timescale value is set to be zero if the corresponding eigenvalue
of the transition probability matrix is zero or negative.)\label{fig:its-Model-II}}
\end{figure}

\subsection{Molecular dynamics simulations}

We consider in this section the metastable state decomposition problem
of molecular dynamics simulation models of alanine dipeptide and deca-alanine.
Alanine dipeptide (sequence acetyl-alanine-methylamide) is a small
molecule which consists of two alanine amino acid units. The structural
and dynamical properties of this molecule have been thoroughly studied,
and its conformation space (phase space) can be conveniently described
by two backbone dihedral angles $\varphi$ and $\psi$ (see Fig.~\ref{fig:alanine-dipeptide}).
Deca-alanine is a small peptide composed of $10$ alanine residues,
and its configuration can be described by $18$ backbone dihedral
angles. We perform twenty simulations of $200\mathrm{ns}$ of molecular
dynamics of alanine dipeptide with sample interval $20\mathrm{ps}$
and six $500\mathrm{ns}$ molecular dynamics simulations of deca-alanine
with sample interval $100\mathrm{ps}$ (The detailed simulation model
is given in \citep{Feliks2013variational}).

\begin{figure}
\begin{centering}
\includegraphics[width=0.45\textwidth]{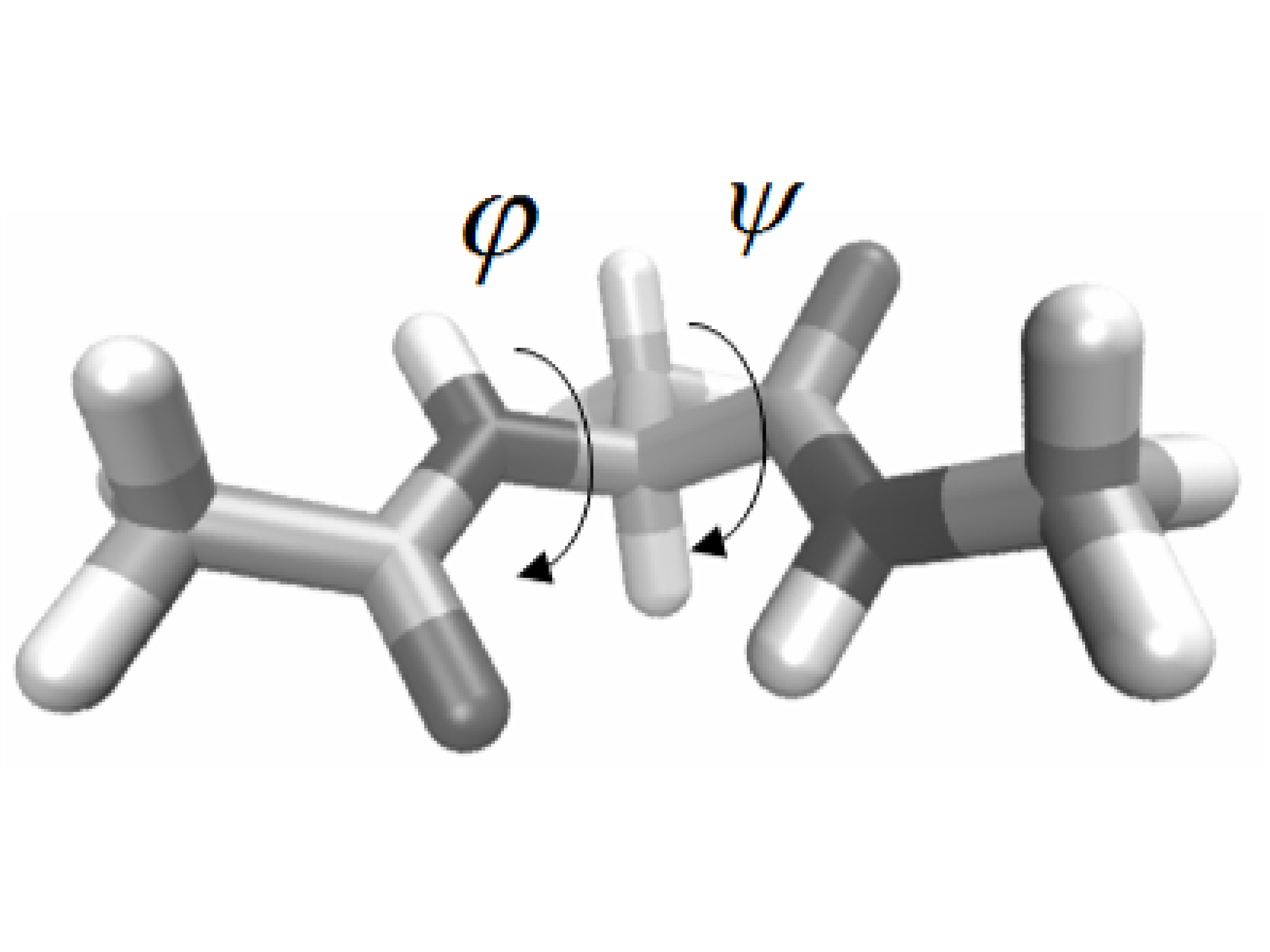}
\par\end{centering}

\protect\caption{Illustration of the structure of alanine dipeptide\label{fig:alanine-dipeptide}}
\end{figure}

The metastable state decomposition methods are applied to each simulation
trajectory ($\kappa=3$ for alanine dipeptide and $\kappa=2$ for
deca-alanine), and the corresponding $Q$ values and implied timescales
are calculated from all the other trajectories of the same molecule.
Moreover, considering the periodicity of angular data, we represent
the molecular state as a vector $\mathbf{x}$ consisting of sin/cos
of the dihedral angles in experiments.

Fig.~\ref{fig:Potential-function-ad} shows the potential function
of alanine dipeptide in the space of $(\varphi,\psi)$ and the three
metastable states which can be manually identified according to experience,
and data points sampled from one simulation trajectory are displayed
in Fig.~\ref{fig:simulation-ad}. The decomposition results of the
simulation trajectory are plotted in Fig.~\ref{fig:Decomposition-results-ad}.
As can be seen, $k$-medoids and MMC fail to indentify the metastable
structure of alanine dipeptide, and the decompositions obtained by
PCCA+ and M\textsuperscript{3}C are consistent with the manual decomposition.
Table \ref{tab:Q-md} and Fig.~\ref{fig:its-mdad} summarize $Q$
values and implied timescales given by decomposition results of simulation
trajectories of alanine dipeptide. It can be observed that PCCA+ and
M\textsuperscript{3}C performs significantly better than the geometric
clustering methods $k$-medoids and MMC. M\textsuperscript{3}C achieves
the similar average $Q$ values and implied timescales as PCCA+, but
the corresponding standard deviations of M\textsuperscript{3}C are
lower than that of PCCA+, which shows M\textsuperscript{3}C has more
stable performance for this molecular dynamics simulation model.

The decomposition results of deca-alanine are displayed in Table \ref{tab:Q-md}
and Fig.~\ref{fig:its-mddd}. The $Q$ values obtained by M\textsuperscript{3}C
are significantly smaller than that given by the other methods. In
contrast to alanine dipeptide, the kinetics of deca-alanine is much
more complicated and it is difficult to accurately estimate the relaxation
timescales. In \citep{Feliks2013variational}, a lower bound $6.5\mathrm{ns}$
for the second relaxation timescale is given. Moreover, according
to the variational principle \citep{noe2013variational}, the second
implied timescale obtained from a metastable state decomposition is
always smaller than the true one if there is no statistical noise.
So we can conclude from Fig.~\ref{fig:its-mddd} that M\textsuperscript{3}C
gives more accurate estimate of $\mathrm{ITS}_{2}$ than the other
methods for this molecular dynamics model.

\begin{figure}
\subfloat[Potential function estimated from simulation trajectories, where dashed
lines represent boundaries between $3$ metastable states which are
manually identified. \label{fig:Potential-function-ad}]{\begin{centering}
\includegraphics[width=0.45\textwidth]{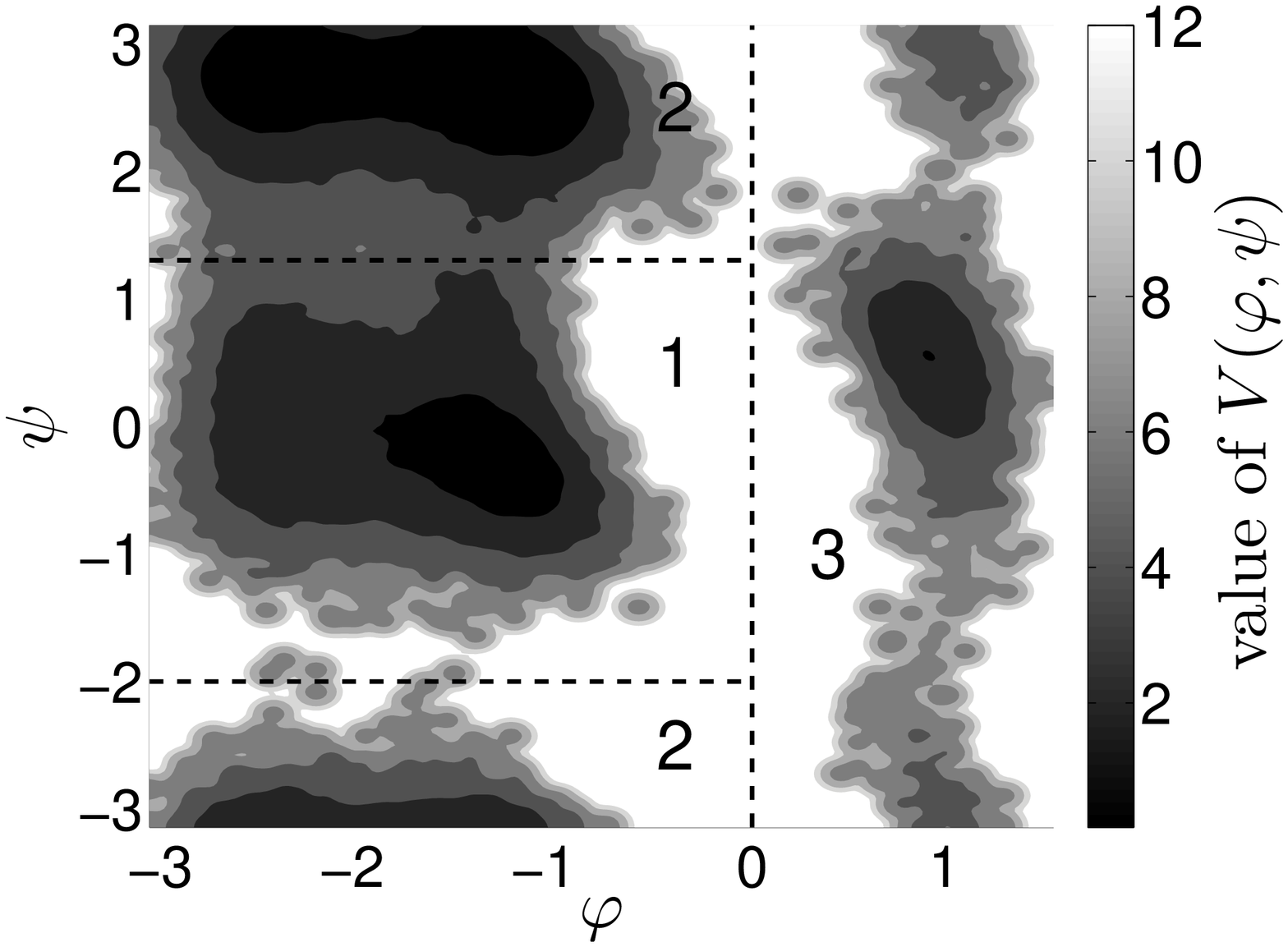}
\par\end{centering}

}\hfill{}\subfloat[Data points sampled from a simulation trajectory with sample interval
$\Delta t=20\mathrm{ps}$.\label{fig:simulation-ad}]{\begin{centering}
\includegraphics[width=0.45\textwidth]{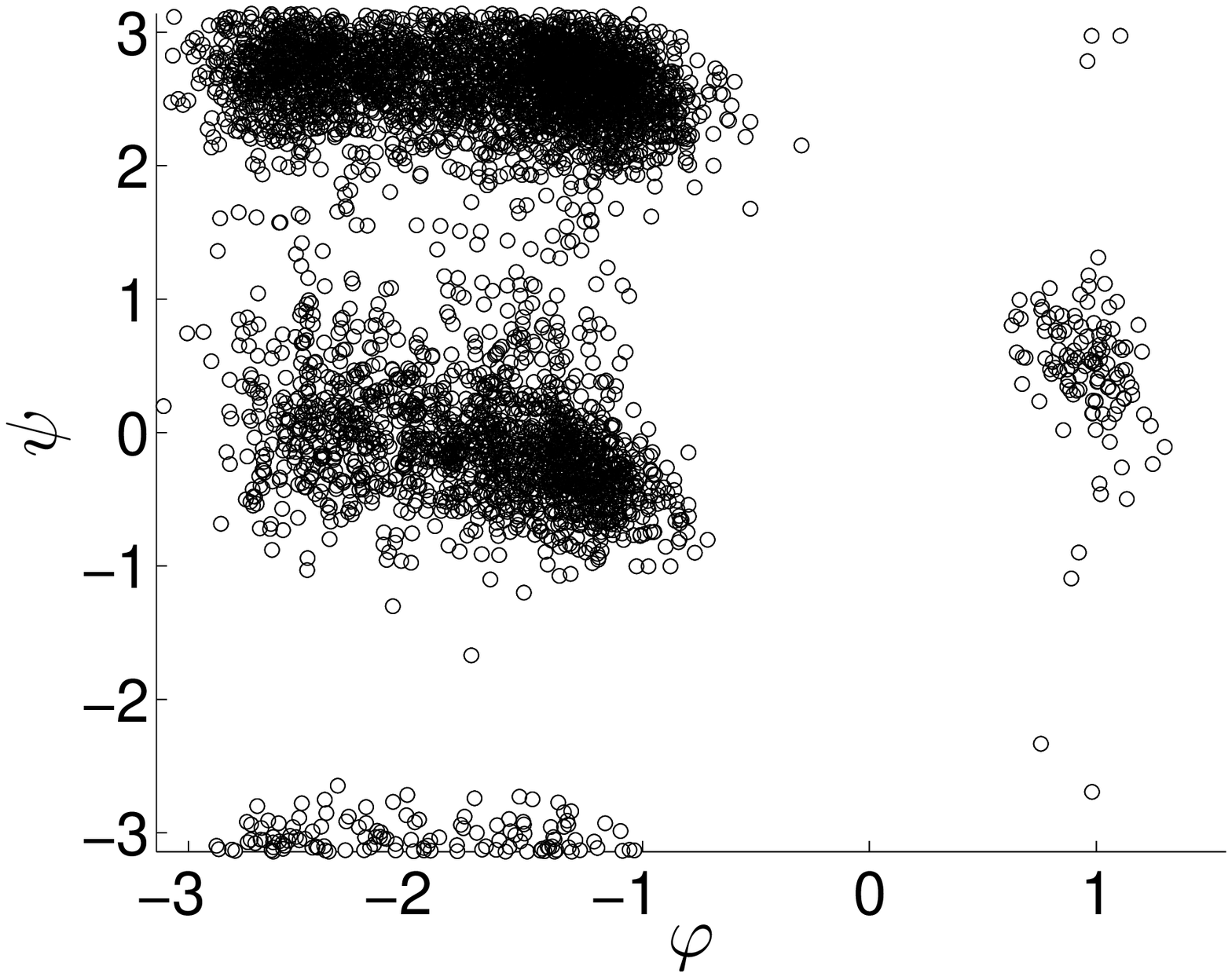}
\par\end{centering}

}

\protect\caption{Illustration of the molecular dynamics simulation of alanine dipetide\label{fig:Illustration-of-ad}}
\end{figure}

\begin{figure}
\subfloat[\MMMC]{\begin{centering}
\includegraphics[width=0.45\textwidth]{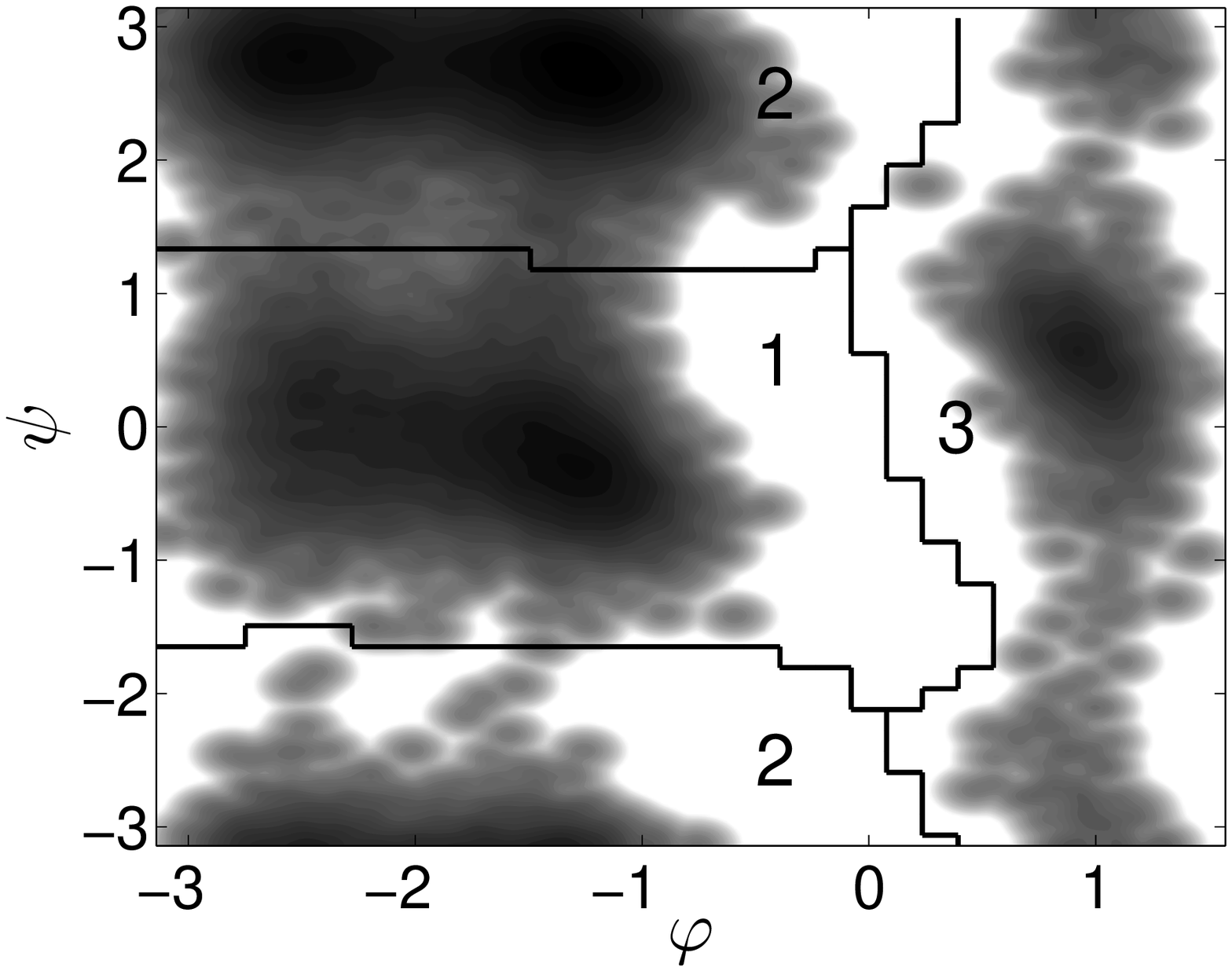}
\par\end{centering}

}\hfill{}\subfloat[$k$-medoids]{\begin{centering}
\includegraphics[width=0.45\textwidth]{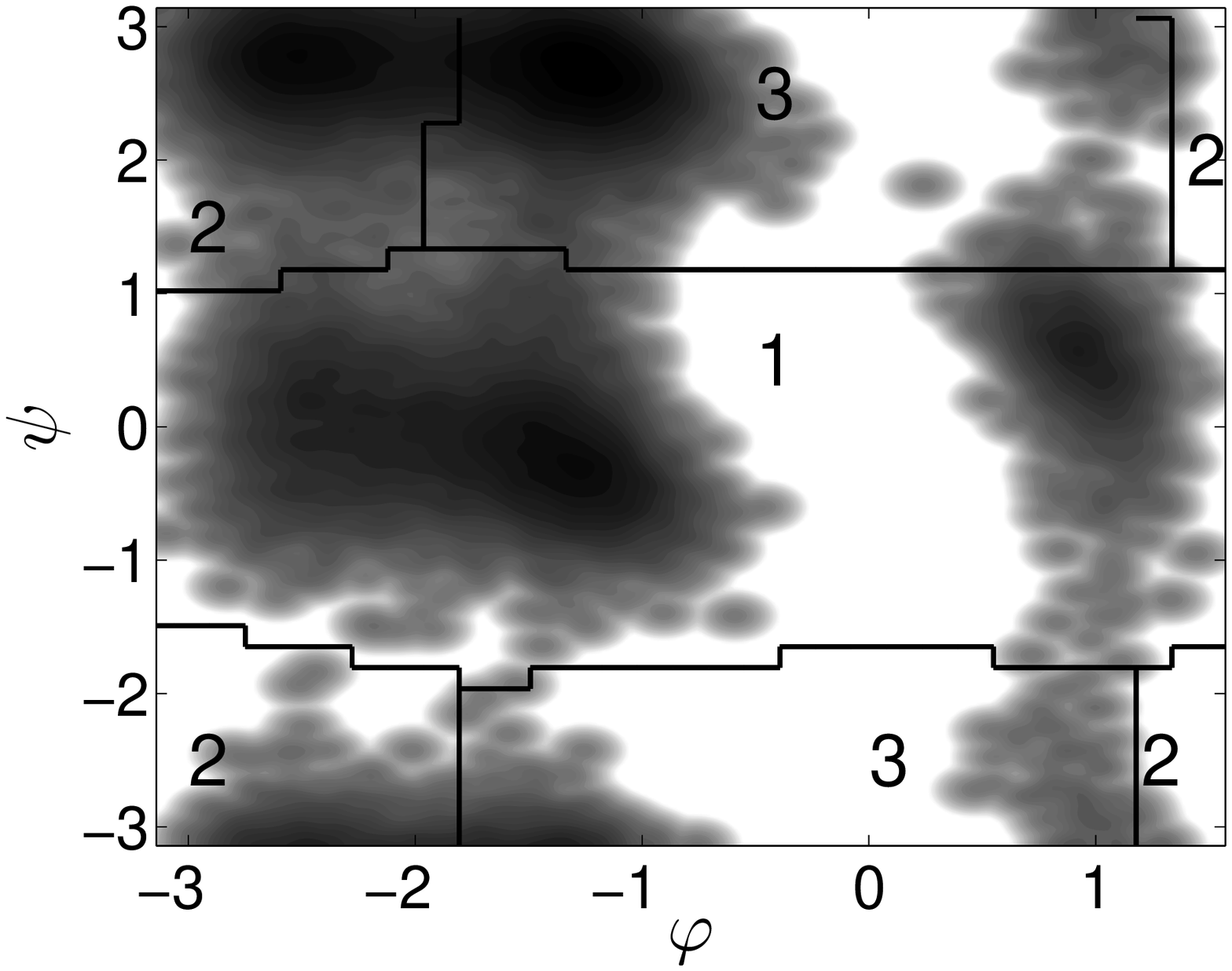}
\par\end{centering}

}

\subfloat[MMC]{\begin{centering}
\includegraphics[width=0.45\textwidth]{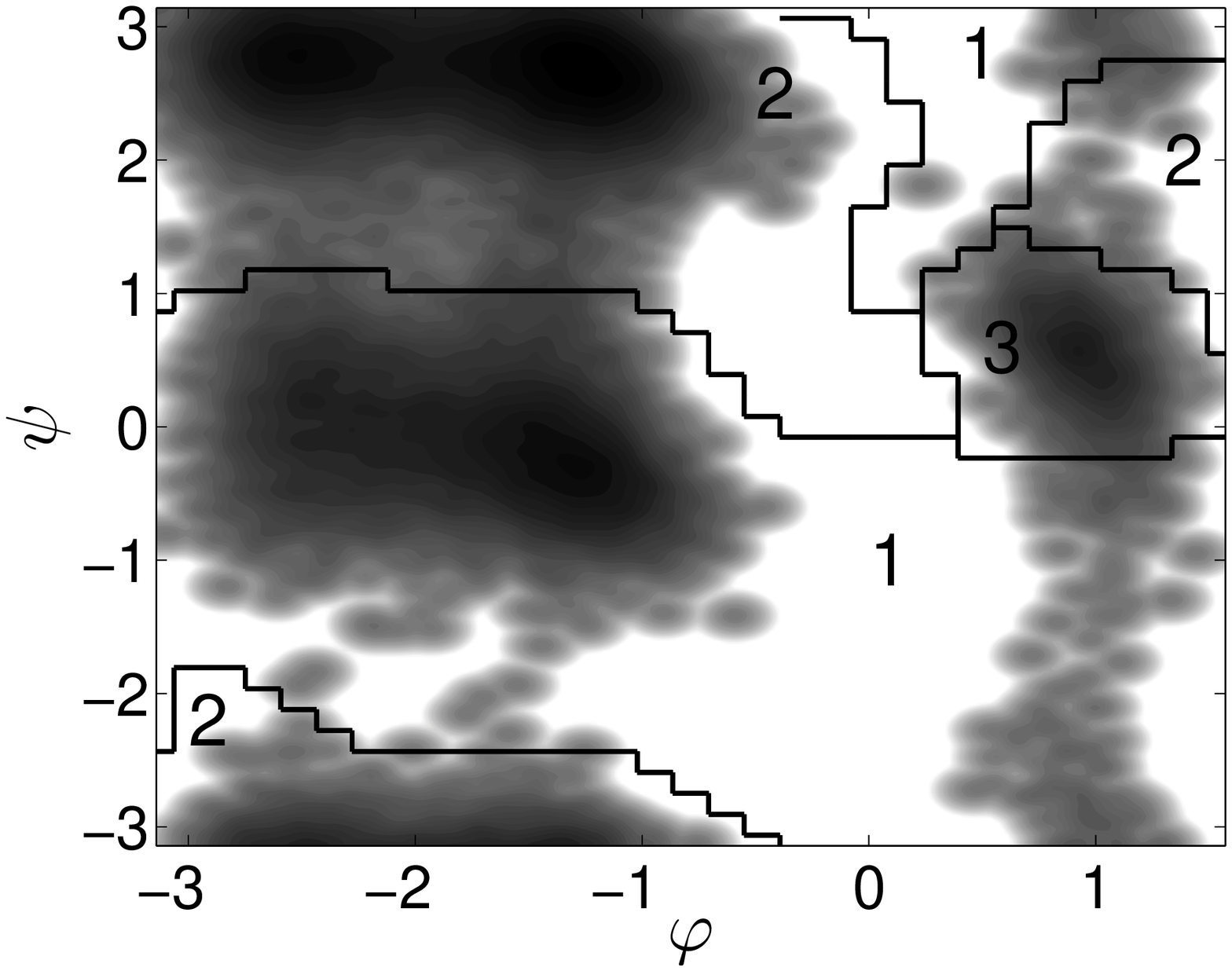}
\par\end{centering}

}\hfill{}\subfloat[PCCA+ (with $200$ bins)]{\begin{centering}
\includegraphics[width=0.45\textwidth]{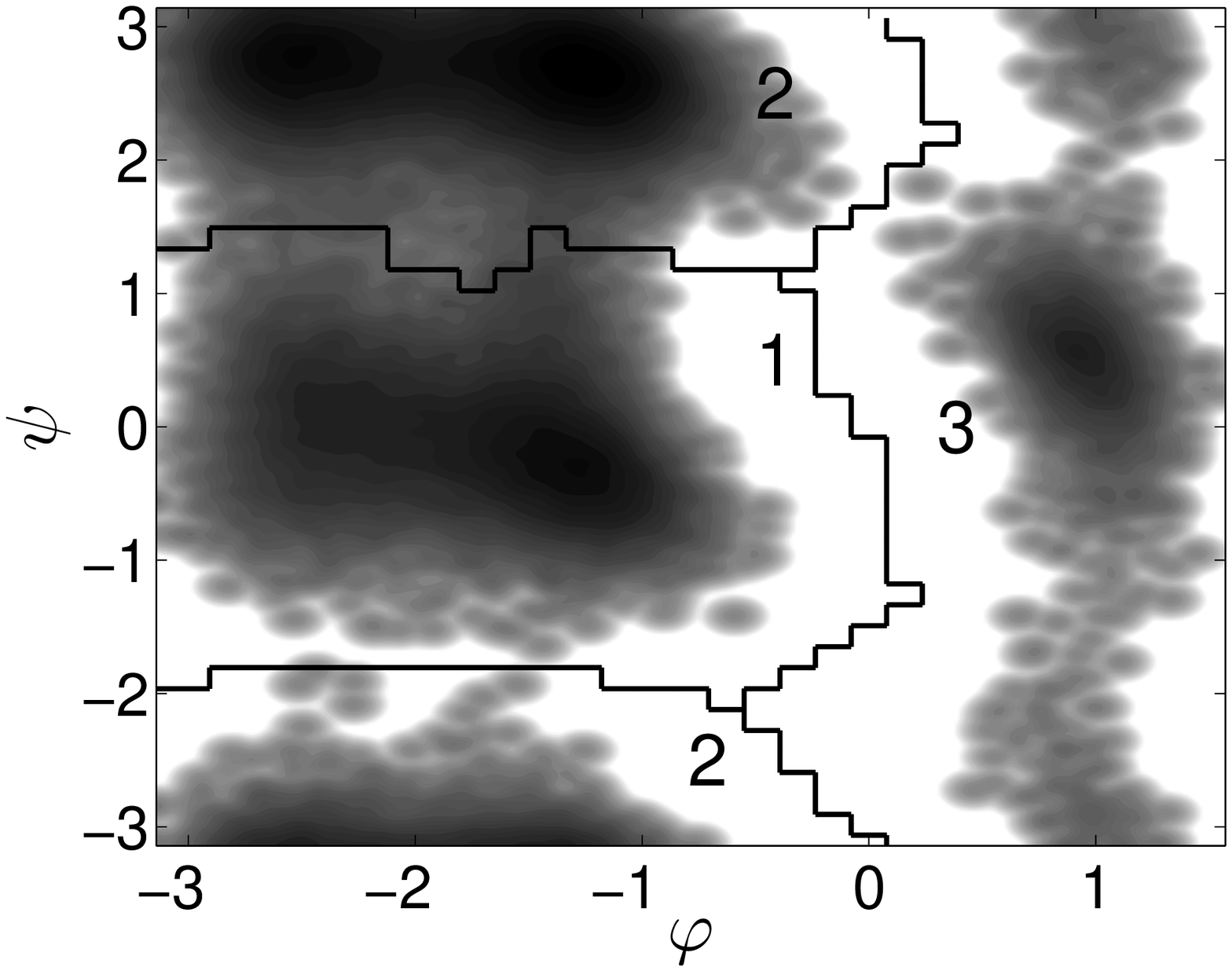}
\par\end{centering}

}

\protect\caption{Decomposition results of the simulation model of alanine dipetide,
where black lines represent boundaries between macrostates. The boundaries
are computed by the finite element method with mesh size $0.05\pi\times0.05\pi$.\label{fig:Decomposition-results-ad}}
\end{figure}

\begin{table}
\protect\caption{Means and standard deviations of $Q$ values calculated over $10$
independent experiments of alanine dipeptide and $6$ independent
experiments of deca-alanine\label{tab:Q-md}}

\resizebox{\textwidth}{!}{

\begin{tabular}{|c|c|c|c|c|c|c|c|c|}
\hline 
\multicolumn{1}{|c|}{} & $k$-medoids & MMC & $\begin{array}{c}
\text{PCCA+}\\
\text{(100 bins)}
\end{array}$ & $\begin{array}{c}
\text{PCCA+}\\
\text{(200 bins)}
\end{array}$ & $\begin{array}{c}
\text{PCCA+}\\
\text{(300 bins)}
\end{array}$ & $\begin{array}{c}
\text{PCCA+}\\
\text{(400 bins)}
\end{array}$ & $\begin{array}{c}
\text{PCCA+}\\
\text{(500 bins)}
\end{array}$ & \multicolumn{1}{c|}{\MMMC}\tabularnewline
\hline 
\hline 
alanine dipeptide & $1.7655\pm0.0027$ & $2.1521\pm0.5060$ & $2.7360\pm0.0077$ & $2.7381\pm0.0018$ & $2.7367\pm0.0059$ & $2.7340\pm0.0037$ & $2.7328\pm0.0052$ & $\mathbf{2.7397\pm0.0003}$\tabularnewline
\hline 
deca-alanine & $1.8113\pm0.0502$ & $1.8819\pm0.0454$ & $1.9033\pm0.0292$ & $1.8904\pm0.0402$ & $1.8894\pm0.0815$ & $1.9227\pm0.0238$ & $1.8896\pm0.0764$ & $\mathbf{1.9592\pm0.0038}$\tabularnewline
\hline 
\end{tabular}

}
\end{table}

\begin{figure}
\subfloat[]{\begin{centering}
\includegraphics[width=0.45\textwidth]{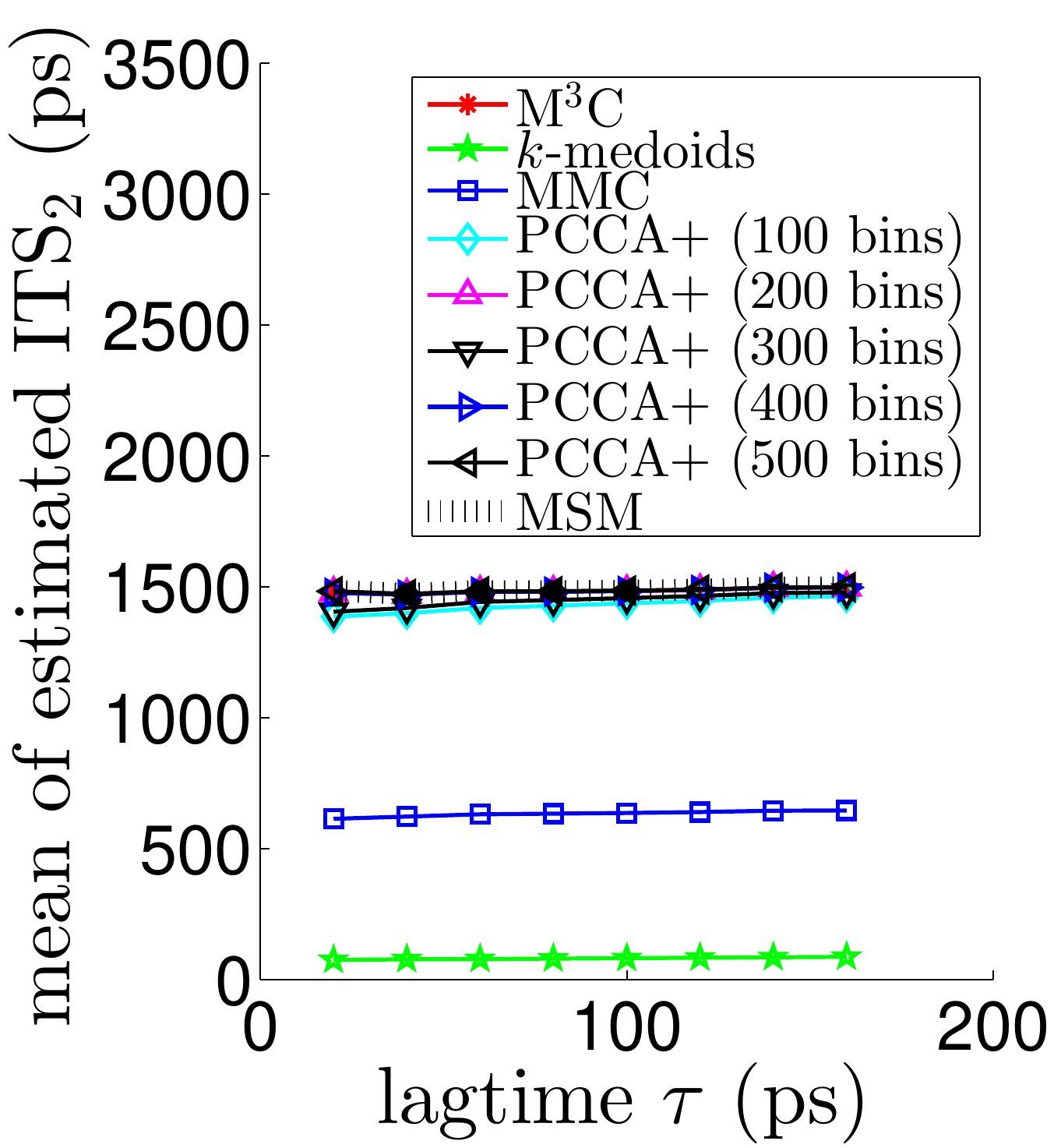}
\par\end{centering}

}\hfill{}\subfloat[]{\begin{centering}
\includegraphics[width=0.45\textwidth]{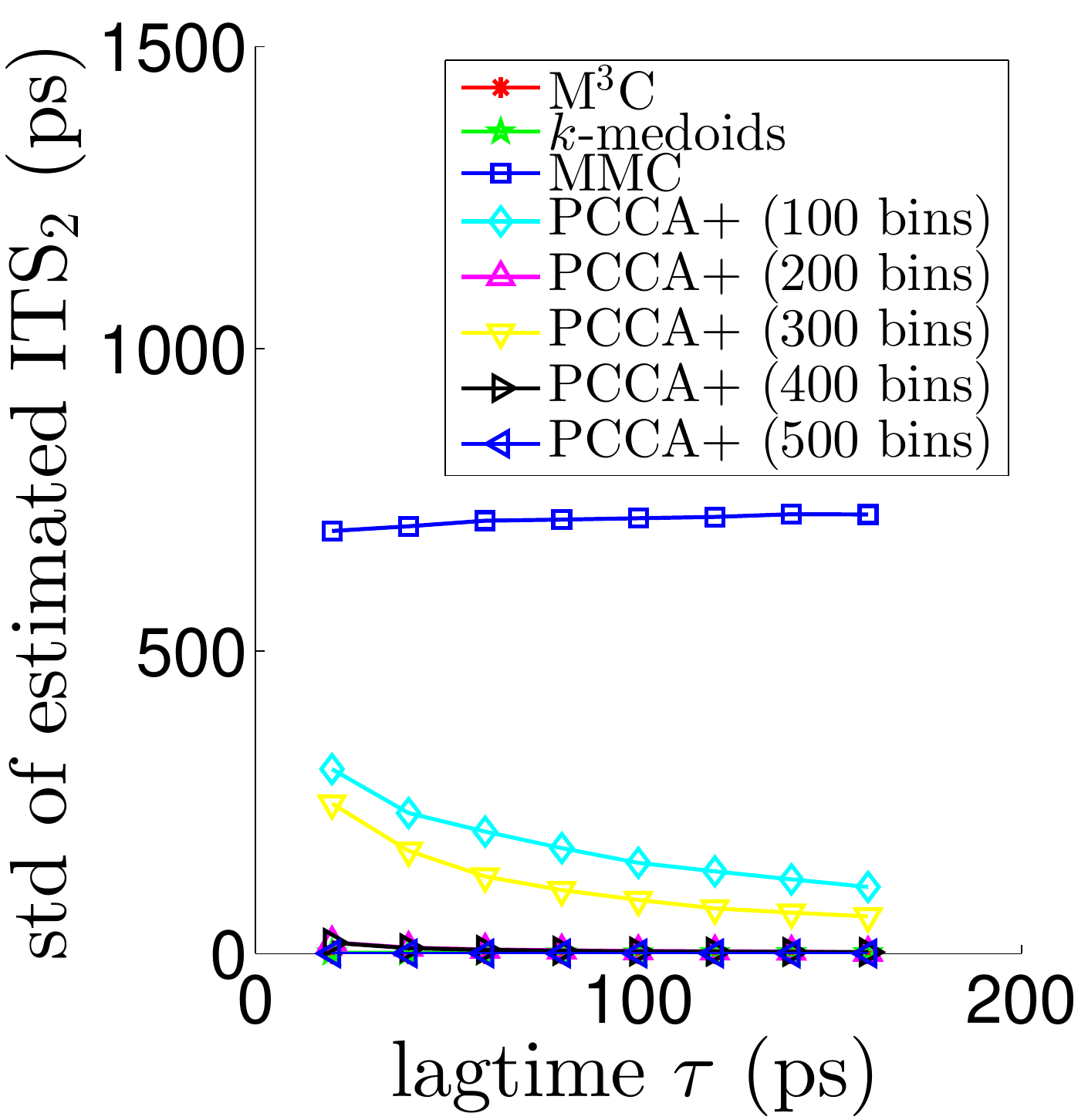}
\par\end{centering}

}

\subfloat[]{\begin{centering}
\includegraphics[width=0.45\textwidth]{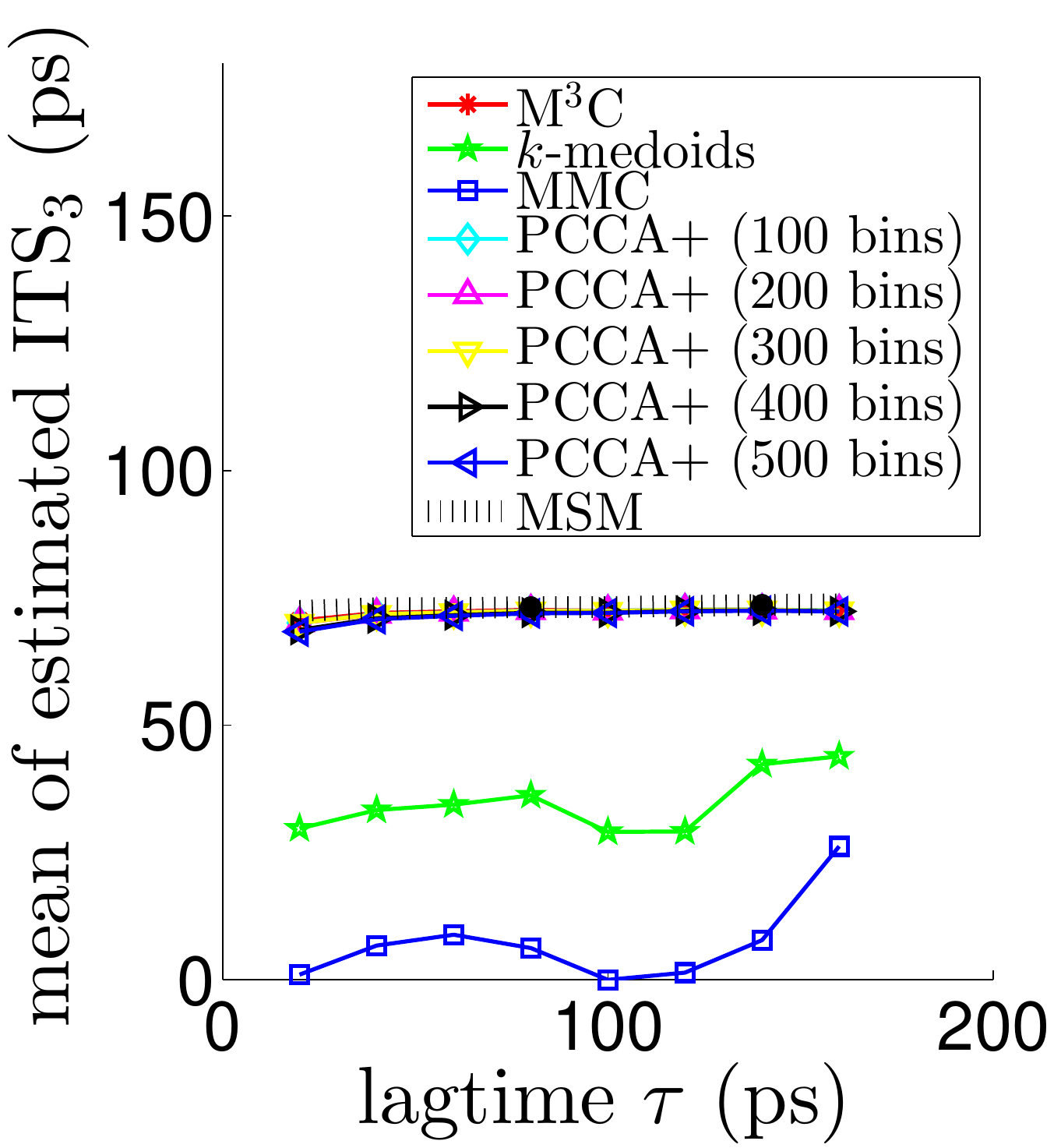}
\par\end{centering}

}\hfill{}\subfloat[]{\begin{centering}
\includegraphics[width=0.45\textwidth]{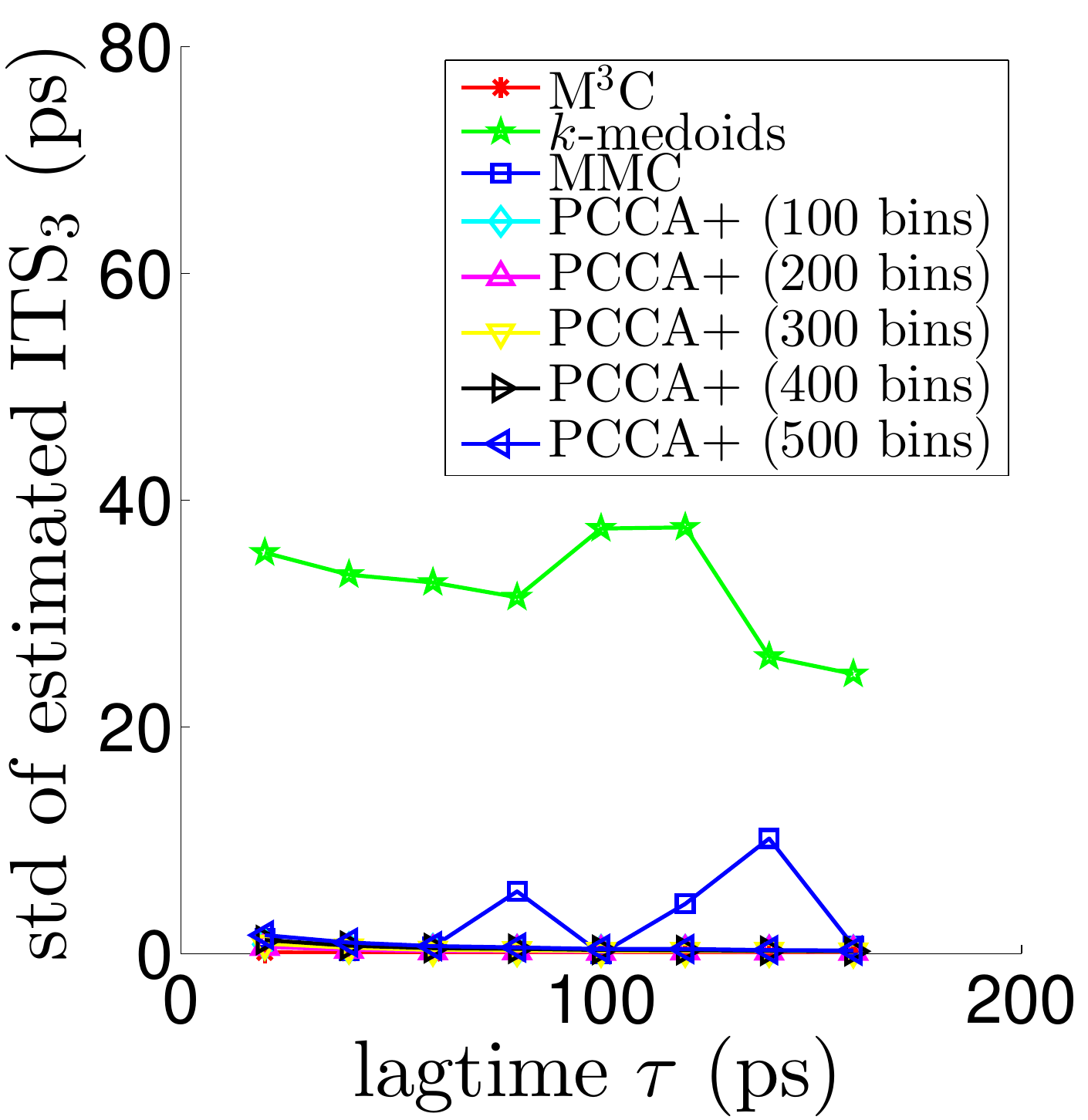}
\par\end{centering}

}

\protect\caption{Means and standard deviations of estimated implied timescales of alanine
dipeptide obtained by different decomposition methods, where dotted
lines indicate estimates of the second and third relaxation timescales
of alanine dipeptide computed by $500$-state Markov state models.\label{fig:its-mdad}}
\end{figure}

\begin{figure}
\subfloat[Model I]{\begin{centering}
\includegraphics[width=0.45\textwidth]{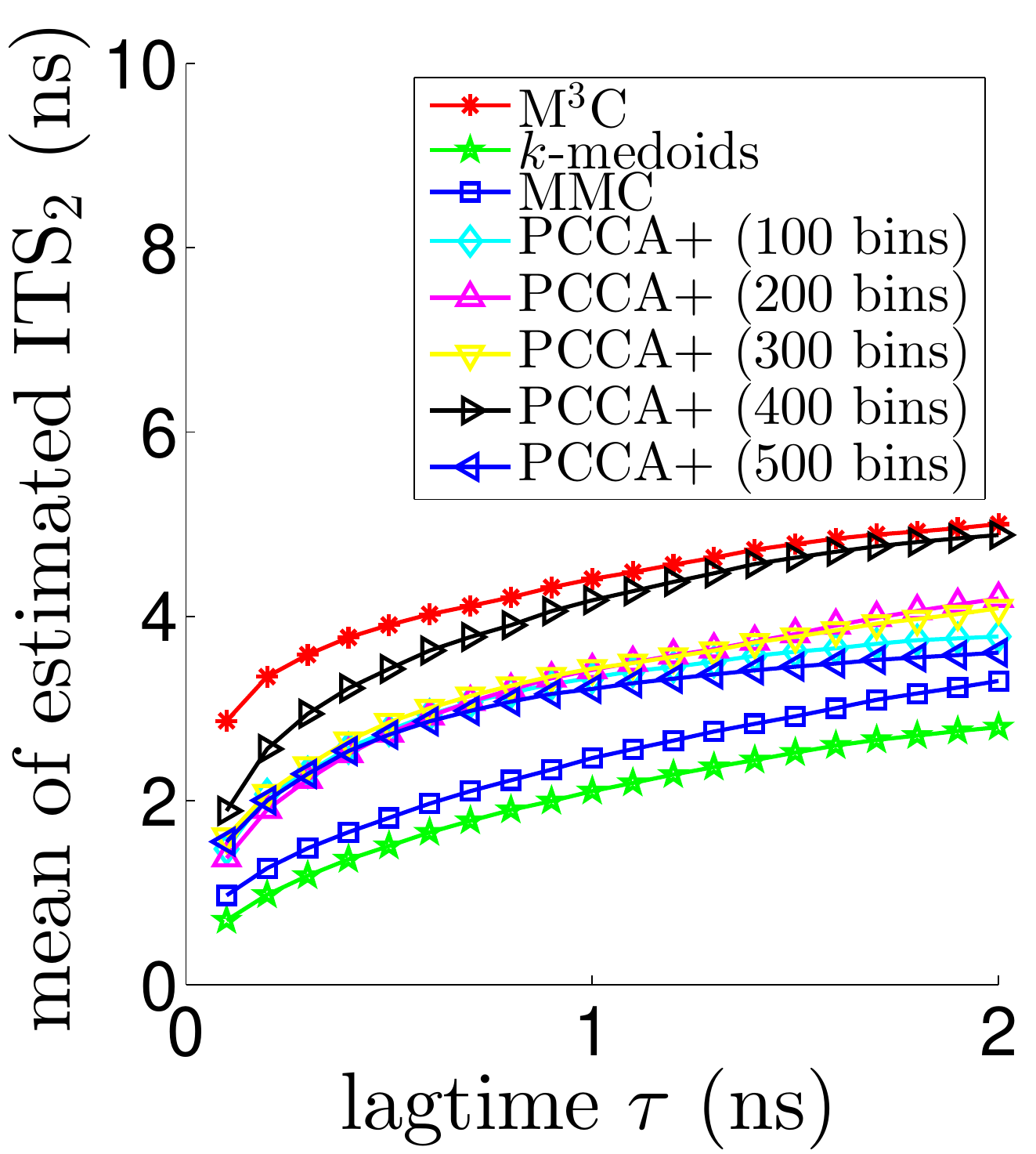}
\par\end{centering}

}\hfill{}\subfloat[Model II]{\begin{centering}
\includegraphics[width=0.45\textwidth]{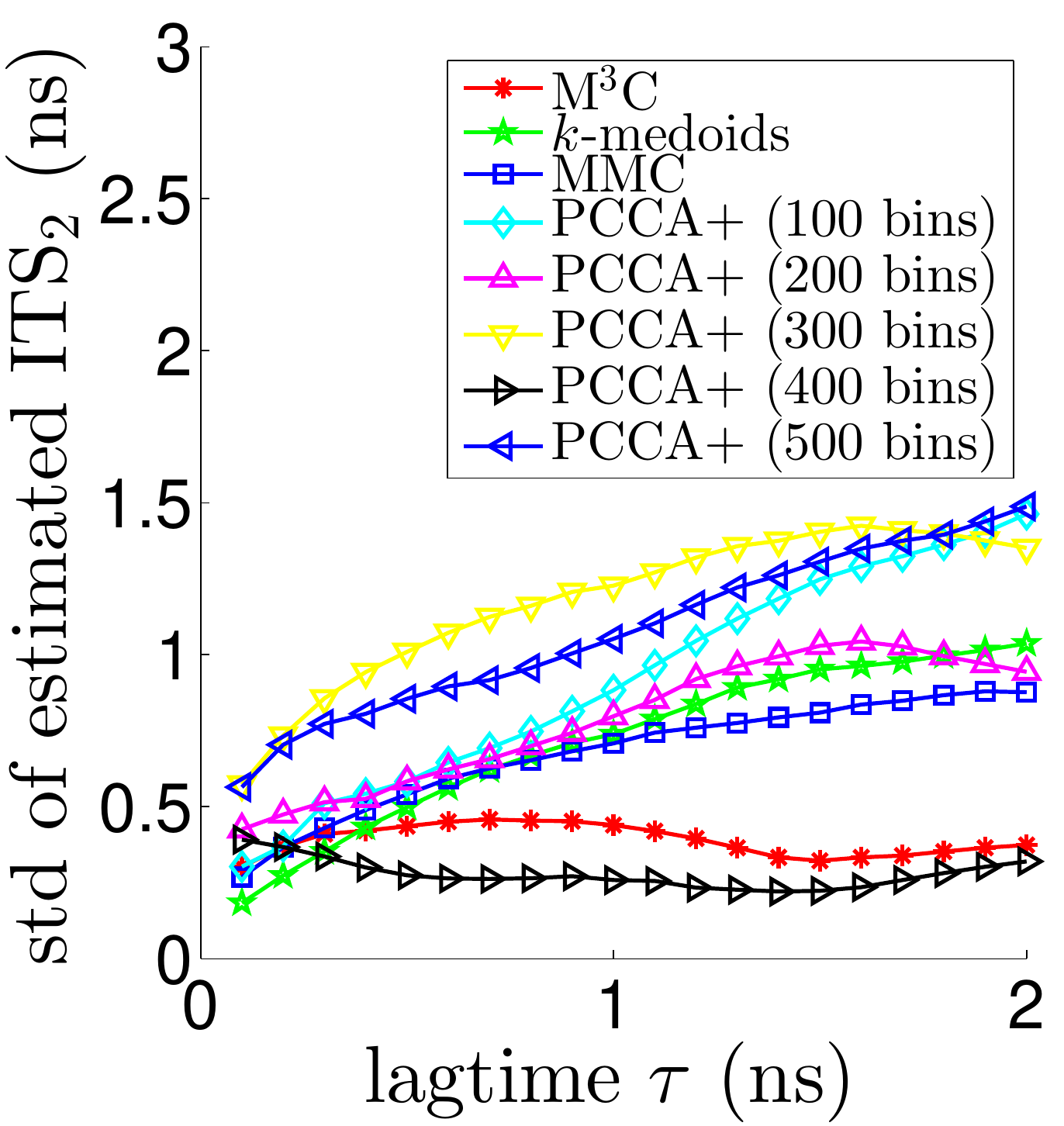}
\par\end{centering}

}

\protect\caption{Means and standard deviations of estimated implied timescales of deca-alanine
obtained by different decomposition methods.\label{fig:its-mddd}}
\end{figure}

\section{Conclusion}

Large margin methods have turned out to be an effective and robust
approach for supervised and unsupervised learning problems. In this
paper, we apply the large margin principle to the metastable state
decomposition problem, and propose a \emph{maximum margin metastable
clustering (M\textsuperscript{3}C) }method to identify metastable
states of complex stochastic systems. The key step is to design a
large margin metastable constraint \eqref{eq:metastable-largin-margin-constraint}
by combining the metastability criterion and large margin criterion,
where we assign a class label to each transition pair in trajectories
instead of a single data point. Then the error of metastable state
decomposition can be expressed by the misclassification loss function
in large margin learning, and a lot of well developed computational
techniques for large margin learning, such as kernel based feature
mapping and convex relaxation, can be utilized. Moreover, we present
a hybrid optimization algorithm which mixes global search and local
search strategies to solve M\textsuperscript{3}C problems with large-scale
data sets. In contrast to previous metastable state decomposition
methods including geometric clustering methods and kinetic clustering
methods, the M\textsuperscript{3}C method can effectively utilize
both the geometric information and the dynamical information provided
by trajectories without pre-discretization in space, and our experimental
analysis reveal that the M\textsuperscript{3}C method yields more
accurate and robust decomposition results than traditional geometric
and kinetics clustering methods in most cases.

The major drawback of M\textsuperscript{3}C is that the computing
burden will be heavy for very large data set, because it need to iteratively
solve an SVM-like problem. In the future, we will use some modern
SVM techniques such as Pegasos\citep{shalev2011pegasos} and core
vector machine\citep{tsang2005core} to improve the efficiency of
M\textsuperscript{3}C. Moreover, we will investigate how to extend
our method to the problem of slow process decomposition of metastable
systems \citep{noe2013variational,Guillermo2013identification} by
incorporating the distance metric learning technique \citep{DBLP:journals/corr/BelletHS13}
so that it can not only detect metastable states but also extract
dominant dynamical features from simulation and experimental data.

\section*{Acknowledgments}

This work is supported by the Deutsche Forschungsgemeinschaft (DFG)
under grant Number WU 744/1-1, and the author would like to thank
Feliks N\"uske (FU Berlin) for providing the molecular dynamics simulation
data.

\appendix

\section{Proof of equivalence between \eqref{eq:metastable-largin-margin-constraint}
and \eqref{eq:large-margin-constraint-pair}\label{sec:Proof-of-equivalence-margin-constraint}}

Suppose first that \eqref{eq:large-margin-constraint-pair} holds.
Substituting $(\bar{k},\ubar k)=(k,y_{n})$ and $(\bar{k},\ubar k)=(y_{n},k)$
into \eqref{eq:large-margin-constraint-pair}, we can get
\begin{equation}
\begin{array}{l}
\left(\mathbf{w}_{y_{n}y_{n}}^{\intercal}-\mathbf{w}_{\bar{k}\ubar k}^{\intercal}\right)\bm{\phi}\left(\bar{\mathbf{x}}_{n},\ubar{\mathbf{x}}_{n}\right)+\left(b_{y_{n}y_{n}}-b_{\bar{k}\ubar k}\right)+1_{y_{n}=\bar{k}=\ubar k}\\
\begin{array}{rl}
= & \left(\mathbf{w}_{y_{n}}^{\intercal}-\mathbf{w}_{k}^{\intercal}\right)\bm{\phi}\left(\bar{\mathbf{x}}_{n}\right)+\left(b_{y_{n}}-b_{k}\right)+1_{y_{n}=k}\\
\ge & 1
\end{array}
\end{array}
\end{equation}
and
\begin{equation}
\begin{array}{l}
\left(\mathbf{w}_{y_{n}y_{n}}^{\intercal}-\mathbf{w}_{\bar{k}\ubar k}^{\intercal}\right)\bm{\phi}\left(\bar{\mathbf{x}}_{n},\ubar{\mathbf{x}}_{n}\right)+\left(b_{y_{n}y_{n}}-b_{\bar{k}\ubar k}\right)+1_{y_{n}=\bar{k}=\ubar k}\\
\begin{array}{rl}
= & \left(\mathbf{w}_{y_{n}}^{\intercal}-\mathbf{w}_{k}^{\intercal}\right)\bm{\phi}\left(\ubar{\mathbf{x}}_{n}\right)+\left(b_{y_{n}}-b_{k}\right)+1_{y_{n}=k}\\
\ge & 1
\end{array}
\end{array}
\end{equation}
So \eqref{eq:large-margin-constraint-pair} is a sufficient condition
for \eqref{eq:metastable-largin-margin-constraint}.

We now show the necessity of \eqref{eq:large-margin-constraint-pair}
for \eqref{eq:metastable-largin-margin-constraint}. Substituting
$k=\bar{k}$ and $k=\ubar k$ into the two inequalities in \eqref{eq:metastable-largin-margin-constraint},
respectively, yields
\begin{equation}
\left(\mathbf{w}_{y_{n}y_{n}}^{\intercal}-\mathbf{w}_{\bar{k}\ubar k}^{\intercal}\right)\bm{\phi}\left(\bar{\mathbf{x}}_{n},\ubar{\mathbf{x}}_{n}\right)+\left(b_{y_{n}y_{n}}-b_{\bar{k}\ubar k}\right)+\left(1_{y_{n}=\bar{k}}+1_{y_{n}=\ubar k}-1\right)\ge1
\end{equation}
Note that
\begin{equation}
1_{y_{n}=\bar{k}}+1_{y_{n}=\ubar k}-1=\left\{ \begin{array}{ll}
1, & y_{n}=\bar{k}=\ubar k\\
0, & y_{n}\in\{\bar{k},\ubar k\},\bar{k}\neq\ubar k\\
-1, & \mathrm{otherwise}
\end{array}\right.
\end{equation}
and
\begin{equation}
1_{y_{n}=\bar{k}=\ubar k}=\left\{ \begin{array}{ll}
1, & y_{n}=\bar{k}=\ubar k\\
0, & \mathrm{otherwise}
\end{array}\right.
\end{equation}
Therefore,
\begin{equation}
\begin{array}{l}
\left(\mathbf{w}_{y_{n}y_{n}}^{\intercal}-\mathbf{w}_{\bar{k}\ubar k}^{\intercal}\right)\bm{\phi}\left(\bar{\mathbf{x}}_{n},\ubar{\mathbf{x}}_{n}\right)+\left(b_{y_{n}y_{n}}-b_{\bar{k}\ubar k}\right)+1_{y_{n}=\bar{k}=\ubar k}\\
\begin{array}{rl}
\ge & \left(\mathbf{w}_{y_{n}y_{n}}^{\intercal}-\mathbf{w}_{\bar{k}\ubar k}^{\intercal}\right)\bm{\phi}\left(\bar{\mathbf{x}}_{n},\ubar{\mathbf{x}}_{n}\right)+\left(b_{y_{n}y_{n}}-b_{\bar{k}\ubar k}\right)+\left(1_{y_{n}=\bar{k}}+1_{y_{n}=\ubar k}-1\right)\\
\ge & 1
\end{array}
\end{array}
\end{equation}

From the above, we can conclude that \eqref{eq:large-margin-constraint-pair}
is equivalent to \eqref{eq:metastable-largin-margin-constraint}.

\section{Proof of Theorem \ref{thm:cg-m3c-equivalence}\label{sec:Proof-of-Theorem-cg-m3c-equivalence}}

Let us start with the case that all labels $\mathbf{y}^{c}$ in \eqref{eq:m3c-coarse-grained}
are given. In this case, the coarse-grained M\textsuperscript{3}C
problem \eqref{eq:m3c-coarse-grained} is reduced to a simple quadratic
programming problem: 
\begin{equation}
\begin{array}{cl}
\min\limits _{\mathbf{W},\bm{\xi}^{c}} & \frac{1}{2}\beta\left\Vert \mathbf{W}\right\Vert ^{2}+\mathbf{c}^{\intercal}\bm{\xi}^{c}\\
\mathrm{s.t.} & \forall n=1,\ldots,N^{c},\quad\forall\bar{k},\ubar k=1,\ldots,\kappa,\\
 & \left(\mathbf{w}_{y_{n}^{c}y_{n}^{c}}^{\intercal}-\mathbf{w}_{\bar{k}\ubar k}^{\intercal}\right)\bm{\phi}\left(\bar{\mathbf{x}}_{n}^{c},\ubar{\mathbf{x}}_{n}\right)+1_{y_{n}^{c}=\bar{k}=\ubar k}\ge1-\xi_{n}^{c},
\end{array}\label{eq:m3c-coarse-grained-supervised}
\end{equation}
For the sake of convenience, here we let $\mathbf{e}_{k}$ denote
a $\kappa$ dimensional vector with only the $k$-th element being
$1$ and others $0$, and define a matrix $\mathbf{B}\in\mathbb{R}^{2\kappa\times\kappa^{2}}$
as
\begin{equation}
\mathbf{B}=\left[\begin{array}{c}
\bar{\mathbf{B}}\\
\ubar{\mathbf{B}}
\end{array}\right]=\left[\begin{array}{cccccccc}
\mathbf{e}_{1} & \mathbf{e}_{2} & \cdots & \mathbf{e}_{\kappa} & \mathbf{e}_{1} & \mathbf{e}_{1} & \cdots & \mathbf{e}_{\kappa}\\
\mathbf{e}_{1} & \mathbf{e}_{2} & \cdots & \mathbf{e}_{\kappa} & \mathbf{e}_{2} & \mathbf{e}_{3} & \cdots & \mathbf{e}_{\kappa-1}
\end{array}\right]\label{eq:B-definition}
\end{equation}
i.e., each column of $\mathbf{B}$ is an element of $\{(\mathbf{e}_{i}^{\intercal},\mathbf{e}_{j}^{\intercal})^{\intercal}|i,j\in[1,\kappa]\}$
and the first $\kappa$ columns of $\mathbf{B}$ is equal to $\left[\begin{array}{cc}
\mathbf{I} & \mathbf{I}\end{array}\right]^{\intercal}$, where $\bar{\mathbf{B}},\ubar{\mathbf{B}}\in\mathbb{R}^{\kappa\times\kappa^{2}}$
are two submatrices consisting of the first $\kappa$ and the last
$\kappa$ rows of $\mathbf{B}$.

By introducing a dual variable $\bm{\Lambda}\in\mathbb{R}^{N^{c}\times\kappa^{2}}$
and using \eqref{eq:MD-definition} and \eqref{eq:B-definition},
the Lagrangian of \eqref{eq:m3c-coarse-grained-supervised} can be
written as
\begin{eqnarray}
\mathcal{L} & = & \frac{1}{2}\beta\left\Vert \mathbf{W}\right\Vert ^{2}+\mathbf{c}^{\intercal}\bm{\xi}^{c}-\mathrm{tr}\left(\mathbf{D}^{\intercal}\mathrm{diag}\left(\bm{\Lambda}\mathbf{1}\right)\left(\bar{\mathbf{X}}+\ubar{\mathbf{X}}\right)\mathbf{W}^{\intercal}\right)\nonumber \\
 &  & -\mathrm{tr}\left(\left[\begin{array}{cc}
\bm{\Lambda}^{\intercal}\mathbf{D} & \mathbf{0}\mathbf{0}^{\intercal}\end{array}\right]\right)+\mathrm{tr}\left(\bm{\Lambda}^{\intercal}\left(\bar{\mathbf{X}}\mathbf{W}^{\intercal}\bar{\mathbf{B}}+\ubar{\mathbf{X}}\mathbf{W}^{\intercal}\ubar{\mathbf{B}}\right)\right)\nonumber \\
 &  & +\left(\mathbf{1}-\bm{\xi}^{c}\right)^{\intercal}\bm{\Lambda}\mathbf{1}\label{eq:Lagrangian}
\end{eqnarray}
where $\bar{\mathbf{X}}=\left(\bm{\phi}(\bar{\mathbf{x}}_{1}^{c}),\ldots,\bm{\phi}(\bar{\mathbf{x}}_{N^{c}}^{c})\right)^{\intercal}\in\mathbb{R}^{N^{c}\times d}$
and $\ubar{\mathbf{X}}=\left(\bm{\phi}(\ubar{\mathbf{x}}_{1}^{c}),\ldots,\bm{\phi}(\ubar{\mathbf{x}}_{N^{c}}^{c})\right)^{\intercal}\in\mathbb{R}^{N^{c}\times d}$.
Setting the derivatives of the Lagrangian \eqref{eq:Lagrangian} with
respect to $\mathbf{W}$ and $\bm{\xi}^{c}$ to zero, and adding the
constraint $\bm{\Lambda}\ge0$, we obtain the following dual problem
of \eqref{eq:m3c-coarse-grained-supervised}:

\begin{equation}
\begin{array}{cl}
\max\limits _{\bm{\Lambda}} & -\frac{1}{2\beta}\mathrm{tr}\left(\left[\begin{array}{cc}
\bar{\mathbf{B}}\bm{\Lambda}^{\intercal} & \ubar{\mathbf{B}}\bm{\Lambda}^{\intercal}\end{array}\right]\mathbf{K}\left[\begin{array}{cc}
\bar{\mathbf{B}}\bm{\Lambda}^{\intercal} & \ubar{\mathbf{B}}\bm{\Lambda}^{\intercal}\end{array}\right]^{\intercal}\right)\\
 & +\frac{1}{\beta}\mathrm{tr}\left(\mathbf{D}^{\intercal}\mathbf{C}\left[\begin{array}{cc}
\mathbf{I} & \mathbf{I}\end{array}\right]\mathbf{K}\left[\begin{array}{cc}
\bar{\mathbf{B}}\bm{\Lambda}^{\intercal} & \ubar{\mathbf{B}}\bm{\Lambda}^{\intercal}\end{array}\right]^{\intercal}\right)\\
 & -\mathrm{tr}\left(\left[\begin{array}{cc}
\bm{\Lambda}^{\intercal}\mathbf{D} & \mathbf{0}\mathbf{0}^{\intercal}\end{array}\right]\right)-\frac{1}{2\beta}\mathrm{tr}\left(\mathbf{M}\mathbf{C}\mathbf{K}^{s}\mathbf{C}\right)+\mathbf{1}^{\intercal}\mathbf{c}\\
\mathrm{s.t.} & \bm{\Lambda}\mathbf{1}=\mathbf{1},\\
 & \bm{\Lambda}\ge0.
\end{array}\label{eq:m3c-coarse-grained-supervised-dual}
\end{equation}
where
\begin{equation}
\mathbf{K}=\left[\begin{array}{cc}
\mathbf{K}^{(1,1)} & \mathbf{K}^{(1,2)}\\
\mathbf{K}^{(2,1)} & \mathbf{K}^{(2,2)}
\end{array}\right]
\end{equation}
\begin{equation}
\mathbf{K}^{s}=\mathbf{K}^{(1,1)}+\mathbf{K}^{(1,2)}+\mathbf{K}^{(2,1)}+\mathbf{K}^{(2,2)}\label{eq:Ks}
\end{equation}
and
\begin{equation}
\mathbf{K}^{(1,1)}=\bar{\mathbf{X}}^{\intercal}\bar{\mathbf{X}},\mathbf{K}^{(1,2)}=\bar{\mathbf{X}}^{\intercal}\ubar{\mathbf{X}},\mathbf{K}^{(1,1)}=\ubar{\mathbf{X}}^{\intercal}\bar{\mathbf{X}},\mathbf{K}^{(1,2)}=\ubar{\mathbf{X}}^{\intercal}\ubar{\mathbf{X}}
\end{equation}
Putting \eqref{eq:m3c-coarse-grained-supervised-dual} into the form
of a standard quadratic programming problem, we have
\begin{equation}
\begin{array}{cl}
\max\limits _{\bm{\Lambda}} & -\frac{1}{2\beta}\mathrm{vec}\left(\bm{\Lambda}\right)^{\intercal}\mathbf{P}\mathrm{vec}\left(\bm{\Lambda}\right)+\mathbf{q}^{\intercal}\mathrm{vec}\left(\bm{\Lambda}\right)-\frac{1}{2\beta}\mathrm{tr}\left(\mathbf{M}\mathbf{C}\mathbf{K}^{s}\mathbf{C}\right)+\mathbf{1}^{\intercal}\mathbf{c}\\
\mathrm{s.t.} & \left(\mathbf{1}^{\intercal}\otimes\mathbf{I}\right)\mathrm{vec}\left(\bm{\Lambda}\right)=\mathbf{1},\\
 & \mathrm{vec}\left(\bm{\Lambda}\right)\ge0.
\end{array}\label{eq:m3c-coarse-grained-supervised-dual-qp}
\end{equation}
where
\begin{eqnarray}
\mathbf{P} & = & \left(\bar{\mathbf{B}}^{\intercal}\bar{\mathbf{B}}\right)\otimes\mathbf{K}^{(1,1)}+\left(\bar{\mathbf{B}}^{\intercal}\ubar{\mathbf{B}}\right)\otimes\mathbf{K}^{(1,2)}\nonumber \\
 &  & +\left(\ubar{\mathbf{B}}^{\intercal}\bar{\mathbf{B}}\right)\otimes\mathbf{K}^{(2,1)}+\left(\ubar{\mathbf{B}}^{\intercal}\ubar{\mathbf{B}}\right)\otimes\mathbf{K}^{(2,2)}
\end{eqnarray}
and
\begin{eqnarray}
\mathbf{q} & = & \mathrm{vec}\bigg(\frac{1}{\beta}\left(\mathbf{K}^{(1,1)}+\mathbf{K}^{(2,1)}\right)^{\intercal}\mathbf{C}^{\intercal}\mathbf{D}\bar{\mathbf{B}}\nonumber \\
 &  & +\frac{1}{\beta}\left(\mathbf{K}^{(1,2)}+\mathbf{K}^{(2,2)}\right)^{\intercal}\mathbf{C}^{\intercal}\mathbf{D}\ubar{\mathbf{B}}-\left[\begin{array}{cc}
\mathbf{D} & \mathbf{0}\mathbf{0}^{\intercal}\end{array}\right]\bigg)\label{eq:q}
\end{eqnarray}
(The definition of $\mathrm{vec}\left(\cdot\right)$ is given in the
list of notation.) According to Lemma \ref{lem:qp} and the strong
duality theorem \citep{Boyd2004convex}, we can conclude that \eqref{eq:m3c-coarse-grained-supervised}
has the same minimum with the following optimization problem:
\begin{equation}
\begin{array}{cl}
\min\limits _{\bm{\alpha},\mathbf{v},\bm{\theta}} & \frac{1}{2}\bm{\theta}^{\intercal}\bm{\theta}-\mathbf{c}^{\intercal}\bm{\alpha}-\frac{1}{2\beta}\mathrm{tr}\left(\mathbf{M}\mathbf{C}\mathbf{K}^{s}\mathbf{C}\right)+\mathbf{1}^{\intercal}\mathbf{c}\\
\mathrm{s.t.} & \mathbf{q}+\left(\mathbf{1}\otimes\mathbf{I}\right)\bm{\alpha}+\mathbf{v}+\mathbf{R}\bm{\theta}=0,\\
 & \mathbf{v}\ge0.
\end{array}\label{eq:m3c-coarse-grained-supervised-dual-dual}
\end{equation}
where $\mathbf{R}$ is a full column rank matrix satisfying
\begin{equation}
\frac{1}{\beta}\mathbf{P}=\mathbf{R}\mathbf{R}^{\intercal}\label{eq:R}
\end{equation}
(Note that $\mathbf{R}$ may not be a square matrix if $\mathbf{P}$
is not full rank.)

Combining the equivalence between \eqref{eq:m3c-coarse-grained-supervised}
and \eqref{eq:m3c-coarse-grained-supervised-dual-dual} and the proposition
mentioned in Remark \ref{rem:M-D}, the theorem is proved.
\begin{lem}
\label{lem:qp}For a quadratic programming problem defined by
\begin{equation}
\begin{array}{cl}
\min\limits _{\mathbf{x}} & \frac{1}{2}\mathbf{x}^{\intercal}\mathbf{A}\mathbf{x}+\mathbf{b}^{\intercal}\mathbf{x}\\
\mathrm{s.t.} & \mathbf{E}\mathbf{x}=\mathbf{c},\\
 & \mathbf{x}\ge0.
\end{array}\label{eq:qp}
\end{equation}
if $\mathbf{A}$ can be decomposed as $\mathbf{A}=\mathbf{R}\mathbf{R}^{\intercal}$
with $\mathbf{R}$ being full column rank and there is an optimal
solution to \eqref{eq:qp}, then the minimum of \eqref{eq:qp} is
equal to the maximum of the following problem:
\begin{equation}
\begin{array}{cl}
\max\limits _{\bm{\alpha},\mathbf{v},\bm{\theta}} & -\frac{1}{2}\bm{\theta}^{\intercal}\bm{\theta}+\mathbf{c}^{\intercal}\bm{\alpha}\\
\mathrm{s.t.} & \mathbf{b}^{\intercal}-\bm{\alpha}^{\intercal}\mathbf{E}-\mathbf{v}^{\intercal}=\bm{\theta}^{\intercal}\mathbf{R}^{\intercal},\\
 & \mathbf{v}\ge0.
\end{array}\label{eq:qp-dual}
\end{equation}
\end{lem}
\begin{proof}
The Lagrangian of \eqref{eq:qp} is
\begin{equation}
\mathcal{L}\left(\mathbf{x},\bm{\alpha},\mathbf{v}\right)=\frac{1}{2}\mathbf{x}^{\intercal}\mathbf{A}\mathbf{x}+\left(\mathbf{b}^{\intercal}-\bm{\alpha}^{\intercal}\mathbf{E}-\mathbf{v}^{\intercal}\right)\mathbf{x}+\mathbf{c}^{\intercal}\bm{\alpha}
\end{equation}
Then the dual problem of \eqref{eq:qp} can be written as
\begin{equation}
\begin{array}{cl}
\max\limits _{\bm{\alpha},\mathbf{v}} & g\left(\bm{\alpha},\mathbf{v}\right)\\
\mathrm{s.t.} & \mathbf{v}\ge0.
\end{array}
\end{equation}
with
\begin{equation}
g\left(\bm{\alpha},\mathbf{v}\right)=\inf_{\mathbf{x}}\mathcal{L}\left(\mathbf{x},\bm{\alpha},\mathbf{v}\right)
\end{equation}
We now analyze the value of $g\left(\bm{\alpha},\mathbf{v}\right)$
in different cases.
\begin{description}
\item [{Case (i)}] There is a $\bm{\theta}$ such that
\begin{equation}
\mathbf{b}^{\intercal}-\bm{\alpha}^{\intercal}\mathbf{E}-\mathbf{v}^{\intercal}=\bm{\theta}^{\intercal}\mathbf{R}^{\intercal}\label{eq:case-1}
\end{equation}
 We have $\mathcal{L}\left(\mathbf{x},\bm{\alpha},\mathbf{v}\right)=\frac{1}{2}\left(\mathbf{R}^{\intercal}\mathbf{x}\right)^{\intercal}\left(\mathbf{R}^{\intercal}\mathbf{x}\right)+\bm{\theta}^{\intercal}\left(\mathbf{R}^{\intercal}\mathbf{x}\right)+\mathbf{c}^{\intercal}\bm{\alpha}$
and $g\left(\bm{\alpha},\mathbf{v}\right)=-\frac{1}{2}\bm{\theta}^{\intercal}\bm{\theta}+\mathbf{c}^{\intercal}\bm{\alpha}$.
\item [{Case (ii)}] There is no $\bm{\theta}$ satisfying \eqref{eq:case-1}.
It is easy to see that we can find an $\mathbf{x}$ such that $\mathbf{R}^{\intercal}\mathbf{x}=0$
and $\left(\mathbf{b}^{\intercal}-\bm{\alpha}^{\intercal}\mathbf{E}-\mathbf{v}^{\intercal}\right)\mathbf{x}\neq0$.
Then $g\left(\bm{\alpha},\mathbf{v}\right)=-\infty$.
\end{description}

Combining the above results yields the conclusion of the lemma.

\end{proof}

\section{Optimization procedure for MMC\label{sec:Optimization-procedure-for-mmc}}

The optimization algorithm for solving MMC problems in our experiments
is described by Algorithm \ref{alg:mmc}. It can be seen that this
algorithm also combines global search and local search techniques
through coarse graining like the algorithm for \MMMC proposed in
this paper.

\begin{algorithm}
\begin{algorithmic}[1]

\STATE generate a coarse-grained set $S^{c}$ with cardinality $N^{c}$
and normalized weights $\mathbf{c}=(c_{1},\ldots,c_{N^{c}})^{\intercal}$
from $\mathcal{S}$ by the $k$-medoids algorithm

\STATE solve the MMC problem on $S^{c}$ by the SDP relaxation algorithm
\citep{xu2005unsupervised} to get class labels $\mathbf{y}^{c}=(y_{1}^{c},\ldots,y_{N^{c}}^{c})$
of $\mathcal{S}^{c}$

\STATE calculate class labels $\mathbf{y}^{(0)}=(y_{1},\ldots,y_{\left|\mathcal{S}\right|})$
of data points in $\mathcal{S}$ from $\mathbf{y}^{c}$

\STATE solve the MMC problem on $\mathcal{S}$ by the local search
algorithm proposed in \citep{zhang2009maximum} starting from $\mathbf{y}=\mathbf{y}^{(0)}$.

\end{algorithmic}

\protect\caption{Optimization procedure for MMC\label{alg:mmc}}
\end{algorithm}

\section{Implementation procedure of PCCA+\label{sec:Implementation-procedure-of-PCCA}}

In this paper, we perform the PCCA+ clustering as shown in Algorithm
\ref{alg:pcca}.

\begin{algorithm}
\begin{algorithmic}[1]

\STATE partition all the data into $N^{c}$ bins $\mathbf{x}_{1}^{c},\ldots,\mathbf{x}_{N^{c}}^{c}$
by the $k$-medoids algorithm

\STATE estimate the transition matrix $\mathbf{P}=[P_{ij}]=[\Pr(\mathbf{x}_{t+\Delta t}\in\mathbf{x}_{j}^{c}|\mathbf{x}_{t+\Delta t}\in\mathbf{x}_{i}^{c})]$
by the maximum likelihood algorithm in \citep{prinz2011markov}

\STATE apply the Markov compression algorithm in \citep{deuflhard2005robust}
to lump the $N^{c}$ bins into $\kappa$ metastable state.

\end{algorithmic}

\protect\caption{Optimization procedure for PCCA+\label{alg:pcca}}
\end{algorithm}

\section{Description of Model I and Model II\label{sec:Description-of-ModelI-II}}

Model I is governed by the Fokker-Planck equation
\begin{equation}
\mathrm{d}\mathbf{x}_{t}=-\left[\begin{array}{c}
\frac{1}{4}\frac{\partial}{\partial x^{(1)}}\\
9\frac{\partial}{\partial x^{(2)}}
\end{array}\right]U_{\mathrm{I}}\left(\mathbf{x}\right)\mathrm{d}t+\left[\begin{array}{cc}
\frac{\sqrt{2}}{2}\\
 & 3\sqrt{2}
\end{array}\right]\mathrm{d}\mathbf{u}_{t}\label{eq:Model-I}
\end{equation}
where $\mathbf{x}_{t}=(x_{t}^{(1)},x_{t}^{(2)})$ denotes the system
state at time $t$, $\mathbf{u}_{t}$ is a two-dimensional Wiener
process, and
\begin{eqnarray}
U_{\mathrm{I}}\left(\mathbf{x}\right) & = & -8\sum_{(\mu_{1},\mu_{2})\in\{-1,0,1\}\times\{-\frac{1}{8},\frac{1}{8}\}}\exp\Bigg(-8\left(x^{(1)}-\mu_{1}\right)^{2}\nonumber \\
 &  & -200\left(\frac{x^{(2)}}{6}-\mu_{2}\right)^{2}\Bigg)\nonumber \\
 &  & +\frac{4}{5}\left(x^{(1)}\right)^{4}+\frac{16}{9}\left(x^{(2)}\right)^{2}\label{eq:Model-II}
\end{eqnarray}
The stochastic differential equation of Model II is
\begin{equation}
\mathrm{d}\mathbf{x}_{t}=-\frac{1}{\gamma}\left[\begin{array}{c}
\frac{\partial}{\partial x^{(1)}}\\
\frac{\partial}{\partial x^{(2)}}
\end{array}\right]U_{\mathrm{II}}\left(\mathbf{x}\right)\mathrm{d}t+\sqrt{\frac{2}{\gamma}}\mathrm{d}\mathbf{u}_{t}
\end{equation}
with
\begin{eqnarray}
U_{\mathrm{II}}\left(\mathbf{x}\right) & = & -4\gamma\exp\left(-16\left(\left(\left\Vert \mathbf{x}\right\Vert -1.6\right)^{2}+\max\left\{ \theta-\frac{\pi}{2},0\right\} ^{2}\right)\right)\nonumber \\
 &  & -4\gamma\exp\left(-0.8\left(x^{(1)}+1\right)^{2}-32\left(x^{(2)}-0.5\right)^{2}\right)\nonumber \\
 &  & -4\gamma\exp\left(-0.8\left(x^{(1)}+1\right)^{2}-32\left(x^{(2)}+0.5\right)^{2}\right)\nonumber \\
 &  & +\frac{1}{5}\gamma\left(\left(x^{(1)}\right)^{4}+\left(x^{(2)}\right)^{4}\right)
\end{eqnarray}
$\theta=\mathrm{atan2}\left(x^{(2)},x^{(1)}\right)$ and $\gamma=1.67$.
It is easy to verify that Model I and Model II are time-reversible
diffusion processes, and their equilibrium distributions are $\pi\left(\mathbf{x}\right)\propto\exp\left(-U_{\mathrm{I}}\left(\mathbf{x}\right)\right)$
and $\pi\left(\mathbf{x}\right)\propto\exp\left(-U_{\mathrm{II}}\left(\mathbf{x}\right)\right)$
separately.

Moreover, we utilize the Euler-Maruyama method \citep{Kloeden1992Numerical}
to solve \eqref{eq:Model-I} and \eqref{eq:Model-II} in this paper.

\bibliographystyle{elsarticle-num}
\bibliography{mmc_metastable}

\end{document}